\documentclass[english,10pt]{article}

\usepackage[latin9]{inputenc} 
\usepackage[T1]{fontenc}    
\usepackage{lmodern}
\usepackage{geometry}
\usepackage[unicode=true,
 bookmarks=false,
 breaklinks=false,pdfborder={0 0 1},colorlinks=false]
 {hyperref}
\hypersetup{
 colorlinks,citecolor=blue,filecolor=blue,linkcolor=blue,urlcolor=blue}
\usepackage{booktabs}       
\usepackage{xcolor}
\usepackage{cite} 
\usepackage{bm}
\usepackage{amsmath,amsthm,amssymb,mathtools}
\usepackage{nicefrac}       
\usepackage{microtype}      
\usepackage{appendix}
\usepackage{enumerate}
\usepackage[linesnumbered,ruled,vlined]{algorithm2e}
\usepackage{algorithmic} 
\usepackage{xspace}
\usepackage{graphicx}
\usepackage{tabularx, makecell, booktabs}
\usepackage{colortbl}
\usepackage{tablefootnote}
\usepackage{natbib}
\usepackage{verbatim}
\usepackage{adjustbox}
\usepackage{nicematrix}
\usepackage{multirow}
\usepackage{dsfont}

\DeclareFontFamily{U}{mathx}{}
\DeclareFontShape{U}{mathx}{m}{n}{<-> mathx10}{}
\DeclareSymbolFont{mathx}{U}{mathx}{m}{n}
\DeclareMathAccent{\widecheck}{0}{mathx}{"71}

\allowdisplaybreaks

\newcommand{\mymid}{\,|\,}

\newtheorem{theorem}{\textbf{Theorem}}

\newtheorem{claim}[theorem]{\textbf{Claim}}
\newtheorem{lemma}[theorem]{\textbf{Lemma}}

\theoremstyle{theorem}\newtheorem{remark}{\textbf{Remark}}
\theoremstyle{theorem}\newtheorem{definition}{\textbf{Definition}}
\theoremstyle{theorem}\newtheorem{assumption}{\textbf{Assumption}}

\newcommand{\simon}[1]{\textcolor{cyan}{[Simon: #1]}}

\definecolor{yxc}{RGB}{255,0,0}

\title{Settling the Sample Complexity of Online Reinforcement Learning}

\author{%
	Zihan Zhang\thanks{Department of Electrical and Computer Engineering, Princeton University; email: \texttt{\{zz5478,jasonlee\}@princeton.edu}.}\\
  Princeton  \and
 Yuxin Chen\thanks{Department of Statistics and Data Science, University of Pennsylvania; email: \texttt{yuxinc@wharton.upenn.edu}.}\\
 UPenn
 \and
Jason D.~Lee\footnotemark[1] \\
 Princeton 
 \and
 Simon S.~Du\thanks{Paul G. Allen School of Computer Science and Engineering, University of Washington; email: \texttt{ssdu@cs.washington.edu}.}\\
 U.~Washington
}
\date{July 2023; ~~Revised: April 2025}

\ifdefined\usebigfont

\usepackage{times}
\usepackage[fontsize=13pt]{scrextend}
\makeatletter
\@ifpackageloaded{geometry}{\AtBeginDocument{\newgeometry{letterpaper,left=1.56in,right=1.56in,top=1.71in,bottom=1.77in}}}{\usepackage[letterpaper,left=1.56in,right=1.56in,top=1.71in,bottom=1.77in]{geometry}}
\AtBeginDocument{\newgeometry{letterpaper,left=1.56in,right=1.56in,top=1.71in,bottom=1.77in}}
\makeatother
\else
\fi
\begin{document}
\maketitle

\begin{abstract}
    A central issue lying at the heart of online reinforcement learning (RL) is data efficiency. 
While a number of recent works achieved asymptotically minimal regret in online RL, 
the optimality of these results is only guaranteed in a ``large-sample'' regime, 
imposing enormous burn-in cost in order for their algorithms to operate optimally. 
How to achieve minimax-optimal regret without incurring any burn-in cost has been an open problem in RL theory.  

We settle this problem for finite-horizon inhomogeneous Markov decision processes. 
Specifically, we prove that a modified version of $\mathtt{MVP}$ (Monotonic Value Propagation), an optimistic model-based algorithm proposed by \citet{zhang2020reinforcement}, achieves a regret on the order of (modulo log factors)
\begin{equation*}
	\min\big\{  \sqrt{SAH^3K}, \,HK \big\},
\end{equation*}
where $S$ is the number of states, $A$ is the number of actions, $H$ is the horizon length, and $K$ is the total number of episodes. 
This regret matches the minimax lower bound for the entire range of sample size $K\geq 1$, essentially eliminating any burn-in requirement. It also translates to a PAC sample complexity (i.e., the number of episodes needed to yield $\varepsilon$-accuracy) of $\frac{SAH^3}{\varepsilon^2}$ up to log factor, which is minimax-optimal for the full $\varepsilon$-range. 
 Further, we extend our theory to unveil the influences of problem-dependent quantities like the optimal value/cost and certain variances.  
The key technical innovation lies in a novel analysis paradigm (based on a new concept called ``profiles'') to decouple complicated statistical dependency across the sample trajectories --- a long-standing challenge facing the analysis of online RL in the sample-starved regime.


\end{abstract}

\setcounter{tocdepth}{2}
\tableofcontents

\section{Introduction}\label{sec:intro}

In reinforcement learning (RL), 
an agent is often asked to learn optimal decisions (i.e., the ones that maximize cumulative reward) through real-time ``trial-and-error'' interactions with an unknown environment. 
This task is commonly dubbed as {\em online RL}, underscoring the critical role of adaptive online data collection and differentiating it from other RL settings that rely upon pre-collected data. A central challenge in achieving sample-efficient online RL boils down to how to optimally balance exploration and exploitation during data collection, 
namely, how to trade off the potential revenue of exploring unknown terrain/dynamics against the benefit of exploiting past experience. While decades-long effort has been invested towards unlocking the capability of online RL, 
how to {\em fully} characterize --- and more importantly, attain --- its fundamental performance limit remains largely unsettled.

In this paper, we take an important step towards settling the sample complexity limit of online RL, 
focusing on tabular Markov Decision Processes (MDPs) with finite horizon and finite state-action space. 
More concretely, imagine that one seeks to learn a near-optimal policy of a time-inhomogeneous MDP with $S$ states, $A$ actions, and horizon length $H$, and is allowed to execute the MDP of interest $K$ times to collect $K$ sample episodes each of length $H$. 
This canonical problem is among the most extensively studied in the RL literature, 
with formal theoretical pursuit dating back to more than 25 years ago (e.g., \citet{kearns1998near}). 
Numerous works have since been devoted to improving the sample efficiency and/or refining the analysis framework~\citep{brafman2002r,kakade2003sample,jaksch2010near,azar2017minimax, jin2018q,dann2017unifying,zanette2019tighter,bai2019provably,zhang2020almost,zhang2020reinforcement,menard2021ucb,li2021breaking,domingues2021episodic}. 
As we shall elucidate momentarily, however, information-theoretic optimality has only been achieved in the ``large-sample'' regime. 
When it comes to the most challenging sample-hungry regime, 
there remains a substantial gap between the state-of-the-art regret upper bound and the best-known minimax lower bound, 
which motivates the research of this paper.

\subsection{Inadequacy of prior art: enormous burn-in cost} 

While past research has obtained asymptotically optimal (i.e., optimal when $K$ approaches infinity) regret bounds in the aforementioned setting, 
all of these results incur an enormous burn-in cost --- that is, the minimum sample size needed for an algorithm to operate sample-optimally --- which we explain in the sequel.  
For simplicity of presentation, we assume that each immediate reward lies within the normalized range $[0,1]$ when discussing the prior art.

\paragraph{Minimax lower bound.}
To provide a theoretical benchmark,
we first make note of the best-known minimax regret lower bound developed by \citet{jin2018q,domingues2021episodic}:\footnote{Let $\mathcal{X}=\{S,A,H,K,\frac{1}{\delta}\}$, where $1-\delta$ is the target success rate (to be seen shortly). The notation $f(\mathcal{X})=O\big(g(\mathcal{X})\big)$ (or $f(\mathcal{X})\lesssim g(\mathcal{X})$) indicates the existence of some universal constant $c_1>0$ such that $f(\mathcal{X})\leq c_1 g(\mathcal{X})$; 
$f(\mathcal{X})=\Omega\big(g(\mathcal{X})\big)$ (or $f(\mathcal{X})\gtrsim g(\mathcal{X})$) means that there exists some universal constant $c_2>0$ such that $f(\mathcal{X})\geq c_2 g(\mathcal{X})$; and $f(\mathcal{X})=\Theta\big(g(\mathcal{X})\big)$ (or $f(\mathcal{X})\asymp g(\mathcal{X})$) means that $f(\mathcal{X})\lesssim g(\mathcal{X})$ and $f(\mathcal{X})\gtrsim g(\mathcal{X})$ hold simultaneously. 
Moreover, $\widetilde{O}\left(\cdot \right)$, $\widetilde{\Omega}\left(\cdot\right)$ and $\widetilde{\Theta}\left(\cdot\right)$ are defined analogously, except that all logarithmic dependency on the quantities of $\mathcal{X}$ are hidden. 
} 
\begin{align}
	\text{(minimax lower bound)} \qquad 
	\Omega\left(\min\big\{\sqrt{SAH^3K} ,\, HK \big\}\right), 
	\label{eq:minimax-lower-bound-1}
\end{align}
assuming that the immediate reward at each step falls within $[0,1]$ and imposing no restriction on $K$.  
Given that a regret of $O(HK)$ can be trivially achieved (as the sum of rewards in any $K$ episodes cannot exceed $HK$), 
we shall sometimes drop the $HK$ term and simply write
\begin{align}
	\text{(minimax lower bound)} \qquad 
	\Omega\big(\sqrt{SAH^3K} \big) \qquad \text{if }K \geq SAH. 
	\label{eq:minimax-lower-bound-2}
\end{align}

\paragraph{Prior upper bounds and burn-in cost.}
We now turn to the upper bounds developed in prior literature. For ease of presentation, we shall assume 
\begin{align}
	K \geq SAH  
\end{align}
in the rest of this subsection unless otherwise noted. Log factors are also ignored in the discussion below.

The first paper that achieves asymptotically optimal regret is \citet{azar2017minimax}, 
which came up with a model-based algorithm called $\mathtt{UCBVI}$ that enjoys a regret bound $\widetilde{O}\big(\sqrt{SAH^3K} + H^3S^2A\big)$. 
A close inspection reveals that this regret matches the minimax lower bound \eqref{eq:minimax-lower-bound-2} if and only if 
\begin{align}
	\big(\text{burn-in cost of \citet{azar2017minimax}}\big) \qquad K \gtrsim S^3A H^3, 
\end{align}
due to the presence of the lower-order term $H^3S^2A$ in the regret bound. 
This burn-in cost is clearly undesirable, 
since the sample size available in many practical scenarios might be far below this requirement.

In light of its fundamental importance in contemporary RL applications (which often have very large dimensionality and relatively limited data collection capability), reducing the burn-in cost without compromising sample efficiency has emerged as a central problem in recent pursuit of RL theory~\citep{zanette2019tighter,dann2019policy,zhang2020reinforcement,zhou2023sharp,menard2021ucb,li2021breaking,li2021settling,li2022minimax,agarwal2020model,sidford2018variance}. The state-of-the-art regret upper bounds for finite-horizon inhomogeneous MDPs can be summarized below (depending on the size of $K$): 
\begin{subequations}
\begin{align}
	\text{\citep{menard2021ucb}}  \qquad & \widetilde{O}\big(\sqrt{SAH^3K}+SAH^4 \big),   \\
	\text{\citep{zhang2020reinforcement,zhou2023sharp}} \qquad & \widetilde{O}\big(\sqrt{SAH^3K} +S^2AH^2 \big), 
\end{align}
\end{subequations}
meaning that even the most advanced prior results fall short of sample optimality unless 
\begin{align}
	\big(\text{best burn-in cost in past works}\big) \qquad K \gtrsim \min \big\{ SAH^5, S^3AH \big\}.  
\end{align}
The interested reader is referred to Table~\ref{tab:comparisons} for more details about existing regret upper bounds and their associated sample complexities.

In summary, no prior theory was able to achieve optimal sample complexity in the data-hungry regime 
\begin{equation}
SAH\leq  K \lesssim \min \big\{ SAH^5, S^3AH \big\},
\end{equation}
suffering from a significant barrier of either long horizon (as in the term $SAH^5$) or large state space (as in the term $S^3AH$). 
In fact, the information-theoretic limit is yet to be determined within this regime (i.e., neither the achievability results nor the lower bounds had been shown to be tight), although it has been conjectured by \citet{menard2021ucb} that 
the lower bound \eqref{eq:minimax-lower-bound-1} reflects the correct scaling for any sample size $K$.\footnote{Note that the original conjecture in \citet{menard2021ucb} was $\widetilde{\Theta}\big(\sqrt{SAH^3K}+SAH^2\big)$. 
Combining it with the trivial upper bound $HK$ allows one to remove the term $SAH^2$ (with a little algebra).}

\newcommand{\topsepremove}{\aboverulesep = 0mm \belowrulesep = 0mm} \topsepremove

\begin{table}[t]
		\centering
  \resizebox{\textwidth}{!}{
			\begin{tabular}{ c|c|c|c}
				\toprule
				\textbf{Algorithm} & \textbf{Regret upper bound} & \makecell{\textbf{Range of $K$ that} $\vphantom{\frac{1^{7^{7}}}{1^{7}}}$ \\\textbf{attains optimal regret}}   & \makecell{\textbf{Sample complexity} \\ \textbf{(or PAC bound)}} \\ 
				\toprule
\rowcolor{gray!25} 
	$\mathtt{MVP}$ &  &  & \\
\rowcolor{gray!25}
	{\bf (this work, Theorem~\ref{thm1})} & \multirow{-2}{*}{\cellcolor{gray!25}$\min\big\{\sqrt{SAH^3K},HK\big\}$} & 
				\multirow{-2}{*}{\cellcolor{gray!25}$[1,\infty)$ }
				& \multirow{-2}{*}{\cellcolor{gray!25}$\frac{SAH^3}{\varepsilon^2}$ }\\
				\hline
				\makecell{$\mathtt{UCBVI}$\\\citep{azar2017minimax}} & $\min\big\{\sqrt{SAH^3K}+S^2AH^3,\,HK\big\}$  & $[S^3AH^3,\infty)$  & $\frac{SAH^3}{\varepsilon^2} + \frac{S^2AH^3}{\varepsilon}$ \\
				\hline

				\makecell{
$\mathtt{ORLC}$\\\citep{dann2019policy} }& $\min\big\{\sqrt{SAH^3K}+S^2AH^4,\,HK\big\}$ & $[S^3AH^5,\infty)$ & $\frac{SAH^3}{\varepsilon^2}+\frac{S^2AH^4}{\varepsilon}$\\
  \hline
			
			\makecell{$\mathtt{EULER}$\\\citep{zanette2019tighter}} &    $\min\big\{\sqrt{SAH^3K}+ S^{3/2}AH^3(\sqrt{S}+\sqrt{H}),\,HK\big\}$ & $\big[S^{2}AH^3(\sqrt{S}+\sqrt{H}),\infty\big)$ & $\frac{SAH^3}{\varepsilon^2} + \frac{S^2AH^3(\sqrt{S}+\sqrt{H})}{\varepsilon}$\\
    \hline			
				\makecell{ $\mathtt{UCB}\text{-}\mathtt{Adv}$\\\citep{zhang2020almost}} & $\min\big\{\sqrt{SAH^3K}+S^2A^{3/2}H^{33/4}K^{1/4},\,HK\big\}$ & $[ S^6A^4H^{27},\infty)$ & $\frac{SAH^3}{\varepsilon^2} + \frac{S^{8/3}A^2H^{11}}{\varepsilon^{4/3}}$ 
  \\
  \hline
  \makecell{$\mathtt{MVP}$\\\citep{zhang2020reinforcement}} & $\min\big\{\sqrt{SAH^3K}+S^2AH^2,\,HK\big\}$  &  $[S^3AH,\infty)$ &  $\frac{SAH^3}{\varepsilon^2} + \frac{S^2AH^2}{\varepsilon}$\\
  \hline  
				\makecell{$\mathtt{UCBMQ}$\\\citep{menard2021ucb}} & $\min\big\{\sqrt{SAH^3K}+SAH^4,\,HK\big\}$ & $[SAH^5,\infty)$ & $\frac{SAH^3}{\varepsilon^2} + \frac{SAH^4}{\varepsilon}$\\
  \hline 
  \makecell{$\mathtt{Q}$-$\mathtt{Earlysettled}$-$\mathtt{Adv}$
  \\\citep{li2021breaking}}  & $\min\big\{\sqrt{SAH^3K}+SAH^6,\,HK\big\}$ & $[SAH^{9},\infty)$ &  $\frac{SAH^3}{\varepsilon^2} + \frac{SAH^{6}}{\varepsilon}$\\
  \toprule
	\makecell{Lower bound\\\citep{domingues2021episodic}} & $\min\big\{\sqrt{SAH^3K},HK\big\}$ & n/a & $\frac{SAH^3}{\varepsilon^2}$  \\
				\toprule
			\end{tabular}}
		\caption{
			Comparisons between our result and prior works that achieve asymptotically optimal regret for finite-horizon inhomogeneous MDPs (with all log factors omitted),   where $S$ (resp.~$A$) is the number of states (resp.~actions), $H$ is the planning horizon, and $K$ is the number of episodes. The third column reflects the burn-in cost, and the sample complexity (or PAC bound) refers to the number of episodes needed to yield $\varepsilon$ accuracy. 
The results provided here account for all $K\geq 1$ or all $\varepsilon \in (0,H]$. 
Our paper is the only one that gives regret (resp.~PAC) bound matching the minimax lower bound for the entire range of $K$ (resp.~$\varepsilon$).  
 \label{tab:comparisons}
}
	\end{table}


\paragraph{Comparisons with other RL settings and key challenges.} 
In truth, the incentives to minimize the burn-in cost and improve data efficiency arise in multiple other settings beyond online RL.  
For instance, in an idealistic setting that assumes access to a simulator (or a generative model) --- a model that allows the learner to query arbitrary state-action pairs to draw samples --- a recent work \citet{li2020breaking} developed a perturbed model-based approach that is provably optimal without incurring any burn-in cost. Analogous results have been obtained in \citet{li2021settling} for offline RL --- a setting that requires policy learning to be performed based on historical data --- unveiling the full-range optimality of a pessimistic model-based algorithm.

Unfortunately, the algorithmic and analysis frameworks developed in the above two works fail to accommodate the online counterpart. 
The main hurdle stems from the complicated statistical dependency intrinsic to episodic online RL; 
for instance, in online RL, the empirical transition probabilities and the running estimates of the value function are oftentimes statistically dependent in an intertwined manner (unless we waste data).  
How to decouple the intricate statistical dependency without compromising data efficiency constitutes the key innovation of this work. 
More precise, in-depth technical discussions will be provided in Section~\ref{sec:tec}.


\subsection{A peek at our main contributions} 

We are now positioned to summarize the main findings of this paper. 
Focusing on time-inhomogeneous finite-horizon MDPs, 
our main contributions can be divided into two parts:  
 the first part fully settles the minimax-optimal regret and sample complexity of online RL, 
 whereas the second part further extends and augments our theory to make apparent the impacts of certain problem-dependent quantities. 
 Throughout this subsection, the regret metric $\mathsf{Regret}(K)$ captures the cumulative sub-optimality gap (i.e., the gap between the performance of the policy iterates and that of the optimal policy) over all $K$ episodes, 
 to be formally defined in \eqref{eq:defn-regret}.

\subsubsection{Settling the optimal sample complexity with no burn-in cost} 

Our first result {\em fully} determines the sample complexity limit of online RL in a minimax sense, 
allowing one to attain the optimal regret regardless of the number $K$ of episodes that can be collected. 
\begin{theorem}\label{thm1}  
For any $K \ge 1$ and any $0<\delta<1$, 
there exists an algorithm (see Algorithm~\ref{alg:main}) obeying
\begin{equation}
	\mathsf{Regret}(K) 
	\lesssim \min\bigg\{\sqrt{SAH^{3}K \log^5\frac{SAHK}{\delta}},HK\bigg\}
	\label{eq:regret-Thm1}
\end{equation}
with probability at least $1-\delta$.
\end{theorem}
The optimality of our regret bound \eqref{eq:regret-Thm1} can be readily seen given that 
it matches the minimax lower bound \eqref{eq:minimax-lower-bound-1} (modulo some logarithmic factor). 
One can also easily translate the above regret bound into sample complexity or probably approximately correct (PAC) bounds: 
the proposed algorithm is able to return an $\varepsilon$-suboptimal policy with high probability using at most 
\begin{equation}
	\text{(sample complexity)}\qquad 
	\widetilde{O}\left(\frac{SAH^3}{\varepsilon^2}\right) \quad \text{episodes}
\end{equation}
(or equivalently, $\widetilde{O}\big(\frac{SAH^4}{\varepsilon^2}\big)$ sample transitions as each episode has length $H$).  
Remarkably, this result holds true for the entire $\varepsilon$ range (i.e., any $\varepsilon \in (0,H]$), essentially eliminating the need of any burn-in cost.  
It is noteworthy that even in the presence of an idealistic generative model, this order of sample size is un-improvable \citep{azar2013minimax,li2020breaking}.

The algorithm proposed herein is a modified version of $\mathtt{MVP}$: {\em Monotonic Value Propagation}. 
Originally proposed by \citet{zhang2020reinforcement},  
the $\mathtt{MVP}$ method falls under the category of model-based approaches, 
a family of algorithms that construct explicit estimates of the probability transition kernel before value estimation and policy learning. 
Notably, a technical obstacle that obstructs the progress in understanding model-based algorithms  
arises from the exceedingly large model dimensionality: 
given that the dimension of the transition kernel scales proportionally with $S^2$, 
all existing analyses for model-based online RL fell short of effectiveness unless the sample size already far exceeds $S^2$ \citep{azar2017minimax,zhang2020reinforcement}.  
To overcome this undesirable source of burn-in cost, 
a crucial step is to empower the analysis framework in order to accommodate the highly sub-sampled regime (i.e., a regime where the sample size scales linearly with $S$), 
which we shall elaborate on in Section~\ref{sec:tec}. 
The full proof of Theorem~\ref{thm1} will be provided in Section~\ref{app:thmmain}.




\subsubsection{Extension: optimal problem-dependent regret bounds}

In practice, RL algorithms often perform far more appealingly than what their worst-case performance guarantees would suggest. 
This motivates a recent line of works that 
investigate optimal performance 
in a more problem-dependent fashion \citep{talebi2018variance,simchowitz2019non,zanette2019tighter,zhou2023sharp,fruit2018efficient,xu2021fine,yang2021q,jin2020reward,wagenmaker2022first,zhao2023variance,dann2021beyond,tirinzoni2021fully}. 
Encouragingly, the proposed algorithm automatically achieves optimality on a more refined problem-dependent level, 
without requiring prior knowledge of additional problem-specific knowledge. 
This results in several extended theorems that take into account different problem-dependent quantities.

The first extension below investigates how the optimal value influences the regret bound. 
\begin{theorem}[Optimal value-dependent regret] \label{thm:first}
For any $K \ge 1$, 
Algorithm~\ref{alg:main} satisfies
\begin{align}
	\mathbb{E}\big[\mathsf{Regret}(K)\big] \lesssim \min\big\{ \sqrt{SAH^{2}Kv^{\star}}, Kv^{\star}\big\} \log^{5}(SAHK), 
	\label{eq:optimal-regret-value}
\end{align}
%
where $v^{\star}$ is the value of the optimal policy averaged over the initial state distribution (to be formally defined in \eqref{eq:defn-vstar-formal}).
\end{theorem}

Moreover, there is also no shortage of applications where the use of a cost function is preferred over a value function \citep{agarwal2017open,allen2018make,lee2020bias,wang2023benefits}. For this purpose, we provide another variation based upon the optimal cost. 
\begin{theorem}[Optimal cost-dependent regret]\label{thm:cost}
For any $K \ge 1$ and any $0<\delta<1$, 
Algorithm~\ref{alg:main} achieves 
\begin{align}
	\mathsf{Regret}(K) \leq \widetilde{O}\left(\min\big\{\sqrt{SAH^2K c^{\star}}+SAH^2,\, K(H-c^{\star}) \big\}\right)
	\label{eq:optimal-regret-cost}
\end{align}
with probability exceeding $1-\delta$, where $c^{\star}$ denotes the cost of the optimal policy averaged over the initial state distribution (to be formally defined in \eqref{eq:defn-cstar-formal}).
\end{theorem}
\noindent 
It is worth noting that: despite the apparent similarity between the statements of Theorem~\ref{thm:first} and Theorem~\ref{thm:cost}, 
they do not imply each other, although their proofs are very similar to each other.

Finally, we establish another regret bound that reflects the effect of certain variance quantities of interest.  
\begin{theorem}[Optimal variance-dependent regret] \label{thm:var}
For any $K \ge 1$ and any $0<\delta<1$, 
Algorithm~\ref{alg:main} obeys
\begin{align}
	\mathsf{Regret}(K) \leq \widetilde{O}\left(\min\big\{\sqrt{SAHK\mathrm{var}}+SAH^2,\,KH\big\}\right) 
	\label{eq:optimal-regret-var}
\end{align}
with probability at least $1-\delta$, 
	where $\mathrm{var}$ is a certain variance-type metric (to be formally defined in \eqref{eq:defn-var-formal}).
\end{theorem}

Two remarks concerning the above extensions are in order: 
\begin{itemize}
	\item 
		In the worst-case scenarios, the quantities $v^{\star}$, $c^{\star}$ and $\mathrm{var}$ can all be as large as the order of $H$, 
		in which case Theorems~\ref{thm:first}-\ref{thm:var} readily recover Theorem~\ref{thm1}. 
In contrast, the advantages of Theorems~\ref{thm:first}-\ref{thm:var} 
		over Theorem~\ref{thm1} become more evident 
		in those favorable cases (e.g., the situation where $v^{\star} \ll H$ or $c^{\star} \ll H$, or the case when the environment is nearly deterministic (so that $\mathrm{var}\approx 0$)).

	\item 
Interestingly, the regret bounds in Theorems~\ref{thm:first}-\ref{thm:var} 
		all contain a lower-order term $SAH^2$, 
and one might naturally wonder whether this term is essential. 
To demonstrate the unavoidable nature of this term and hence the optimality of Theorems~\ref{thm:first}-\ref{thm:var},
we will provide matching lower bounds, to be detailed in Section~\ref{sec:extensions}.
\end{itemize}

\subsection{Related works}

\label{sec:rel}
Let us take a moment to discuss several related theoretical works on tabular RL. 
Note that there has also been an active line of research that exploits low-dimensional function approximation to further reduce sample complexity, 
which is beyond the scope of this paper.

Our discussion below focuses on two mainstream approaches that have received widespread adoption: the model-based approach and the model-free approach. 
In a nutshell, 
model-based algorithms decouple model estimation and policy learning, and often use the learned transition kernel to
 compute the value function and find a desired policy. 
In stark contrast, the model-free approach attempts to estimate the optimal value function and optimal policy directly without explicit estimation of the model. 
In general, model-free algorithms only require $O(SAH)$ memory --- needed when storing the running estimates for Q-functions and value functions --- while the model-based counterpart might require $\Omega(S^2AH)$ space in order to store the estimated transition kernel.

\paragraph{Sample complexity for RL with a simulator.} 
As an idealistic setting that separates the consideration of exploration from that of estimation, 
RL with a simulator (or generative model) has been studied by numerous works, 
allowing the learner to draw independent samples for any state-action pairs~\citep{kearns1998finite,pananjady2020instance,kakade2003sample,azar2013minimax,agarwal2020model,wainwright2019variance,wainwright2019stochastic,sidford2018near,sidford2018variance,chen2020finite,li2020breaking,li2023q,li2022minimax,even2003learning,shi2023curious,beck2012error,cui2021minimax}.
While both model-based and model-free approaches are capable of achieving asymptotic sample optimality \citep{sidford2018variance,wainwright2019variance,azar2013minimax,agarwal2020model}, 
all model-free algorithms that enjoy asymptotically optimal sample complexity suffer from dramatic burn-in cost. 
Thus far,  only the model-based approach has been shown to fully eliminate the burn-in cost  
for both discounted  infinite-horizon MDPs and inhomogeneous finite-horizon MDPs \citep{li2020breaking}.  
The full-range optimal sample complexity for time-homogeneous finite-horizon MDPs in the presence of a simulator remains open.

\paragraph{Sample complexity for offline RL.} 
The subfield of offline RL is concerned with learning based purely on a pre-collected dataset 
\citep{levine2020offline}. 
A frequently used mathematical model assumes that historical data are collected (often independently) using some behavior policy, 
and the key challenges (compared with RL with a simulator) come from distribution shift and incomplete data coverage. 
The sample complexity of offline RL has been the focus of 
a large strand of recent works, with asymptotically optimal sample complexity achieved by multiple algorithms~\citep{jin2021pessimism,xie2021policy,yin2022near,ren2021nearly,shi2022pessimistic,qu2020finite,yan2022efficacy,rashidinejad2021bridging,li2022settling,li2021sample,wang2022gap}.
Akin to the simulator setting, the fully optimal sample complexity (without burn-in cost) has only been achieved via the model-based approach when it comes to discounted infinite-horizon and inhomogeneous finite-horizon settings~\citep{li2022settling}. 
All asymptotically optimal model-free methods incur substantial burn-in cost. The case with time-homogeneous finite-horizon MDPs also remains unsettled.

\paragraph{Sample complexity for online RL.} 
Obtaining optimal sample complexity (or regret bound) in online RL without incurring any burn-in cost 
has been one of the most fundamental open problems in RL theory. 
In fact, the past decades have witnessed a flurry of activity towards improving the sample efficiency of online RL,  
partial examples including \citet{kearns1998near,brafman2002r,kakade2003sample,strehl2006pac,strehl2008analysis,kolter2009near,bartlett2009regal,jaksch2010near,szita2010model,lattimore2012pac,osband2013more,dann2015sample,agralwal2017optimistic,dann2017unifying,jin2018q,efroni2019tight,fruit2018efficient,zanette2019tighter,cai2019provably,dong2019q,russo2019worst,pacchiano2020optimism,neu2020unifying,zhang2020almost,zhang2020reinforcement,tarbouriech2021stochastic,xiong2021randomized,menard2021ucb,wang2020long,li2021settling,li2021breaking,domingues2021episodic,zhang2022horizon,li2023minimax,li2024reward,ji2023regret}. 
Unfortunately, no work has been able to conquer this problem completely: 
the state-of-the-art result for model-based algorithms still incurs a burn-in that scales at least quadratically in $S$ 
\citep{zhang2020reinforcement}, while the burn-in cost of the best model-free algorithms (particularly with the aid of variance reduction introduced in \citet{zhang2020almost}) still suffers from highly sub-optimal horizon dependency \citep{li2021breaking}.

\subsection{Notation} 
Before proceeding, let us introduce a set of notation to be used throughout. 
Let $1$ and $0$ indicate respectively the all-one vector and the all-zero vector. Let $e_{s}$ denote the $s$-th standard basis vector (which has $1$ at the $s$-th coordinate and $0$ otherwise).  
For any set $\mathcal{X}$,  $\Delta(\mathcal{X})$ represents the set of probability distributions over the set $\mathcal{X}$. 
For any positive integer $N$, we denote $[N]=\{1,\ldots,N\}$.
For any two vectors $x,y$ with the same dimension, 
we use $\langle x, y\rangle$ (or $x^{\top}y$) to denote the inner product of $x$ and $y$.   
For any integer $S>0$, any probability vector $p\in \Delta([S])$ and another vector $v=[v_i]_{1\leq i\leq S}$, 
we denote by 
\begin{equation}
	\mathbb{V}(p,v) \coloneqq \langle p, v^2 \rangle - (\langle p, v\rangle)^2 = \big\langle p,  \big(v-\langle p, v\rangle 1 \big) ^2  \big\rangle 
\end{equation}
 the associated variance,  
where $v^2=[v_i^2]_{1\leq i\leq S}$ represents element-wise square of $v$. 
For any two vectors $a=[a_i]_{1\leq i\leq n}$ and $b=[b_i]_{1\leq i\leq n}$, 
the notation $a\geq b$ (resp.~$a\leq b$) means $a_i\geq b_i$ (resp.~$a_i\leq b_i$) holds simultaneously for all $i$.  
Without loss of generality, we assume throughout that  $K$ is a power of $2$ to streamline presentation.

%
%
%
%

\section{Problem formulation}
\label{sec:pre}

In this section, we introduce the basics of tabular online RL, as well as some basic assumptions to be imposed throughout. 

\paragraph{Basics of finite-horizon MDPs.}
This paper concentrates on time-inhomogeneous (or nonstationary) finite-horizon MDPs. 
Throughout the paper, 
we employ $\mathcal{S}=\{1,\ldots,S\}$ to denote the state space, 
$\mathcal{A}=\{1,\ldots,A\}$ the action space, and $H$ the planning horizon. 
The notation $P=\big\{P_h: \mathcal{S}\times \mathcal{A} \rightarrow \Delta(\mathcal{S}) \big\}_{1\leq h\leq H}$ denotes the probability transition kernel of the MDP; 
for any current state $s$ at any step $h$, if action $a$ is taken, then the state at the next step $h+1$ of the environment is randomly drawn from $P_{s,a,h}\coloneqq P_h(\cdot \mymid s,a) \in \Delta(\mathcal{S})$.   
Also, the notation $R=\big\{R_{s,a,h}\in \Delta([0,H]) \big\}_{1\leq h\leq H, s\in \mathcal{S}, a\in \mathcal{A}}$ 
indicates the reward distribution; that is,  
while executing action $a$ in state $s$ at step $h$, 
the agent receives an immediate reward --- which is non-negative and possibly stochastic --- drawn from the distribution $R_{s,a,h}$. 
We shall also denote by $r=\big\{r_h(s,a) \big\}_{1\leq h\leq H, s\in \mathcal{S}, a\in \mathcal{A}}$ the mean reward function, 
so that $r_{h}(s,a)\coloneqq \mathbb{E}_{r'\sim R_{s,a,h}}[r'] \in [0,H]$ for any $(s,a,h)$-tuple.     
Additionally, a deterministic policy $\pi=\{\pi_h: \mathcal{S}\rightarrow \mathcal{A}  \}_{1\leq h\leq H}$ 
stands for an action selection rule, so that the action selected in state $s$ at step $h$ is given by $\pi_h(s)$. 
The readers can consult standard textbooks (e.g., \citet{bertsekas2019reinforcement}) for more extensive descriptions.

In each episode, 
a trajectory $(s_1,a_1,r_1', s_2,\ldots,s_H, a_H, r_H')$ is rolled out as follows: 
the learner starts from an initial state $s_1$ independently drawn 
from some fixed (but unknown) distribution $\mu\in \Delta(\mathcal{S})$;  
for each step $1\leq h\leq H$, 
the learner takes action $a_h$, gains an immediate reward $r_h'  \sim R_{s_h,a_h,h}$, 
and the environment transits to the state $s_{h+1}$ at step $h+1$ according to $P_{s_h,a_h,h}$.  
Note that both the reward and the state transition are independently drawn from their respective distributions, depending solely on the current state-action-step triple but not any previous outcomes. 
All of our results in this paper operate under the following assumption on the total reward. 
\begin{assumption}\label{assum1}
	For any possible trajectory $(s_1,a_1,r_1', \ldots,s_H, a_H, r_H')$, one always has $0\leq \sum_{h=1}^H r_h'\leq H$.
\end{assumption}
\noindent 
As can be easily seen, Assumption~\ref{assum1} is less stringent than another common choice that assumes 
$r_h'\in [0,1]$ for any $h$ in any episode. 
In particular, Assumption~\ref{assum1} allows for sparse and spiky rewards along an episode; 
more discussions can be found in \citep{jiang2018open,wang2020long}.

\paragraph{Value function and Q-function.} 
For any given policy $\pi$, one can define the value function $V^{\pi}=\{V_h^{\pi}: \mathcal{S}\rightarrow \mathbb{R}\}$ 
and the $Q$-function $Q^{\pi}=\{Q_h^{\pi}: \mathcal{S}\times \mathcal{A}\rightarrow \mathbb{R}\}$ such that
\begin{subequations}
\begin{align}
	V_h^{\pi}(s) &\coloneqq \mathbb{E}_{\pi}\left[\sum_{j=h}^H r_{j}' \,\Big|\, s_h=s\right], \qquad &&\forall (s,h)\in \mathcal{S}\times [H], \\
	Q_h^{\pi}(s,a) &\coloneqq \mathbb{E}_{\pi}\left[\sum_{j=h}^H r_{j}' \,\Big|\, (s_h,a_h)=(s,a)\right],
	\qquad &&\forall (s,a,h)\in \mathcal{S}\times \mathcal{A}\times [H],
\end{align}
\end{subequations}
where the expectation $\mathbb{E}_{\pi}[\cdot]$ is taken over the randomness of an episode $\big\{(s_h,a_h,r_h')\big\}_{1\leq h\leq H}$ 
generated under policy $\pi$, 
that is, $a_{j} = \pi_{j}(s_{j})$ for every $h\leq  j \leq H$ (resp.~$h<  j \leq H$) is chosen 
in the definition of $V_h^{\pi}$ (resp.~$Q_h^{\pi}$). 
Accordingly, we define the optimal value function and the optimal $Q$-function respectively as: 
\begin{subequations}
\begin{align}
	V_h^{\star}(s) &\coloneqq \max_{\pi}V^{\pi}_h(s) , \qquad &&\forall (s,h)\in \mathcal{S}\times [H],\\
	Q_h^{\star}(s,a) &\coloneqq \max_{\pi}Q^{\pi}_h(s,a) \qquad &&\forall (s,a,h)\in \mathcal{S}\times \mathcal{A}\times [H]. 
\end{align}
\end{subequations}
Throughout this paper, we shall often abuse the notation by letting both $V_h^{\pi}$ and $V_h^{\star}$ (resp.~$Q_h^{\pi}$ and $Q_h^{\star}$) represent $S$-dimensional (resp.~$SA$-dimensional) vectors containing all elements of the corresponding value functions (resp.~Q-functions).  
Two important properties are worth mentioning: (a) the optimal value and the optimal Q-function are linked by the Bellman equation:
\begin{align}
Q_h^{\star}(s,a) = r_h(s,a)+\big\langle P_{h,s,a}, V_{h+1}^{\star} \big\rangle,\qquad V_h^{\star}(s) = \max_{a'}Q_h^{\star}(s,a'), 
\qquad \forall (s,a,h)\in \mathcal{S}\times \mathcal{A}\times[H]; 
\end{align}
(b) there exists a deterministic policy, denoted by $\pi^{\star}$, that achieves optimal value functions and Q-functions for all state-action-step tuples simultaneously, that is, 
\[
	V_h^{\pi^{\star}}(s)=V_h^{\star}(s) \qquad \text{and} \qquad 
	Q_h^{\pi^{\star}}(s,a)=Q_h^{\star}(s,a),\qquad \forall (s,a,h)\in \mathcal{S}\times \mathcal{A}\times[H]. 
\]

\paragraph{Data collection protocol and performance metrics.} 
During the learning process, 
the learner is allowed to collect $K$ episodes of samples (using arbitrary policies it selects). 
More precisely, in the $k$-th episode, 
the learner is given an independently generated initial state $s_1^k \sim \mu$, 
and executes policy $\pi^k$ (chosen based on data collected in previous episodes) to obtain a sample trajectory 
$\big\{ (s_h^k,a_h^k,r_h^k) \big\}_{1\leq h\leq H}$, 
with $s_h^k$, $a_h^k$ and $r_h^k$ denoting the state, action and immediate reward at step $h$ of this episode.

To evaluate the learning performance, 
a widely used metric is the (cumulative) regret over all $K$ episodes:   
\begin{align}
	\mathsf{Regret}(K) \coloneqq \sum_{k=1}^K \left( V^{\star}_{1}(s_1^k) -V^{\pi^k}_1(s_1^k) \right),
	\label{eq:defn-regret}
\end{align}
and our goal is to design an online RL algorithm that minimizes $\mathsf{Regret}(K)$ regardless of the allowable sample size $K$. 
It is also well-known (see, e.g., \citet{jin2018q}) that 
a regret bound can often be readily translated into a PAC sample complexity result, 
the latter of which counts the number of episodes needed to find an $\varepsilon$-optimal policy $\widehat{\pi}$ in the sense that $\mathbb{E}_{s_1 \sim \mu}\big[V_1^{\star}(s_1) - V^{\widehat{\pi}}_1(s_1)\big] \le \varepsilon$.  
For instance, the reduction argument in \citet{jin2018q} reveals that: if an algorithm achieves  $\mathsf{Regret}(K)\leq f(S,A,H) K^{1-\alpha}$ for some function $f$ and some parameter $\alpha \in (0,1)$, then by randomly selecting a policy from $\{\pi^k\}_{1\leq k\leq K}$ as $\widehat{\pi}$ one achieves  $\mathbb{E}_{s_1 \sim \mu}\big[V_1^{\star}(s_1) - V_1^{\widehat{\pi}}(s_1)\big] \lesssim f(S,A,H) K^{-\alpha}$, 
thus resulting in a sample complexity bound of $\big(\frac{f(S,A,H)}{\varepsilon}\big)^{1/\alpha}$.

\section{A model-based algorithm: Monotonic Value Propagation}\label{sec:alg}

In this section, 
we formally describe our algorithm: a simple variation of the model-based algorithm called {\em Monotonic Value Propagation} proposed by \citet{zhang2020reinforcement}. 
We present the full procedure in Algorithm~\ref{alg:main}, and point out several key ingredients. 
\begin{itemize}
	\item {\em Optimistic updates using upper confidence bounds (UCB).} The algorithm implements the optimism principle in the face of uncertainty 
		by adopting the frequently used UCB-based framework (see, e.g., $\mathtt{UCBVI}$ by \citet{azar2017minimax}). 
More specifically, 
the learner calculates the optimistic Bellman equation backward (from $h=H,\ldots,1$): 
		it first computes an empirical estimate $\widehat{P}=\{\widehat{P}_h\in \mathbb{R}^{SA\times S}\}_{1\leq h\leq H}$ of the transition probability kernel  
		as well as an empirical estimate $\widehat{r}= \{\widehat{r}_h \in \mathbb{R}^{SA}\}_{1\leq h\leq H}$ of the mean reward function, 
and then maintains upper estimates for the associated value function and Q-function using  
\begin{subequations}
\label{eq:Qh-Vh-UCB-informal}	
\begin{align}
Q_{h}(s,a)\, & \leftarrow\,\min\big\{\widehat{r}_{h}(s,a)+\langle\widehat{P}_{s,a,h},V_{h+1}\rangle+b_{h}(s,a),H\big\}, 
	\label{eq:Qh-UCB-informal}\\
V_{h}(s)\, & \leftarrow\, \max\nolimits_{a}Q_{h}(s,a) 
\end{align}
\end{subequations}
for all state-action pairs. 
Here, $Q_{h}$ (resp.~$V_h$) indicates the running estimate for the Q-function (resp.~value function), 
whereas $b_{h}(s,a)\geq 0$ is some suitably chosen bonus term that compensates for the uncertainty.  
The above opportunistic Q-estimate in turn allows one to obtain a policy estimate (via a simple greedy rule), 
which will then be executed to collect new data. 
The fact that we first estimate the model (i.e., the transition kernel and mean rewards) makes it a model-based approach.
Noteworthily, the empirical model $(\widehat{P},\widehat{r})$ shall be updated multiple times as new samples continue to arrive, 
and hence the updating rule \eqref{eq:Qh-Vh-UCB-informal} will be invoked multiple times as well.

\end{itemize}

\begin{algorithm}[H]
	\DontPrintSemicolon
\caption{Monotonic Value Propagation ($\mathtt{MVP}$)~\citep{zhang2020reinforcement}\label{alg:main}}
	\textbf{input:} state space $\mathcal{S}$, action space $\mathcal{A}$, horizon $H$, total number of episodes $K$, confidence parameter $\delta$, 
	$c_1=\frac{460}{9}$, $c_2 = 2\sqrt{2}$, $c_3=\frac{544}{9}$. \\
	\textbf{initialization: } set $\delta' \leftarrow \frac{\delta}{200SAH^2K^2}$, and for all $(s,a,s',h)\in \mathcal{S}\times \mathcal{A}\times\mathcal{S}\times [H]$, set $\theta_h(s,a)\leftarrow 0$, $\kappa_h(s,a)\leftarrow 0$, $N^{\mathsf{all}}_h(s,a,s')\leftarrow 0$, $N_h(s,a,s')\leftarrow 0$, $N_h(s,a)\leftarrow 0$, $Q_h(s,a)\leftarrow H$, $V_h(s)\leftarrow H$. \\
	\For{$k=1,2,\ldots,K$} {
		Set $\pi^k$ such that $\pi_h^k(s) = \arg\max_{a}Q_h(s,a)$ for all $s\in \mathcal{S}$ and $h\in [H]$. {\color{blue}\tcc{policy update.}}
		\For {$h=1,2,...,H$} {
			Observe $s_{h}^k$, 
			take action $ a_h^k= \arg\max_{a}Q_h(s_h^k,a)$, 
			receive  $r_h^k$,  observe $s_{h+1}^k$. \label{line:choose_action} 
			{\color{blue}\tcc{sampling.}}
			$(s,a,s')\leftarrow (s_h^k,a_h^k,s_{h+1}^k)$. \\
			Update $N^{\mathsf{all}}_h(s,a) \leftarrow  N^{\mathsf{all}}_h( s,a )+1$, $N_h(s,a,s') \leftarrow   N_h(s,a,s')+1$, $\theta_h(s,a)\leftarrow \theta_h(s,a)+r_h^k$, $\kappa_h(s,a)\leftarrow \kappa_h(s,a)+(r_h^k)^2$. \\
		{\color{blue}\tcc{perform updates using data of this epoch.\label{line:a1}}}
		\If{$N^{\mathsf{all}}_h(s,a)\in \{1,2,\ldots, 2^{\log_2K}\}$ \label{line:rp_update_start} }   {\label{line:trigger-set}
			 $N_h(s,a)\leftarrow \sum_{\widetilde{s}}N_h(s,a,\widetilde{s})$.  
			{\color{blue}\tcp{number of visits to $(s,a,h)$ in this epoch.} \label{line:Nh-update}}
			$\widehat{r}_h(s,a)\leftarrow \frac{\theta_h(s,a)}{N_h(s,a)}$. \label{line:r-hsa-update}
			{\color{blue}\tcp{empirical rewards of this epoch.}} 
			$\widehat{\sigma}_h(s,a)\leftarrow \frac{\kappa_h(s,a)}{N_h(s,a) }  $. 
			{\color{blue}\tcp{empirical squared rewards of this epoch.}}
			$\widehat{P}_{s,a,h}(\widetilde{s}) \leftarrow  \frac{N_h(s,a,\widetilde{s})}{N_h(s,a)}$ for all $\widetilde{s} \in \mathcal{S}$.  \label{line:P-hsa-update}
			{\color{blue}\tcp{empirical transition for this epoch.}}
			Set TRIGGERED = TRUE, 
			and $\theta_h(s,a)\leftarrow 0$, $\kappa_h(s,a)\leftarrow 0$,  $N_h(s,a,\widetilde{s})\leftarrow 0$  for all $\widetilde{s}\in \mathcal{S}$. 
		}
		}
		{\color{blue}\tcc{optimistic Q-estimation using empirical model of this epoch.}}
		\If {\textnormal{TRIGGERED= TRUE}} {
			Set TRIGGERED = FALSE, and $V_{H+1}(s)\leftarrow 0$ for all $s\in \mathcal{S}$. \\
			\For{$h=H,H-1,...,1$} {
				\For{$(s,a)\in \mathcal{S}\times \mathcal{A}$} {

					\begin{align} 
						\vspace{-3ex}
						b_h(s,a) &\leftarrow c_1 \sqrt{\frac{   \mathbb{ V}(\widehat{P}_{s,a,h} ,V_{h+1}) \log \frac{1}{\delta'}  }{ \max\{N_h(s,a),1 \} }}+c_2 \sqrt{\frac{\big(\widehat{\sigma}_h(s,a)- (\widehat{r}_h(s,a))^2 \big)\log \frac{1}{\delta'}}{\max\{N_h(s,a),1\}}} \qquad\qquad\qquad\qquad\qquad\qquad\nonumber\\
						&\qquad\qquad\qquad +c_3\frac{H\log \frac{1}{\delta'}}{ \max\{N_h(s,a) ,1\}  },  \label{eq:update1}  \\
						Q_h(s,a) &\leftarrow \min\big\{    \widehat{r}_h(s,a)+\langle \widehat{P}_{s,a,h}, V_{h+1} \rangle +b_h(s,a)    ,H\big\},\,
						V_{h}(s) \leftarrow \max_{a}Q_h(s,a).
						\label{eq:updateq}
					\end{align}
					\vspace{-3ex}
				}
			}
		}
	}
\end{algorithm}

\begin{itemize}

	\item {\em An epoch-based procedure and a doubling trick.} 
		Compared to the original $\mathtt{UCBVI}$ \citep{azar2017minimax}, 
		one distinguishing feature of $\mathtt{MVP}$ is to update the empirical transition kernel and empirical rewards in an epoch-based fashion, 
		as motivated by a doubling update framework adopted in \citet{jaksch2010near}. 
		More concretely, the whole learning process is divided into consecutive epochs via a simple doubling rule; 
		namely, whenever there exits a $(s,a,h)$-tuple whose visitation count reaches a power of 2, we end the current epoch,  reconstruct the empirical model (cf.~lines~\ref{line:r-hsa-update} and \ref{line:P-hsa-update} of Algorithm~\ref{alg:main}), compute the Q-function and value function using the newly updated transition kernel and rewards (cf.~\eqref{eq:updateq}), and then start a new epoch with an updated sampling policy. 
		This stands in stark contrast with the original $\mathtt{UCBVI}$, which computes new estimates for the transition model, Q-function and value function in every episode. With this doubling rule in place, the estimated transition probability vector for each $(s,a,h)$-tuple will be updated by no more than $\log_2K$ times, 
		a feature that plays a pivotal role in significantly reducing some sort of covering number needed in our covering-based analysis (as we shall elaborate on shortly in Section~\ref{sec:tec}). In each epoch, the learned policy is induced by the optimistic Q-function estimate --- computed based on the empirical transition kernel of the {\em current} epoch --- which will then be employed to collect samples in {\em all} episodes of the next epoch. 
		More technical explanations of the doubling update rule will be provided in Section~\ref{sec:tec1}.

	\item {\em Monotonic bonus functions.} Another crucial step in order to ensure near-optimal regret lies in careful designs of the data-driven bonus terms $\{b_h(s,a)\}$ in \eqref{eq:Qh-UCB-informal}. 
		Here, we adopt the monotonic Bernstein-style bonus function for $\mathtt{MVP}$ originally proposed in \citet{zhang2020reinforcement}, 
		to be made precise in \eqref{eq:update1}. 
		Compared to the bonus function in $\mathtt{Euler}$~\citep{zanette2019tighter} and $\mathtt{UCBVI}$~\citep{azar2017minimax},  the monotonic bonus form has a cleaner structure that effectively avoids large lower-order terms. Note that in order to enable variance-aware regret, we also need to keep track of the empirical variance of the (stochastic) immediate rewards. 

\end{itemize}

\begin{remark}
We note that a doubling update rule has also been used in the original $\mathtt{MVP}$ \citep{zhang2020reinforcement}. 
A subtle difference between our modified version and the original one lies in that: when the visitation count for some $(s,a,h)$ reaches $2^i$ for some integer $i\geq 1$, we only use the second half of the samples (i.e., the $\{2^{i-1}+l\}_{l=1}^{2^{i-1}}$-th samples) to compute the empirical model, whereas the original $\mathtt{MVP}$ makes use of all the $2^i$ samples. This modified step turns out to be helpful in our analysis, 
	while still preserving sample efficiency in an orderwise sense (since the latest batch always contains at least half of the samples). 
\end{remark}

\section{Key technical innovations}\label{sec:tec}

In this section, we point out the key technical hurdles the previous approach encounters when mitigating the burn-in cost, 
and put forward a new strategy to overcome such hurdles. 
For ease of presentation, let us introduce a set of augmented notation to indicate several running iterates in Algorithm~\ref{alg:main}, 
which makes clear the dependency on the episode number $k$ and will be used throughout all of our analysis. 
\begin{itemize}


	\item  $\widehat{P}^k_{s,a,h}\in \mathbb{R}^{S}$: the latest update of the empirical transition probability vector $\widehat{P}_{s,a,h}$ before the $k$-th episode. 
	
	\item  $\widehat{r}_h^k(s,a)\in [0,H]$: the latest update of the empirical reward $\widehat{r}_h(s,a)$ before the $k$-th episode. 

	\item  $\widehat{\sigma}_h^k(s,a) \in [0,H^2] $: the latest update of the empirical squared reward $\widehat{\sigma}_h(s,a)$ before the $k$-th episode.

	\item  $b_h^k(s,a) \geq 0 $: the latest update of the bonus term $b_h(s,a)$ before the $k$-th episode.

	\item  $N^{k,\mathsf{all}}_{h}(s,a)$: the total visitation count of the $(s,a,h)$-tuple before the beginning of the $k$-th episode.

	\item  $N_h^k(s,a)$: the visitation count $N_h(s,a)$ of the $(s,a,h)$-tuple of the latest doubling batch used to compute $\widehat{P}_{s,a,h}$ before the $k$-th episode.  
		When $N^{k,\mathsf{all}}_{h}(s,a)=0$, we define $N_h^k(s,a)=1$ for ease of presentation.  

	\item  $V_h^k\in \mathbb{R}^{S}$: the value function estimate $V_h$ before the beginning of the $k$-th episode. 
	\item  $Q_h^k\in \mathbb{R}^{SA}$: the Q-function estimate $Q_h$ before the beginning of the $k$-th episode. 
	
\end{itemize}
Another notation for the empirical transition probability vector is also introduced below: 
\begin{itemize}
	\item For any $j\geq 2$ (resp.~$j=1$), let  $\widehat{P}^{(j)}_{s,a,h}$ be the empirical transition probability vector for $(s,a,h)$ computed using the $j$-th batch of data, i.e., the $\{2^{j-2}+i\}_{i=1}^{2^{j-2}}$-th samples (resp.~the 1st sample) for $(s,a,h)$.  
	For completeness, we take $\widehat{P}^{(0)}_{s,a,h} = \frac{1}{S} 1$ for the $0$-th batch. 
	
	\item Similarly, let $\widehat{r}^{(j)}_h(s,a)$ (resp.~$\widehat{\sigma}^{(j)}_h(s,a)$) denote the empirical reward (resp.~empirical squared reward) w.r.t.~$(s,a,h)$ based on the $j$-th batch of data. 

\end{itemize}

\subsection{Technical barriers in prior theory for $\mathtt{UCBVI}$} 
\label{sec:technical-barrier-prior}

%

Let us take a close inspection on prior regret analysis for UCB-based model-based algorithms, in order to illuminate the part that calls for novel analysis. 
To simplify presentation, this subsection assumes deterministic rewards so that each empirical reward is replaced by its mean.


Let us look at the original $\mathtt{UCBVI}$ algorithm proposed by \citet{azar2017minimax}. 
Standard decomposition arguments employed in the literature (e.g., \citet{jaksch2010near,azar2017minimax,zhang2020reinforcement}) decompose the regret as follows:  
\begin{align}
\mathsf{Regret}(K)\leq & \sum_{k,h}\Big(\widehat{P}_{s_{h}^{k},a_{h}^{k},h}^{k,\mathsf{all}}-P_{s_{h}^{k},a_{h}^{k},h}\Big)V_{h+1}^{k}+\sum_{k,h}b_{h}^{k}\big(s_{h}^{k},a_{h}^{k}\big)\nonumber\\& +\sum_{k,h}\Big(P_{s_{h}^{k},a_{h}^{k},h}-e_{s_{h+1}^{k}}\Big)\Big(V_{h+1}^{k}-V_{h+1}^{\pi^{k}}\Big);
	\label{eq:key}
\end{align}
see also the derivation in Section~\ref{app:thmmain}. 
Here, we abuse the notation by letting $V_{h+1}^k$ (resp.~$b_h^k$) be the value function estimate (resp.~bonus term) of $\mathtt{UCBVI}$ before the $k$-th episode, 
and in the meantime, we let  $\widehat{P}^{k,\mathsf{all}}_{s,a,h}$ 
represent the empirical transition probability 
for the ($s,a,h$)-tuple computed using {\em all} samples before the $k$-th episode (note that we add the superscript $\mathsf{all}$ to differentiate it from its counterpart in our algorithm).  
In order to achieve full-range optimal regret, 
one needs to bound the three terms on the right-hand side of \eqref{eq:key} carefully, among which two are easy to handle.    
\begin{itemize}
	\item It is known that the second term (i.e., the aggregate bonus) on the right-hand side of \eqref{eq:key} can be controlled in a rate-optimal manner if we adopt suitably chosen Bernstein-style bonus; see, e.g., \citet{zhang2020reinforcement}, which will also be made clear shortly in Section~\ref{app:thmmain}. 

	\item In the meantime, the third term on the right-hand side of \eqref{eq:key} can be easily coped with by means of standard martingale concentration bounds (e.g., the Freedman inequality). 

\end{itemize}




It then comes down to controlling the first term on the right-hand side of \eqref{eq:key}.  
This turns out to be the most challenging part, owing to the complicated statistical dependency between $\widehat{P}^{k,\mathsf{all}}_{s_h^k,a_h^k,h}$ and $V_{h+1}^k$. 
To see this, note that $\widehat{P}^{k,\mathsf{all}}_{s,a,h}$ is constructed based on {\em all} previous samples of $(s,a,h)$, which has non-negligible influences upon $V_{h+1}^k$ as $V_{h+1}^k$ is computed based on previous samples. 
At least two strategies have been proposed to circumvent this technical difficulty, which we take a moment to discuss. 
\begin{itemize}
	\item {\em Strategy 1: replacing $V_{h+1}^k$ with $V_{h+1}^{\star}$ for large $k$.} 
Most prior analysis for model-based algorithms \citep{azar2017minimax,dann2017unifying,zanette2019tighter,zhang2020reinforcement} decomposes  
\begin{align}
	&\sum_{k,h} \Big(\widehat{P}^{k,\mathsf{all}}_{s_h^k,a_h^k,h}-P_{s_h^k,a_h^k,h}\Big)V_{h+1}^k 
	\nonumber\\
	&\qquad \qquad 
	=\sum_{k,h} \Big(\widehat{P}^{k,\mathsf{all}}_{s_h^k,a_h^k,h}-P_{s_h^k,a_h^k,h}\Big)V_{h+1}^{\star}+ \sum_{k,h} \Big(\widehat{P}^{k,\mathsf{all}}_{s_h^k,a_h^k,h}-P_{s_h^k,a_h^k,h}\Big)\big(V_{h+1}^k-V_{h+1}^{\star}\big).
	\label{eq:decompose-past-work}
\end{align}
The rationale behind this decomposition is as follows: 
\begin{itemize}
	\item [(i)] given that $V_{h+1}^{\star}$ is fixed and independent from the data, the first term on the right-hand side of \eqref{eq:decompose-past-work} can be bounded easily using Freedman's inequality; 
	\item [(ii)] the second term on the right-hand side of \eqref{eq:decompose-past-work} would vanish as $V_{h+1}^{k}$ and $V_{h+1}^{\star}$ become exceedingly close (which would happen as $k$ becomes large enough). 
\end{itemize}
Such arguments, however, fall short of tightness when analyzing the initial stage of the learning process: given that $V_{h+1}^k-V_{h+1}^{\star}$ cannot be sufficiently small at the beginning, this approach necessarily results in a huge burn-in cost.

	\item {\em Strategy 2: a covering-based argument.} Let us discuss informally another potential strategy that motivates our analysis.  
		We first take a closer look at the relationship between $\widehat{P}^{k,\mathsf{all}}_{s,a,h}$ 
		and $V_{h+1}^k$.  
		Abusing notation by letting $N^{k,\mathsf{all}}_{h}(s,a)$ be the total number of visits to a $(s,a,h)$-tuple before the $k$-th episode in $\mathtt{UCBVI}$, 
		we can easily observe that $\widehat{P}^{k,\mathsf{all}}_{s,a,h}$ and $V_{h+1}^k$ are statistically independent conditioned on the set 
		$\big\{ N^{k,\mathsf{all}}_{h}(s,a)\big\}_{(s,a,k)\in \mathcal{S}\times \mathcal{A}\times [K]}$. 
		Consequently, if we ``pretend'' that $\{N^{k,\mathsf{all}}_{h}(s,a)\}$ are pre-fixed and independent of $\{\widehat{P}^{k,\mathsf{all}}_{s,a,h}\}$, 
		then one can invoke standard concentration inequalities to obtain a high-probability bound on $\sum_{k,h} \big(\widehat{P}^{k,\mathsf{all}}_{s_h^k,a_h^k,h}-P_{s_h^k,a_h^k,h}\big)V_{h+1}^k$ in a desired manner. 
		The next step would then be to invoke a union bound over all possible configurations of $\{N^{k,\mathsf{all}}_{h}(s,a)\}$, 
		so as to eliminate the above independence assumption. 
		The main drawback of this approach, however, is that there are exponentially many (e.g., in $K$) possible choices of $\{N^{k,\mathsf{all}}_{h}(s,a)\}$,  
		inevitably loosening the regret bound.




\end{itemize}

\subsection{Our approach}\label{sec:tec1}

In light of the covering-based argument in Section~\ref{sec:technical-barrier-prior}, 
one can only hope this analysis strategy to work if substantial compression (i.e., a significantly reduced covering number) 
of the visitation counts  is plausible. 
This motivates our introduction of the doubling batches as described in Section~\ref{sec:alg}, 
so that for each $(s,a,h)$-tuple, the empirical model $\widehat{P}_{s,a,h}$  and its associated visitation count $N_{h}(s,a)$ (for the associated batch) are updated at most $\log_2K$ times (see line~\ref{line:trigger-set} of Algorithm~\ref{alg:main}). 
Compared to the original $\mathtt{UCBVI}$ that recomputes the transition model in every episode, 
our algorithm allows for significant reduction of the covering number of the visitation counts,  
thanks to its much less frequent updates.

Similar to \eqref{eq:key}, we are in need of bounding the following term when analyzing Algorithm~\ref{alg:main}: 
\begin{equation}
	\sum_{k,h} \Big(\widehat{P}^{k}_{s_h^k,a_h^k,h}-P_{s_h^k,a_h^k,h}\Big)V_{h+1}^k .
\end{equation}
In what follows, we present our key ideas that enable tight analysis of this quantity, which constitute our main technical innovations.  
The complete regret analysis for Algorithm~\ref{alg:main} is postponed to Section~\ref{app:thmmain}.

\subsubsection{Key concept: profiles}  
\label{eq:sec-profile}
 %
%
One of the most important concepts underlying our analysis for Algorithm~\ref{alg:main} is the so-called ``profile'', defined below. 
\begin{definition}[Profile]
	\label{defn:profile}
Consider any combination $\{N_{h}^{k,\mathsf{all}}(s,a)\}_{(s,a,h,k)\in \mathcal{S}\times \mathcal{A}\times [H]\times [K]}$. 
For any $k\in [K]$, define
\begin{subequations}
\begin{align}
	\forall (s,a,h)\in \mathcal{S}\times \mathcal{A}\times [H]: \quad 
	I^k_{s,a,h} \coloneqq \begin{cases} \max\big\{ j\in \mathbb{N}: 2^{j-1} \leq N_{h}^{k,\mathsf{all}}(s,a) \big\} , & \text{if }N_{h}^{k,\mathsf{all}}(s,a)>0; \\
	0, & \text{if }N_{h}^{k,\mathsf{all}}(s,a)=0. \end{cases}
\end{align}
The profile for the $k$-th episode $(1\leq k\leq K)$ and the total profile are then defined respectively as 
\begin{align}
	\mathcal{I}^k &\coloneqq \big\{I^k_{s,a,h} \big\}_{(s,a,h)\in \mathcal{S}\times \mathcal{A}\times [H]} \\
	\text{and}\qquad \mathcal{I} &\coloneqq \{\mathcal{I}^k\}_{k=1}^K.
	\label{eq:defn-profile-Ik}
\end{align}
\end{subequations}
%
%
\end{definition}
Clearly, once a total profile $\mathcal{I}$ w.r.t.~$\{N_h^{k,\mathsf{all}}(s,a)\}$ is given, one can write
\begin{equation}
	\widehat{P}^{k}_{s,a,h} = \widehat{P}^{(I^k_{s,a,h})}_{s,a,h},\qquad \forall (s,a,h,k)\in \mathcal{S}\times \mathcal{A}\times [H]\times [K].
	\label{eq:relation-profile-k}
\end{equation}
In other words, a total profile specifies all the time instances and locations when the empirical model is updated. 
Given that each $N_h^k(s,a)$ is recomputed only when the associated empirical model is updated (see line~\ref{line:Nh-update} of Algorithm~\ref{alg:main}), 
the total profile also provides a succinct representation of the set $\{N_h^k(s,a)\}$.

In order to quantify the degree of compression Definition~\ref{defn:profile} offers when representing the update times and locations, 
we provide an upper bound on the number of possible total profiles in the lemma below. 
\begin{lemma}\label{lemma:key2} 
	Suppose that $K\geq SAH\log_2K$. Then the number of all possible total profiles w.r.t.~Algorithm~\ref{alg:main} is at most $$(4SAHK)^{SAH\log_2K +1}.$$ 
\end{lemma}
\begin{proof}
	Define the following set (which will be useful in subsequent analysis as well)
\begin{equation}\mathcal{C} \coloneqq 
	\Big\{ \mathcal{I}=\{\mathcal{I}^1,\ldots,\mathcal{I}^K\} \,\Big|\, \mathcal{I}^1\leq \mathcal{I}^2\leq \cdots \leq \mathcal{I}^K, 
	\mathcal{I}^{k}\in \big\{0, 1,\cdots,\log_2K\big\}^{SAH} \text{ for all } 1\leq k\leq K \Big\}.
	\label{eq:defn-C-choice}
\end{equation}
Due to the monotonicity constraints, it is easily seen that the total profile of any set $\{N_h^k(s,a)\}$ must lie within $\mathcal{C}$. 
It then boils down to proving that  
$|\mathcal{C}|\leq (4SAHK)^{SAH\log_2K +1}$, which can be accomplished via elementary combinatorial calculations. 
The complete proof is deferred to Appendix~\ref{app:pfkey2}. 
\end{proof}
%
In comparison to using $\{N_h^{k,\mathsf{all}}(s,a)\}$ to encode all update times and locations ---  which might have exponentially many (in $K$) possibilities --- 
the use of doubling batches in Algorithm~\ref{alg:main} allows for remarkable compression (as the exponent of the number of possibilities only scales logarithmically in $K$).


\subsubsection{Decoupling the statistical dependency}
\label{sec:decoupling-all}

\paragraph{An expanded view of randomness w.r.t.~state transitions.}


%

To facilitate analysis, we find it helpful to look at a different yet closely related way to generate independent samples from a generative model.   
\begin{definition}[An expanded sample set from a generative model]\label{def:filt2}
	Let $\mathcal{D}^{\mathsf{expand}}$ be a set of $SAHK$ independent samples generated as follows: 
	for each $(s,a,h)\in \mathcal{S}\times \mathcal{A}\times [H]$, draw $K$ independent samples $(s,a,h,s'^{,(i)})$ obeying 
	$s'^{,(i)} \overset{\mathrm{ind.}}{\sim} P_{s,a,h}$ ($1\leq i\leq K$).
\end{definition}
Crucially, $\mathcal{D}^{\mathsf{expand}}$ can be viewed as an expansion of the original dataset --- denoted by $\mathcal{D}^{\mathsf{original}}$ --- collected in online learning,  
as we can couple the data collection processes of $\mathcal{D}^{\mathsf{original}}$ and $\mathcal{D}^{\mathsf{expand}}$ as follows: 
\begin{itemize}
	\item[(i)] generate $\mathcal{D}^{\mathsf{expand}}$ before the beginning of the online process;
	\item[(ii)] during the online learning process, whenever a sample needs to be drawn from $(s,a,h)$, 
		one can take an unused sample of $(s,a,h)$ from $\mathcal{D}^{\mathsf{expand}}$ without replacement. 
\end{itemize}
\noindent 
This allows one to conduct analysis alternatively based on the expanded sample set $\mathcal{D}^{\mathsf{expand}}$, 
which is sometimes more convenient (as we shall detail momentarily). 
Unless otherwise noted, all analyses in our proof assume that $\mathcal{D}^{\mathsf{original}}$ and $\mathcal{D}^{\mathsf{expand}}$ are {\em coupled} through the above simulation process.

In the sequel, we let $\widehat{P}^{(j)}_{s,a,h}$ (cf.~the beginning of Section~\ref{sec:tec})  denote the empirical probability vector based on the $j$-th batch of data from $\mathcal{D}^{\mathsf{original}}$ and $\mathcal{D}^{\mathsf{expand}}$ interchangeably, as long as it is clear from the context.

\paragraph{A starting point: a basic decomposition.} 
We now describe our approach to tackling the complicated statistical dependency between $\widehat{P}_{s,a,h}^k$ and $V_{h+1}^k$. 
To begin with, from relation~\eqref{eq:relation-profile-k} we can write
\begin{align}
 & \sum_{k=1}^{K}\sum_{h=1}^{H}\Big\langle\widehat{P}_{s_{h}^{k},a_{h}^{k},h}^{k}-P_{s_{h}^{k},a_{h}^{k},h},V_{h+1}^{k}\Big\rangle=\sum_{k=1}^{K}\sum_{h=1}^{H}\Big\langle\widehat{P}_{s_{h}^{k},a_{h}^{k},h}^{(I_{s_{h}^{k},a_{h}^{k},h}^{k,\mathsf{true}})}-P_{s_{h}^{k},a_{h}^{k},h},V_{h+1}^{k}\Big\rangle\nonumber\nonumber\\
 & =\sum_{l=0}^{\log_{2}K}\sum_{s,a,h}\bigg\langle\widehat{P}_{s,a,h}^{(l)}-P_{s,a,h},\,\sum_{k=1}^{K}\mathds1\left\{ (s_{h}^{k},a_{h}^{k})=(s,a),I_{s,a,h}^{k,\mathsf{true}}=l\right\} V_{h+1}^{k}\bigg\rangle\nonumber\nonumber\\
 & \leq\sum_{l=1}^{\log_{2}K}\sum_{s,a,h}\bigg\langle\widehat{P}_{s,a,h}^{(l)}-P_{s,a,h},\,\sum_{k=1}^{K}\mathds1\left\{ (s_{h}^{k},a_{h}^{k})=(s,a),I_{s,a,h}^{k,\mathsf{true}}=l\right\} V_{h+1}^{k}\bigg\rangle+ SAH^{2}\nonumber\nonumber\\
 & =\sum_{l=1}^{\log_{2}K}\sum_{j=1}^{2^{l-1}}\bigg\{\sum_{s,a,h}\Big\langle\widehat{P}_{s,a,h}^{(l)}-P_{s,a,h},V_{h+1}^{k_{l,j,s,a,h}}\Big\rangle\bigg\}+SAH^{2},
	\label{eq:PV-sum-decompose}
\end{align}
where $\mathcal{I}^{\mathsf{true}}=\{\mathcal{I}^{1,\mathsf{true}},\cdots,\mathcal{I}^{K,\mathsf{true}}\}$
with $\mathcal{I}^{k,\mathsf{true}}=\{I_{s,a,h}^{k,\mathsf{true}}\}$  
denotes the total profile w.r.t.~the true visitation counts in the online learning process, 
$k_{l,j,s,a,h}$ denotes the episode index of the sample that visits $(s,a,h)$ for the $(2^{l-1}+j)$-th time in the online learning process, 
and we take $V_{h+1}^k=0$ for any $k>K$. Here, the third line makes use of the fact that $0\leq V_{h+1}^{k} (s) \leq H$ for all $s\in \mathcal{S}$. 
The decomposition \eqref{eq:PV-sum-decompose} motivates us to first control the term $\sum_{s,a,h}\big\langle\widehat{P}_{s,a,h}^{(l)}-P_{s,a,h},V_{h+1}^{k_{l,j,s,a,h}}\big\rangle$, 
leading to the following 3-step analysis strategy.  
\begin{itemize}
	\item[1)] For any given total profile $\mathcal{I} \in \mathcal{C} $ and any fixed $1\leq l\leq \log_2K$, 
		develop a high-probability bound on a weighted sum taking the following form
		\begin{equation}
			\sum_{s,a,h} \Big(\widehat{P}^{(l)}_{s,a,h}-P_{s,a,h} \Big) X_{h+1,s,a},
			\label{eq:P-X-term}
		\end{equation}
		where each vector $X_{h+1,s,a}$ is any deterministic function of $\mathcal{I}$ and the samples collected for steps $h'\geq h+1$. 
		Given the statistical independence between $\widehat{P}^{(l)}_{s,a,h}$ and those samples for steps $h'\geq h+1$ (in the view of $\mathcal{D}^{\mathsf{expand}}$), 
		we can bound \eqref{eq:P-X-term} using standard martingale concentration inequalities.

	\item[2)] Take the union bound over all possible $\mathcal{I}\in \mathcal{C}$  
		--- with the aid of Lemma~\ref{lemma:key2} --- to obtain a uniform control of the term \eqref{eq:P-X-term}, 
		simultaneously accounting for all $\mathcal{I} \in \mathcal{C} $ and all associated sequences $\{X_{h+1,s,a}\}$.

	\item[3)] We then demonstrate that the above uniform bounds can be applied to the decomposition \eqref{eq:PV-sum-decompose} to obtain a desired bound. 
			%
		%

\end{itemize}

\paragraph{Main steps.} 
We now carry out the above three steps. 

\medskip \noindent
 {\em \underline{Steps 1) and 2).}} Let us first specify the types of vectors $\{X_{h,s,a}\}$ mentioned above in \eqref{eq:P-X-term}. 
 For each total profile $\mathcal{I} \in \mathcal{C}$ (cf.~\eqref{eq:defn-C-choice}), 
 consider any set $\big\{ \mathcal{X}_{h,\mathcal{I}} \big\}_{1\leq h\leq H}$ obeying: for each $1\leq h\leq H$,  
\begin{itemize}

	\item  $\mathcal{X}_{h+1,\mathcal{I}}$ is given by a {\em deterministic} function of $\mathcal{I}$ and 
		\[
			\Big\{\widehat{P}^{(I^k_{s,a,h'})}_{s,a,h'},\widehat{r}^{(I^k_{s,a,h'})}_{h'}(s,a),\widehat{\sigma}^{(I^k_{s,a,h'})}_{h'}(s,a) \Big\}_{ h< h' \leq H, (s,a,k)\in \mathcal{S}\times \mathcal{A} \times [K]};
		\]

	\item $\|X\|_{\infty}\leq H$ for each vector $X\in \mathcal{X}_{h,\mathcal{I}}$;

	\item $\mathcal{X}_{h,\mathcal{I}}$  is a set of no more than $K+1$ non-negative vectors in $\mathbb{R}^S$, and contains the all-zero vector $0$.

\end{itemize}
Given such a construction of $\big\{ \mathcal{X}_{h,\mathcal{I}} \big\}$, 
we can readily conduct  Steps 1) and 2), with a uniform concentration bound stated below.  
\begin{lemma}\label{lemma:uniform}
	Suppose that $K\geq SAH\log_2K$, and construct a set $\big\{ \mathcal{X}_{h,\mathcal{I}} \big\}_{1\leq h\leq H}$ for each $\mathcal{I} \in \mathcal{C}$ 
	satisfying the above properties. 
	 %
Then with probability at least $1- \delta'$, 
\begin{align}
	&\sum_{s,a,h\in \mathcal{S}\times \mathcal{A}\times [H]}\big\langle \widehat{P}_{s,a,h}^{(l)}-P_{s,a,h}, X_{h+1,s,a} \big\rangle 
	\leq \sum_{s,a,h\in \mathcal{S}\times \mathcal{A}\times [H]} \max\Big\{ \big\langle \widehat{P}_{s,a,h}^{(l)}-P_{s,a,h}, X_{h+1,s,a} \big\rangle, 0 \Big\}
	\nonumber\\
	&\qquad\quad \leq \sqrt{\frac{8}{2^{l-2}}\sum_{s,a,h}\mathbb{V}\big(P_{s,a,h},X_{h+1,s,a}\big)\left(6SAH\log_{2}^{2}K+\log\frac{1}{\delta'}\right)} \notag\\
	&\qquad\qquad\qquad\qquad+\frac{4H}{2^{l-2}}\left(6SAH\log_{2}^{2}K+\log\frac{1}{\delta'}\right)\label{eq:xx1-aux-1}
\end{align}
holds simultaneously for all $\mathcal{I} \in \mathcal{C}$, all $2\leq l\leq\log_{2}K+1$, and all sequences $\{X_{h,s,a}\}_{(s,a,h)\in\mathcal{S}\times\mathcal{A}\times[H]}$
obeying $X_{h,s,a}\in\mathcal{X}_{h+1,\mathcal{I}},$ $\forall (s,a,h)\in \mathcal{S}\times\mathcal{A}\times[H]$.  Here, we recall that $\delta'=\frac{\delta}{200SAH^2K^2}$. 
\end{lemma}
\begin{proof} 
	We first invoke the Freedman inequality to bound the target quantity for any fixed $\mathcal{I}\in \mathcal{C}$, any fixed integer $l$, and 
	any fixed feasible sequence $\{X_{h,s,a}\}$, 
	before applying the union bound to establish uniform control. 
See Appendix~\ref{app:pfuniform} for details. \end{proof}

\medskip \noindent
 {\em \underline{Step 3).}} 
Next, we turn to Step 3), which is accomplished via the following lemma. 
Note that we also provide upper bounds for two additional quantities: $\sum_{k,h}\max\big\{\big\langle\widehat{P}_{s_{h}^{k},a_{h}^{k},h}^{k}-P_{s_{h}^{k},a_{h}^{k},h},V_{h+1}^{k}\big\rangle,0\big\}$ 
and $\sum_{k,h}\big\langle\widehat{P}_{s_{h}^{k},a_{h}^{k},h}^{k}-P_{s_{h}^{k},a_{h}^{k},h},\big(V_{h+1}^{k}\big)^{2}\big\rangle$, 
which will be useful in subsequent analysis. 
\begin{lemma}\label{lemma:decouple}
Suppose that $K\geq SAH\log_2K$. With probability exceeding $1-\delta'$, we have
\begin{align*} 
	& \sum_{k=1}^{K}\sum_{h=1}^{H}\Big\langle\widehat{P}_{s_{h}^{k},a_{h}^{k},h}^{k}-P_{s_{h}^{k},a_{h}^{k},h},V_{h+1}^{k}\Big\rangle
	\leq \sum_{k=1}^{K}\sum_{h=1}^{H}\max\Big\{\Big\langle\widehat{P}_{s_{h}^{k},a_{h}^{k},h}^{k}-P_{s_{h}^{k},a_{h}^{k},h},V_{h+1}^{k}\Big\rangle,0\Big\} \\
	& \qquad\leq \sqrt{16(\log_{2}K)\sum_{k=1}^{K}\sum_{h=1}^{H}\mathbb{V}\big(P_{s_{h}^{k},a_{h}^{k},h},V_{h+1}^{k}\big)\left(6SAH\log_{2}^{2}K+\log\frac{1}{\delta'}\right)} \\
	&\qquad\qquad\qquad\qquad +49SAH^{2}\log_{2}^{3}K+8H(\log_{2}K)\log\frac{1}{\delta'}
\end{align*}
and
\begin{align*} 
	& \sum_{k=1}^{K}\sum_{h=1}^{H}\Big\langle\widehat{P}_{s_{h}^{k},a_{h}^{k},h}^{k}-P_{s_{h}^{k},a_{h}^{k},h},\big(V_{h+1}^{k}\big)^{2}\Big\rangle\\
 & \qquad\leq8H\sqrt{(\log_{2}K)\sum_{k=1}^{K}\sum_{h=1}^{H}\mathbb{V}\big(P_{s_{h}^{k},a_{h}^{k},h},V_{h+1}^{k}\big)\left(6SAH\log_{2}^{2}K+\log\frac{1}{\delta'}\right)}\\
	&\qquad\qquad\qquad+49SAH^{3}\log_{2}^{3}K+8H^2(\log_{2}K)\log\frac{1}{\delta'}.
\end{align*}
\end{lemma}
\begin{proof} This result is proved by combining the uniform bound in Lemma~\ref{lemma:uniform} with the decomposition~\eqref{eq:PV-sum-decompose}. 
See Appendix~\ref{app:pfdecouple}. \end{proof}

Thus far, we have obtained high-probability bounds on the most challenging terms. The complete proof of Theorem~\ref{thm1} will be presented next in Section~\ref{app:thmmain}.

\section{Proof of Theorem~\ref{thm1}}\label{app:thmmain}

This section is devoted to proving Theorem~\ref{thm1}. 
For notational convenience, let $B$ be a logarithmic term
\begin{equation}
	B=4000 (\log_2 K)^3 \log(3SAH)\log\frac{1}{\delta'} 
	\label{eq:assumption-K-proof-B}, 
\end{equation}
where we recall that $\delta$ is the confidence parameter in Algorithm~\ref{alg:main} and $\delta' = \frac{\delta}{200SAH^2K^2}$. 
When $K\leq BSAH$, the claimed result in Theorem~\ref{thm1} holds trivially since
\[
\mathsf{Regret}(K)=\sum_{k=1}^{K}\left(V_{1}^{\star}(s_{1}^{k})-V_{1}^{\pi^{k}}(s_{1}^{k})\right)
\leq HK 
= \min\left\{ \sqrt{BSAH^{3}K},HK\right\}.
\]
As a result, it suffices to focus on the scenario with 
\begin{equation}
	K\geq BSAH \qquad \text{with } B=4000 (\log_2 K)^3 \log(3SAH)\log\frac{1}{\delta'} .
	\label{eq:assumption-K-proof}
\end{equation}

Our regret analysis for Algorithm~\ref{alg:main} consists of several steps described below. 

\paragraph{Step 1: the optimism principle.} 
To begin with, we justify that the running estimates of Q-function and value function in Algorithm~\ref{alg:main} are always upper bounds on the optimal Q-function and the optimal value function,  respectively,  
thereby guaranteeing optimism in the face of uncertainty. 
\begin{lemma}[Optimism]\label{lemma:opt}
With probability exceeding $1-4SAHK\delta'$, one has
\begin{equation}
	Q_h^k(s,a)\geq Q_h^{\star}(s,a) \qquad  \text{and}  \qquad V^k_h(s)\geq V^{\star}_h(s)
\end{equation}
for all $(s,a,h,k)$. 
 \end{lemma}
 \begin{proof} See Appendix~\ref{sec:proof-lemma:opt}. \end{proof}


\paragraph{Step 2: regret decomposition.}
In view of the optimism shown in Lemma~\ref{lemma:opt}, 
the regret can be upper bounded by 
\begin{align}
\mathsf{Regret}(K) & =\sum_{k=1}^{K}\big(V_{1}^{\star}(s_{1}^{k})-V_{1}^{\pi^{k}}(s_{1}^{k})\big)\leq\sum_{k=1}^{K}\big(V_{1}^{k}(s_{1}^{k})-V_{1}^{\pi^{k}}(s_{1}^{k})\big)
	\label{eq:regret-UB1}
\end{align}
with probability at least $1-4SAHK\delta'$. 
In order to control the right-hand side of \eqref{eq:regret-UB1}, 
we first make note of the following upper bound on $V_{1}^{k}(s_{1}^{k})$. 
\begin{lemma}\label{lemma:decomdetail}
For every $1\leq k\leq K$, one has
$$
	V_1^k(s_1^k) \leq \sum_{h=1}^{H} \left( \big\langle \widehat{P}^k_{s_h^k,a_h^k,h} - P_{s_h^k,a_h^k,h}, V_{h+1}^k \big\rangle + b_h^k(s_h^k,a_h^k) + \widehat{r}_h^k(s_h^k,a_h^k) + \big\langle P_{s_h^k,a_h^k,h}-e_{s_{h+1}^k}, V_{h+1}^k \big\rangle \right) .
$$
\end{lemma}
\begin{proof}[Proof of Lemma~\ref{lemma:decomdetail}]
From the construction of $V_h^k$ and $Q_h^k$, it is seen that, for each $1\leq h\leq H$,  
\begin{align}
	V_h^k(s_h^k) & = Q_h^k(s_h^k, a_h^k) \leq \widehat{r}_h^k(s_h^k,a_h^k) + \widehat{P}^k_{s_h^k,a_h^k,h}V_{h+1}^k + b_h^k(s_h^k,a_h^k) \nonumber
\\ &  =  \big\langle \widehat{P}^k_{s_h^k,a_h^k,h} - P_{s_h^k,a_h^k,h}, V_{h+1}^k \big\rangle + b_h^k(s_h^k,a_h^k) + \widehat{r}_h^k(s_h^k,a_h^k) + \big\langle P_{s_h^k,a_h^k,h}-e_{s_{h+1}^k}, V_{h+1}^k \big\rangle + V_{h+1}^k(s_{h+1}^k).\nonumber
\end{align}
Applying this relation recursively over $1\leq h\leq H$ gives 
\begin{align*}
 & V_1^k(s_1^k) \nonumber
 \\ & \leq  \sum_{h=1}^{H} \left( \big\langle \widehat{P}^k_{s_h^k,a_h^k,h} - P_{s_h^k,a_h^k,h}, V_{h+1}^k \big\rangle + b_h^k(s_h^k,a_h^k) 
	+ \widehat{r}_h^k(s_h^k,a_h^k) + \big\langle P_{s_h^k,a_h^k,h}-e_{s_{h+1}^k}, V_{h+1}^k \big\rangle \right) + V_{H+1}^k(s_{H+1}^k),
\end{align*}
which combined with $V_{H+1}^k=0$ concludes the proof. 
\end{proof}

Combine Lemma~\ref{lemma:decomdetail} with \eqref{eq:regret-UB1} to show that, with probability at least $1-4SAHK\delta'$, 
\begin{align}
\mathsf{Regret}(K) & \leq\underset{\eqqcolon\,T_{1}}{\underbrace{\sum_{k=1}^{K}\sum_{h=1}^{H}\big\langle\widehat{P}_{s_{h}^{k},a_{h}^{k},h}^{k}-P_{s_{h}^{k},a_{h}^{k},h},V_{h+1}^{k}\big\rangle}}+\underset{\eqqcolon\,T_{2}}{\underbrace{\sum_{k=1}^{K}\sum_{h=1}^{H}b_{h}^{k}(s_{h}^{k},a_{h}^{k})}}\nonumber\\
 & \quad+\underset{\eqqcolon\,T_{3}}{\underbrace{\sum_{k=1}^{K}\sum_{h=1}^{H}\big\langle P_{s_{h}^{k},a_{h}^{k},h}-e_{s_{h+1}^{k}},V_{h+1}^{k}\big\rangle}}+\underset{\eqqcolon\,T_{4}}{\underbrace{\sum_{k=1}^{K}\left(\sum_{h=1}^{H}\widehat{r}_{h}^{k}(s_{h}^{k},a_{h}^{k})-V_{1}^{\pi^{k}}(s_{1}^{k})\right)}}, 
	\label{eq:decomposition}
\end{align}
leaving us with four terms to control. 
In particular, $T_1$ has already been upper bounded in Section~\ref{sec:tec1}, and hence we shall describe how to bound $T_2,\ldots,T_4$ in the sequel.




\paragraph{Step 3.1: bounding the terms $T_2,T_3$ and $T_4$.}
In this section, we seek to bound the terms  $T_2,T_3$ and $T_4$ defined in the regret decomposition \eqref{eq:decomposition}.  
To do so, we find it helpful to first introduce the following quantities that capture some sort of aggregate variances: 
\begin{subequations}
\label{eq:defn-T56-proof}
\begin{align}
	T_5 &\coloneqq \sum_{k=1}^K\sum_{h=1}^H\mathbb{V}\big(\widehat{P}^k_{s_h^k,a_h^k,h},V_{h+1}^k \big),
	\label{eq:defn-T5-proof} \\
	T_6 &\coloneqq \sum_{k=1}^K \sum_{h=1}^H\mathbb{V} \big(P_{s_h^k,a_h^k,h},V_{h+1}^k \big) ,
	\label{eq:defn-T6-proof}
\end{align}
\end{subequations}
with $T_5$ denoting certain empirical variance and $T_6$ the true variance. 
With these quantities in place, we claim that the following bounds hold true. 
\begin{lemma}\label{lem:bound-T234}
With probability exceeding $1-15SAH^2K^2\delta'$, one has 
%
\begin{subequations}
\label{eq:boundt234}
\begin{align}
	T_2 &\leq 61\sqrt{2SAH(\log_2 K)\Big(\log\frac{1}{\delta'}\Big)T_5} +  8\sqrt{SAH^3K(\log_2 K)\log\frac{1}{\delta'}}+
	151 SAH^2(\log_2K)\log\frac{1}{\delta'},
	\label{eq:boundt2} \\
	|T_3|  &\leq \sqrt{ 8 T_6 \log \frac{1}{\delta'}  } + 3H\log \frac{1}{\delta'}, \label{eq:boundt3} \\
	|T_4| &\leq 6\sqrt{2SAH^3K(\log_2K)\log \frac{1}{\delta'} } +  55SAH^2(\log_2 K)\log \frac{1}{\delta'}.\label{eq:bdt_4f}
\end{align}
\end{subequations}
%
%
%
%
%
%
%
%
\end{lemma}
\begin{proof} See Appendix~\ref{app:pflem:bound-T234}. \end{proof}

\paragraph{Step 3.2: bounding the aggregate variances $T_5$ and $T_6$.} 
The previous bounds on $T_2$ and $T_3$ stated in Lemma~\ref{lem:bound-T234} depend respectively on the aggregate variance $T_5$ and $T_6$  (cf.~\eqref{eq:defn-T5-proof} and \eqref{eq:defn-T6-proof}), 
which we would like to control now. 
By introducing the following quantities: 
\begin{subequations}
\label{eq:defn-T789-proof}
\begin{align}
T_{7} & \coloneqq\sum_{k=1}^{K}\sum_{h=1}^{H}\Big\langle\widehat{P}_{s_{h}^{k},a_{h}^{k},h}^{k}-P_{s_{h}^{k},a_{h}^{k},h},\big(V_{h+1}^{k}\big)^{2}\Big\rangle,\label{eq:defn-T7-proof}\\
T_{8} & \coloneqq\sum_{k=1}^{K}\sum_{h=1}^{H}\Big\langle P_{s_{h}^{k},a_{h}^{k},h}-e_{s_{h+1}^{k}},\big(V_{h+1}^{k}\big)^{2}\Big\rangle,\label{eq:defn-T8-proof}\\
T_9 & \coloneqq \sum_{k=1}^{K}\sum_{h=1}^{H}\max\Big\{\Big\langle\widehat{P}_{s_{h}^{k},a_{h}^{k},h}^{k}-P_{s_{h}^{k},a_{h}^{k},h},V_{h+1}^{k}\Big\rangle,0\Big\}, 
\label{eq:defn-T9-proof}
\end{align}
\end{subequations}
we can upper bound $T_5$ and $T_6$ through the following lemma.  
\begin{lemma}
\label{lem:bound-T56}
With probability at least $1-4SAHK\delta'$,
\begin{subequations}
\label{eq:boundt56}
\begin{align}
T_{5} & \leq T_{7}+T_8+ 2HT_{2}+6KH^{2},
\label{eq:boundt5} \\
T_{6} & \leq 2HT_{2}+6KH^{2}
	+\sqrt{32H^{2}T_{6}\log\frac{1}{\delta'}}+3H^{2}\log\frac{1}{\delta'}
	+2HT_{9}, 
\label{eq:boundt6} \\
|T_{8}|&
	\leq \sqrt{32H^{2}T_{6}\log\frac{1}{\delta'}}+3H^{2}\log\frac{1}{\delta'}	 . 
	\label{eq:boundt8}
\end{align}
\end{subequations}
\end{lemma}
\begin{proof} See Appendix~\ref{sec:pflem:bound-T56}. \end{proof}

\paragraph{Step 3.3: bounding the terms $T_1$, $T_7$ and $T_9$.} 
Taking a look at the above bounds on $T_2,\ldots,T_6$, 
we see that one still needs to deal with the terms $T_1$, $T_7$ and $T_9$ (see \eqref{eq:decomposition}, \eqref{eq:defn-T7-proof} and \eqref{eq:defn-T9-proof}, respectively). 
As it turns out, these quantities have already been bounded in Section~\ref{sec:tec}. 
Specifically, Lemma~\ref{lemma:decouple} tells us that: with probability at least $1-\delta'$, 
%
%
\begin{subequations}
\label{eq:boundt179}
\begin{align}
	T_1\leq T_9&\leq \sqrt{ B SAH \sum_{k=1}^K \sum_{h=1}^H \mathbb{V}(P_{s_h^k,a_h^k,h},V_{h+1}^k)}+BSAH^2 = \sqrt{BSAH T_6}+BSAH^2,\label{eq:boundt1}
 \\ 
	T_7 &\leq H \sqrt{ BSAH  \sum_{k=1}^K \sum_{h=1}^H \mathbb{V}(P_{s_h^k,a_h^k,h}, V_{h+1}^k) }    + BSAH^3 = H\sqrt{BSAH T_6}+BSAH^3 ,\label{eq:boundt7}
\end{align}
\end{subequations}
where we recall that $B=4000(\log_2K)^3\log(3SAH)\log\frac{1}{\delta'} $.

%
%

%
%

\paragraph{Step 4: putting all pieces together.}
The previous bounds \eqref{eq:boundt234}, \eqref{eq:boundt56} and \eqref{eq:boundt179} indicate that: 
with probability at least $1-100SAH^2K^2\delta'$, one has
\begin{subequations}
\label{eq:all-bounds-summary}
\begin{align}
	T_2 &\leq  \sqrt{B SAHT_5} +  \sqrt{BSAH^3K}+ BSAH^2,\label{eq:obt2}
	\\  T_3 &\leq \sqrt{BT_6}+HB  ,\label{eq:obt3}
	\\  T_4 &\leq \sqrt{ BSAH^3K}+BSAH^2,\label{eq:obt4}
	\\  T_5 &\leq T_7 + T_8 + 2H T_2 + 6KH^2 ,\label{eq:obt5}
	\\  T_6 &\leq \sqrt{B H^2T_6} + 2HT_2 + 2HT_9 + BH^2 + 6KH^2,\label{eq:obt6}
	\\  T_8 &\leq \sqrt{BH^2T_6 } + BH^2 ,\label{eq:obt8}
	\\  T_1 &\leq \sqrt{BSAHT_6} + BSAH^2,\label{eq:obt1}
	\\  T_7 &\leq H\sqrt{BSAHT_6} + BSAH^3,\label{eq:obt7}
	\\  T_9 &\leq \sqrt{BSAHT_6}+BSAH^2,\label{eq:obt9}
\end{align}
\end{subequations}
where we again use $B=4000(\log_2K)^3\log(3SAH)\log\frac{1}{\delta'} $.

 To solve the inequalities \eqref{eq:all-bounds-summary}, we resort to the elementary AM-GM inequality: if $a\leq \sqrt{bc}+d$ for some $b,c\geq 0$, then it follows that $a \leq \epsilon b + \frac{1}{2\epsilon}c +d$ for any $\epsilon>0$. This basic inequality combined with  \eqref{eq:all-bounds-summary} gives
 \begin{align}
	 HT_2 &\leq \epsilon T_5 + \left(\frac{1}{2\epsilon}+1\right) BSAH^3+ \frac{3}{2}BSAH^3 + \frac{1}{2} KH^2 ,\nonumber
	 \\  T_6 &\leq \epsilon T_6 + 2HT_2 + 2HT_9 + \left(1+\frac{1}{2\epsilon}\right) BH^2 + 6KH^2,\nonumber
	 \\  HT_9 &\leq \epsilon T_6 +\left( \frac{1}{2\epsilon}+1\right)BSAH^3,\nonumber
	 \\  T_8 &\leq \epsilon T_6 + \left( \frac{1}{2\epsilon}+1\right)BH^2,\nonumber
	 \\  T_7 &\leq  \epsilon T_6 + \left(\frac{1}{2\epsilon} +1\right)BSAH^3,\nonumber
 \end{align}
which in turn result in
\begin{align}
	 T_5 &\leq T_7+T_8 + 2HT_2 + 6KH^2\leq 2\epsilon T_5+2\epsilon T_6 +  \left( \frac{1}{\epsilon}+2\right) BSAH^3+6KH^2;\nonumber
\\  T_{6} & \leq\epsilon T_{6}+2HT_{2}+2HT_{9}+\left(1+\frac{1}{2\epsilon}\right)BH^{2}+6KH^{2} 
	\leq3\epsilon T_{6}+2\epsilon T_{5}+\left(\frac{3}{\epsilon}+8\right)BSAH^{3}+7KH^{2}. \notag
\end{align}
By taking $\epsilon= 1/20$, we arrive at 
\begin{align}
	T_5+T_6\lesssim BSAH^3+KH^2 \asymp KH^2,
\end{align}
where the last relation holds due to our assumption $K\geq SAHB$ (cf.~\eqref{eq:assumption-K-proof}).  
Substituting this into \eqref{eq:all-bounds-summary}  yields
\begin{align}
	T_1 \lesssim \sqrt{BSAH^3K},
	\quad T_2 \lesssim \sqrt{BSAH^3K}, \quad  T_3 \lesssim \sqrt{BKH^2}
	\quad \text{and} \quad
	T_4 &\lesssim \sqrt{ BSAH^3K},
\end{align}
provided that $K\geq SAHB$. 
These bounds taken collectively with \eqref{eq:decomposition} readily give
\[
	\mathsf{Regret}(K) \lesssim \sqrt{ BSAH^3K} . 
\]

Combining the two scenarios (i.e., $K\geq BSAH$ and $K\leq BSAH$) reveals that with probability at least $1-100SAH^2K^2\delta'$, 
\[
	\mathsf{Regret}(K) \lesssim \min\big\{ \sqrt{ BSAH^3K} , HK \big\}
	\lesssim  
	\min\bigg\{ 
	\sqrt{ BSAH^3K \log^5 \frac{SAHK}{\delta'}} , HK \bigg\}.
\]
The proof of Theorem~\ref{thm1} is thus completed by recalling that $\delta' = \frac{\delta}{200SAH^2K^2}$.

\section{Extensions}
\label{sec:extensions}

In this section,  we develop more refined regret bounds for Algorithm~\ref{alg:main} in order to reflect the role of several problem-dependent quantities.  
Detailed proofs are postponed to Appendix~\ref{sec:appfirst} and Appendix~\ref{app:var}.


%
%


%
\subsection{Value-based regret bounds} 
Thus far, we have not yet introduced the crucial quantity $v^{\star}$ in Theorem~\ref{thm:first}, 
which we define now. 
When the initial states are drawn from $\mu$, we define $v^{\star}$ to be the weighted optimal value: 
\begin{equation}
	v^{\star} \coloneqq \mathbb{E}_{s\sim \mu}\big[ V_1^{\star}(s) \big]. 
	\label{eq:defn-vstar-formal}
\end{equation}
Encouragingly, 
the value-dependent regret bound we develop in Theorem~\ref{thm:first} is still minimax-optimal, 
as asserted by the following lower bound. 
\begin{theorem}\label{thm:lb1} 
Consider any $p\in [0,1]$ and $K\geq 1$. 
For any learning algorithm, there exists an MDP with $S$ states, $A$ actions and horizon $H$
	obeying $v^\star\leq  Hp$ and 
	\begin{equation}
		\mathbb{E}\big[\mathsf{Regret}(K)\big]  \gtrsim \min\big\{ \sqrt{SAH^3Kp},\, KHp \big\} .
	\end{equation}
\end{theorem}
%
In fact, the construction of the hard instance (as in the proof of Theorem~\ref{thm:lb1}) is quite simple. 
Design a new branch with $0$ reward and set the probability of reaching this branch to be $1-p$. 
Also, with probability $p$, we direct the learner to a hard instance with regret $\Omega(\min\{\sqrt{SAH^3Kp},KpH\})$ and optimal value $H$. This guarantees that the optimal value obeys $v^\star \leq Hp$ and that the expected regret is at least 
$$\Omega\Big(\min\big\{ \sqrt{SAH^3Kp},KHp   \big\}\Big) \gtrsim \min\big\{ \sqrt{SAH^2Kv^\star},Kv^\star   \big\}.$$
See Appendix~\ref{app:lb} for more details.

%

%
\subsection{Cost-based regret bounds} 
Next, we turn to the cost-aware regret bound as in Theorem~\ref{thm:cost}. 
Note that all other results except for Theorem~\ref{thm:cost} (and a lower bound in this subsection) are about rewards as opposed to cost. 
In order to facilitate discussion, let us first formally formulate the cost-based scenarios.

Suppose that the reward distributions $\{R_{h,s,a}\}_{(s,a,h)}$ are replaced with the cost distributions $\{C_{h,s,a}\}_{(s,a,h)}$, 
where each distribution $C_{h,s,a}\in \Delta([0,H])$ has mean $c_h(s,a)$. 
In the $h$-th step of an episode, the learner pays an immediate cost $c_h\sim C_{h,s_h,a_h}$ instead of receiving an immediate reward $r_h$,  
and the objective of the learner is instead to minimize the total cost $\sum_{h=1}^H c_h$ (in an expected sense). 
The optimal cost quantity $c^{\star}$ is then defined as 
\begin{equation}
	c^{\star} \coloneqq \min_{\pi}\mathbb{E}_{\pi,s_1\sim \mu}\bigg[\sum_{h=1}^H c_h \bigg]. 
\label{eq:defn-cstar-formal}
\end{equation}

In this cost-based setting, we find it convenient to re-define the $Q$-function and value function as follows:
\begin{align}
	 Q_{h}^{\pi,\text{cost}}(s,a) &\coloneqq \mathbb{E}_{\pi}\left[\sum_{h'=h}^H c_{h'} \,\Big|\, (s_h,a_h)=(s,a)\right] ,
	 && \forall (s,a,h)\in \mathcal{S}\times \mathcal{A}\times [H], \nonumber
	\\ 
	V_h^{\pi,\text{cost}}(s) &\coloneqq \mathbb{E}_{\pi}\left[\sum_{h'=h}^H c_{h'} \,\Big|\, s_h = s\right],
	&& \forall (s,h)\in \mathcal{S}\times \times [H], 
	\nonumber
\end{align}
where we adopt different fonts to differentiate them from the original Q-function and value function. 
The optimal cost function is then define by 
$$
	Q_h^{\star,\text{cost}}(s,a) = \min_{\pi}Q_h^{\pi,\text{cost}}(s,a)
	\qquad \text{and} \qquad
	V_h^{\star,\text{cost}}(s)=  \min_{\pi}V_h^{\pi,\text{cost}}(s).
$$
Given the definitions above, 
we overload the notation $\mathsf{Regret}(K)$ to denote the regret for the cost-based scenario as 
$$
	\mathsf{Regret}(K) \coloneqq \sum_{k=1}^K \Big( V^{\pi^k,\text{cost}}_{1}(s_1^k) -  V_1^{\star,\text{cost}}(s_1^k) \Big).
$$
One can also simply regard the cost minimization problem as  reward maximization with negative rewards by choosing $r_h = -c_h$. 
This way allows us to apply Algorithm~\ref{alg:main} directly, except that \eqref{eq:updateq} is replaced by 
\begin{align}
Q_h(s,a) \,\leftarrow\, \max\left\{\min \left\{ \widehat{r}_h(s,a) + \widehat{P}_{s,a,h}V_{h+1}+b_h(s,a), \, 0 \right\} ,\,-H \right\}.
	\label{eq:updatecost}
\end{align}
%
%
%
%
%
Note that the proof of Theorem~\ref{thm:cost} closely resembles that of Theorem~\ref{thm:first}, 
which can be found in Appendix~\ref{app:cost}.
%

To confirm the tightness of Theorem~\ref{thm:cost}, we develop the following matching lower bound, 
which resorts to a similar hard instance as in the proof of Theorem~\ref{thm:lb1}. 
\begin{theorem}\label{corollary:costlb}
Consider any $p\in [0,1/4]$ and any $K\geq 1$.
For any algorithm, one can construct an MDP with $S$ states, $A$ actions and horizon $H$ 
	obeying $c^{\star} \asymp Hp$ and 
	\[
		\mathbb{E}\big[\mathsf{Regret}(K)\big] \gtrsim \min\big\{ \sqrt{SAH^3Kp}+SAH^2, \,KH(1-p) \big\} 
		\asymp \min\big\{ \sqrt{SAH^2Kc^{\star}}+SAH^2, \,KH \big\}.
	\]
\end{theorem}
\noindent 
The proof of this lower bound can be found in Appendix~\ref{app:lbc}.

\subsection{Variance-dependent regret bound}

The final regret bound presented in Theorem~\ref{thm:var} depends on some sort of variance metrics. 
Towards this end, let us first make precise the variance metrics of interest:
\begin{itemize}
	\item[(i)] The first variance metric is defined as 
		\begin{equation}
			\mathrm{var}_1 \coloneqq \max_{\pi}\mathbb{E}_{\pi}\Bigg[\sum_{h=1}^H \mathbb{V}\big(P_{s_h,a_h,h},V_{h+1}^{\star}\big)+\sum_{h=1}^H \mathrm{Var}\big(R_h(s_h,a_h)\big) \Bigg],
			\label{eq:defn-var1}
		\end{equation}
		where $\{(s_h,a_h)\}_{1\leq h\leq H}$ represents a sample trajectory under policy $\pi$. 
		This captures the maximal possible expected sum of variance with respect to the optimal value function $\{V_{h}^{\star}\}_{h=1}^H$. 
		
	\item[(ii)] Another useful variance metric is defined as
		\begin{equation}
			\mathrm{var}_2 \coloneqq \max_{\pi,s}\mathrm{Var}_{\pi}\bigg[\sum_{h=1}^H r_h \,\Big|\, s_1=s\bigg],
			\label{eq:defn-var2}
		\end{equation}
		where $\{r_h\}_{1\leq h\leq H}$ denotes a sample sequence of immediate rewards under policy $\pi$. 
		This indicates the maximal possible variance of the accumulative reward. 
\end{itemize}
The interested reader is referred  to \citet{zhou2023sharp} for further discussion about these two metrics. 
Our final variance metric is then defined as
\begin{align}
	\mathrm{var} \coloneqq \min\big\{\mathrm{var}_1,\mathrm{var}_2 \big\} .
	\label{eq:defn-var-formal}
\end{align}
%



With the above variance metrics in mind, we can then revisit Theorem~\ref{thm:var}. 
As a special case, when the transition model is fully deterministic, the regret bound in Theorem~\ref{thm:var} simplifies to $$\mathsf{Regret}(K)\leq \widetilde{O}\big(\min\big\{SAH^2,\,HK\big\}\big)$$ for any $K\geq 1$, which is roughly the cost of visiting each state-action pair. 
The full proof of Theorem~\ref{thm:var} is postponed to Appenndix~\ref{app:var}. 

To finish up, let us develop a matching lower bound to corroborate the tightness and optimality of Theorem~\ref{thm:var}. 
\begin{theorem}\label{thm:lb3}
Consider any $p\in [0,1]$ and any $K\geq 1$. For any algorithm, one can find an MDP instance with $S$ states, $A$ actions, and horizon $H$ satisfying $\max\{\mathrm{var}_1,\mathrm{var}_2\}\leq H^2p$ and 
	$$	
		\mathbb{E}\big[\mathsf{Regret}(K)\big]  \gtrsim \min\big\{\sqrt{SAH^3Kp}+SAH^2,\,KH\big\}.
	$$
\end{theorem}

The proof of Theorem~\ref{thm:lb3} resembles that of Theorem~\ref{thm:lb1}, except that we need to construct a hard instance when $K\leq SAH/p$. For this purpose, we construct a fully deterministic MDP  (i.e., all of its transitions are deterministic and all rewards are fixed), and show that the learner has to visit about half of the state-action-layer tuples in order to learn a near-optimal policy. 
The proof details are deferred to Appendix~\ref{app:lb}.

\section{Discussion}

Focusing on tabular online RL in time-inhomogeneous finite-horizon MDPs, 
this paper has established the minimax-optimal regret (resp.~sample complexity) --- up to log factors --- for the entire range of sample size $K\geq 1$ (resp.~target accuracy level $\varepsilon \in (0,H]$), thereby fully settling an open problem at the core of recent RL theory. The $\mathtt{MVP}$ algorithm studied herein is model-based in nature. Remarkably, the model-based approach remains the only family of algorithms that is capable of obtaining minimax optimality without burn-ins, regardless of the data collection mechanism in use (e.g., online RL, offline RL, and the simulator setting). We have further unlocked the optimality of this algorithm in a more refined manner, 
making apparent the effect of several problem-dependent quantities (e.g., optimal value/cost, variance statistics) upon the fundamental performance limits. 
The new analysis and algorithmic techniques put forward herein might shed important light on  
how to conquer other RL settings as well.

Moving forward, there are multiple directions that anticipate further theoretical pursuit. 
To begin with, is it possible to develop a model-free algorithm --- which often exhibits more favorable memory complexity compared to the model-based counterpart --- that achieves full-range minimax optimality?  
As alluded to previously, existing paradigms that rely on reference-advantage decomposition (or variance reduction) seem to incur a high burn-in cost \citep{zhang2020almost,li2021breaking}, thus calling for new ideas to overcome this barrier. 
Additionally, multiple other tabular settings (e.g., time-homogeneous finite-horizon MDPs, discounted infinite-horizon MDPs) have also suffered from similar issues regarding the burn-in requirements \citep{zhang2020reinforcement,ji2023regret}. 
Take time-homogeneous finite-horizon MDPs for example: in order to achieve optimal sample efficiency, 
one needs to carefully deal with the statistical dependency incurred by aggregating data from across different time steps to estimate the same transition matrix (due to the homogeneous nature of $P$), which results in more intricate issues than the time-homogeneous counterpart.  
 We believe that resolving these two open problems will greatly enhance our theoretical understanding about online RL and beyond.

\section*{Acknowledgement}
We thank for Qiwen Cui for helpful discussions. 
Y.~Chen is supported in part by the Alfred P.~Sloan Research Fellowship, the Google Research Scholar Award, the AFOSR grant  FA9550-22-1-0198, 
the ONR grant N00014-22-1-2354,  and the NSF grants CCF-2221009 and CCF-1907661.  JDL acknowledges support of Open Philanthropy, NSF IIS 2107304,  NSF CCF 2212262, ONR Young Investigator Award, NSF CAREER Award 2144994, and NSF CCF 2019844. 
SSD acknowledges the support of NSF IIS 2110170, NSF
DMS 2134106, NSF CCF 2212261, NSF IIS 2143493,
NSF CCF 2019844, NSF IIS 2229881, and the Sloan Research Fellowship.

\appendix

\section{Preliminary facts} \label{sec:tech_lemmas}

In this section, we gather several useful results that prove useful in our analysis. We use $\mathds{1}\{\mathcal{E}\}$ to denote the indicator of the event $\mathcal{E}$.
The first result below is a user-friendly version of the celebrated Freedman inequality \citep{freedman1975tail}, 
a martingale counterpart to the Bernstein inequality. See \citet[Lemma~11]{zhang2020model} for the proof. 
\begin{lemma}[Freedman's inequality]\label{lemma:self-norm}
	Let $(M_n)_{n\geq 0}$ be a martingale such that $M_0=0$ and $|M_n-M_{n-1}|\leq c$  $(\forall n\geq 1)$ 
	hold for some quantity $c>0$. 
	Define $\mathsf{Var}_{n} \coloneqq \sum_{k=1}^n \mathbb{E}\left[  (M_{k}-M_{k-1})^2 \mymid \mathcal{F}_{k-1}\right]$ for every $n\geq 0$, where $\mathcal{F}_k$ is the $\sigma$-algebra generated by $(M_1,...,M_{k})$. Then for any integer $n\geq 1$ and any $\epsilon,\delta>0$, one has 
\begin{align}
	\mathbb{P} \left[       |M_n|\geq 2\sqrt{2}\sqrt{\mathsf{Var}_n \log\frac{1}{\delta} } +2\sqrt{\epsilon \log\frac{1}{\delta} } +2c\log\frac{1}{\delta} \right]\leq 2\left(\log_2\left(\frac{nc^2}{\epsilon}\right) +1 \right)\delta.\nonumber
\end{align}
\end{lemma}
Next, letting $\mathsf{Var}(X)$ represent the variance of $X$, 
we record a basic inequality connecting $\mathsf{Var}(X^2)$ with $\mathsf{Var}(X)$ for any bounded random variable $X$. 
\begin{lemma}[Lemma 30 in \citep{chen2021implicit}]\label{lemma:sqv}
	Let $X$ be a random variable, and denote by $C_{\max}$ the largest possible value of $X$. 
	Then we have $\mathsf{Var}(X^2)\leq 4 C_{\max}^2 \mathsf{Var}(X)$.
\end{lemma}
Now, we turn to an intimate connection between the sum of a sequence of bounded non-negative random variables and the sum of their associated conditional random variables (with each random variable conditioned on the past), which is a consequence of basic properties about supermartingales.   
\begin{lemma}[Lemma 10 in \citep{zhang2022horizon}]\label{lemma:con}
Let $X_1,X_2,\ldots$ be a sequence of random variables taking value in $[0,l]$. 
For any $k\geq 1$, let $\mathcal{F}_k$ be the $\sigma$-algebra generated by $(X_1,X_2,\ldots,X_k)$, and define 
	$Y_k \coloneqq \mathbb{E}[X_k \mymid \mathcal{F}_{k-1}]$. Then for any $\delta>0$, we have 
\begin{align}
& \mathbb{P}\left[ \exists n, \sum_{k=1}^n X_k \geq  3\sum_{k=1}^n Y_k+ l\log\frac{1}{\delta}\right]\leq \delta\nonumber
\\  & \mathbb{P}\left[  \exists n,  \sum_{k=1}^n Y_k \geq 3\sum_{k=1}^n X_k + l\log\frac{1}{\delta}  \right]    \leq \delta .\nonumber 
\end{align}
\end{lemma}

The next two lemmas are concerned with concentration inequalities for the sum of i.i.d.~bounded random variables: 
the first one is a version of the Bennet inequality, and the second one is an empirical Bernstein inequality (which replaces the variance in the standard Bernstein inequality with the empirical variance). 
\begin{lemma}[Bennet's inequality]\label{bennet}
Let $Z,Z_1,...,Z_n$  be i.i.d.~random variables with values in $[0,1]$ and let $\delta>0$. Define $\mathbb{V}Z = \mathbb{E}\left[(Z-\mathbb{E}Z)^2 \right]$. Then one has
\begin{align}
\mathbb{P}\left[ \left|\mathbb{E}\left[Z\right]-\frac{1}{n}\sum_{i=1}^n Z_i  \right| > \sqrt{\frac{  2\mathbb{V}Z \log(2/\delta)}{n}} +\frac{\log(2/\delta)}{n} \right]\leq \delta.\nonumber
\end{align}
\end{lemma}
%
%
\begin{lemma}[Theorem 4 in  \citet{maurer2009empirical}]\label{empirical bernstein}
Consider any $\delta>0$ and any integer $n\geq 2$. 
Let $Z,Z_1,...,Z_n$  be a collection of i.i.d.~random variables falling within $[0,1]$. 
Define the empirical mean $\overline{Z} \coloneqq \frac{1}{n}\sum_{i=1}^n Z_{i}$ and empirical variance $\widehat{V}_n  \coloneqq \frac{1}{n}\sum_{i=1}^n (Z_i- \overline{Z})^2$. Then we have
\begin{align}
\mathbb{P}\left[ \left|\mathbb{E}\left[Z\right]-\frac{1}{n}\sum_{i=1}^n Z_i  \right| > \sqrt{\frac{  2\widehat{V}_n \log(2/\delta)}{n-1}} +\frac{7\log(2/\delta)}{3(n-1)} \right] \leq \delta.\nonumber
\end{align}
\end{lemma}

Moreover, we record a simple fact concerning the visitation counts $\{N_h^k(s_h^k,a_h^k)\}$. 
\begin{lemma}\label{lemma:doubling}
Recall the definition of $N_h^k(s_h^k,a_h^k)$ in Algorithm~\ref{alg:main}. It holds that
\begin{align}
\sum_{k=1}^K \sum_{h=1}^H \frac{1}{\max\{ N_h^k(s_h^k,a_h^k),1\}}\leq 2SAH\log_2 K
\end{align}
\end{lemma}
\begin{proof}
In view of the doubling batch update rule, it is easily seen that: for any given $(s,a,h)$, 
\begin{align}
	\sum_{k=1}^K \frac{1}{\max\{ N_h^k(s_h^k,a_h^k),1\}}  \mathds{1}\Big\{(s,a)=\big(s_h^k,a_h^k \big) \Big\} \leq 2\log_2 K,
\end{align}
since each $(s,a,h)$ is associated with at most $\log_2 K $ epochs. 
Summing over $(s,a,h)$ completes the proof.
\end{proof}

As it turns out, Lemma~\ref{lemma:doubling} together with the Freedman inequality allows one to control the difference between the empirical rewards and the true mean rewards, as stated below. 
\begin{lemma}\label{lemma:bdempr}
With probability exceeding $1-2SAHK\delta'$, it holds that
\begin{align}
  \sum_{k=1}^{K}\sum_{h=1}^{H}\left|\widehat{r}_{h}^{k}(s_{h}^{k},a_{h}^{k})-r_{h}(s_{h}^{k},a_{h}^{k})\right|
	& \leq 4\sqrt{2SAH^{2}(\log_{2}K)\log\frac{1}{\delta'}}\sqrt{\sum_{k=1}^K\sum_{h=1}^Hr_{h}(s_{h}^{k},a_{h}^{k})}+52SAH^{2}(\log_{2}K)\log\frac{1}{\delta'} ;
	\nonumber \\
	\sum_{k=1}^{K}\sum_{h=1}^{H}\widehat{r}_{h}^{k}(s_{h}^{k},a_{h}^{k})&\leq2\sum_{k=1}^{K}\sum_{h=1}^{H}r_{h}(s_{h}^{k},a_{h}^{k})+60SAH^{2}(\log_{2}K)\log\frac{1}{\delta'}. \nonumber
\end{align}
\end{lemma}
As an immediate consequence of Lemma~\ref{lemma:bdempr} and the basic fact $\sum_{k,h}r_{h}(s_{h}^{k},a_{h}^{k})\leq KH$, we have
\begin{align}
\sum_{k=1}^{K}\sum_{h=1}^{H}\widehat{r}_{h}^{k}(s_{h}^{k},a_{h}^{k}) & \leq2\sum_{k=1}^{K}\sum_{h=1}^{H}r_{h}(s_{h}^{k},a_{h}^{k})+60SAH^{2}(\log_{2}K)\log\frac{1}{\delta'}\nonumber \\
 & \leq2KH+60SAH^{2}(\log_{2}K)\log\frac{1}{\delta'}\leq3KH\label{eq:sum-empirical-r-UB}
\end{align}
with probability exceeding $1-2SAHK\delta'$, 
where the last inequality holds true under the assumption~\ref{eq:assumption-K-proof}. 
\begin{proof}[Proof of Lemma~\ref{lemma:bdempr}]
In view of Lemma~\ref{empirical bernstein} and the union bound, with probability $1-2SAHK\delta'$ we have 
\begin{align}
\left| \widehat{r}_{h}^{k}(s,a)-r_{h}(s,a) \right| & \leq2\sqrt{2}\sqrt{\frac{\left(\widehat{\sigma}_{h}^{k}(s_{h}^{k},a_{h}^{k})-\big(\widehat{r}_{h}^{k}(s_{h}^{k},a_{h}^{k})\big)^{2}\right)\log\frac{1}{\delta'}}{N_{h}^{k}(s,a)}}+\frac{28H\log\frac{1}{\delta'}}{3N_{h}^{k}(s,a)} \notag\\
 & \leq2\sqrt{2}\sqrt{\frac{H\widehat{r}_{h}^{k}(s,a)\log\frac{1}{\delta'}}{N_{h}^{k}(s,a)}}+\frac{28H\log\frac{1}{\delta'}}{3N_{h}^{k}(s,a)} 
	\notag
\end{align}
simultaneously for all $(s,a,h,k)$ obeying $N_h^k(s,a)>2$, 
	where we take advantage of the basic fact $\widehat{\sigma}_h^k(s_h^k,a_h^k) \leq H\widehat{r}_h^k(s,a)$ (since each immediate reward is upper bounded by $H$). 
Solve the inequality above to obtain 
\begin{align}
\left|\widehat{r}_h^k(s,a) - r_h(s,a)\right|\leq  4\sqrt{\frac{Hr_h(s,a)\log \frac{1}{\delta'}}{N_h^k(s,a)}} + 24\frac{H\log \frac{1}{\delta'}}{N_h^k(s,a)}.\label{eq:bdt4_0}
\end{align}
It is then seen that
\begin{align}
  \sum_{k,h}\left|\widehat{r}_{h}^{k}(s_{h}^{k},a_{h}^{k})-r_{h}(s_{h}^{k},a_{h}^{k})\right|
 & \leq4SAH^{2}+\sum_{k,h}\left(4\sqrt{\frac{Hr_{h}(s_{h}^{k},a_{h}^{k})\log\frac{1}{\delta'}}{N_{h}^{k}(s_{h}^{k},a_{h}^{k})}}+24\frac{H\log\frac{1}{\delta'}}{N_{h}^{k}(s_{h}^{k},a_{h}^{k})}\right)\nonumber\\
 & \leq4SAH^{2}+4\sqrt{\sum_{k,h}\frac{H\log\frac{1}{\delta'}}{N_{h}^{k}(s_{h}^{k},a_{h}^{k})}}\cdot\sqrt{\sum_{k,h}r_{h}(s_{h}^{k},a_{h}^{k})}+24\sum_{k,h}\frac{H\log\frac{1}{\delta'}}{N_{h}^{k}(s_{h}^{k},a_{h}^{k})}. \nonumber
\end{align}
Here, the second inequality arises from Cauchy-Schwarz,  
whereas the term $4SAH^2$ accounts for those state-action pairs with $N_h^k(s,a)\leq 2$ (since there are at most $2SAH$ such occurances and it holds that $\left|\widehat{r}_{h}^{k}(s_{h}^{k},a_{h}^{k})-r_{h}(s_{h}^{k},a_{h}^{k})\right|\leq 2H$). 
This together with Lemma~\ref{lemma:doubling} then leads to 
\begin{align*}
  \sum_{k,h}\left|\widehat{r}_{h}^{k}(s_{h}^{k},a_{h}^{k})-r_{h}(s_{h}^{k},a_{h}^{k})\right|
 & \leq4SAH^{2}+4\sqrt{2SAH^{2}(\log_{2}K)\log\frac{1}{\delta'}}\sqrt{\sum_{k,h}r_{h}(s_{h}^{k},a_{h}^{k})}+48SAH^{2}(\log_{2}K)\log\frac{1}{\delta'}
	\nonumber \\
 & \leq 4\sqrt{2SAH^{2}(\log_{2}K)\log\frac{1}{\delta'}}\sqrt{\sum_{k,h}r_{h}(s_{h}^{k},a_{h}^{k})}+52SAH^{2}(\log_{2}K)\log\frac{1}{\delta'}.
	\nonumber
\end{align*}
Moreover, the AM-GM inequality implies that
\begin{align*}
\sum_{k,h}\widehat{r}_{h}^{k}(s_{h}^{k},a_{h}^{k})-\sum_{k,h}r_{h}(s_{h}^{k},a_{h}^{k}) & \leq \sum_{k=1}^{K}\sum_{h=1}^{H}r_{h}(s_{h}^{k},a_{h}^{k})+8SAH^{2}(\log_{2}K)\log\frac{1}{\delta'}+52SAH^{2}(\log_{2}K)\log\frac{1}{\delta'}
\end{align*}
\[
\Longrightarrow\qquad\sum_{k,h}\widehat{r}_{h}^{k}(s_{h}^{k},a_{h}^{k})\leq 2\sum_{k,h}r_{h}(s_{h}^{k},a_{h}^{k})+60SAH^{2}(\log_{2}K)\log\frac{1}{\delta'},
\]
thus concluding the proof.    
\end{proof}


\section{Proofs of key lemmas in Section~\ref{sec:tec}}\label{app:mfsectec}

\subsection{Proof of Lemma~\ref{lemma:key2}}\label{app:pfkey2}
It suffices to develop an upper bound on the cardinality of $\mathcal{C}$ (cf.~\eqref{eq:defn-C-choice}). 
Setting   
\begin{equation}
	M = \log_2 K \qquad\qquad \text{and} \qquad\qquad N = SAH,
\end{equation}
we find it helpful to introduce the following useful sets: 
\begin{subequations}
\begin{align}
	\mathcal{C}^{\mathsf{distinct}}(l) 
	&\coloneqq \Big\{ \mathcal{I} = \{\mathcal{I}^1,\ldots,\mathcal{I}^l\} \mid \mathcal{I}^{1}\leq \cdots \leq\mathcal{I}^{l}  ,  \mathcal{I}^{\tau}\in \{0,1,\cdots,M\}^N 
	\text{ and } \mathcal{I}^{\tau}\neq \mathcal{I}^{\tau+1}  ~(\forall\tau) 
 \Big\} ;
	\\
	\mathcal{C}^{\mathsf{distinct}} &\coloneqq \bigcup_{l\geq 1}\mathcal{C}^{\mathsf{distinct}}(l). 
\end{align}
\end{subequations}
In words, $\mathcal{C}^{\mathsf{distinct}}(l)$ can be viewed as the set of non-decreasing length-$l$ paths in $\{0,1,\cdots,M\}^N$, with all points on a path being distinct;   
 $\mathcal{C}^{\mathsf{distinct}}$ thus consists of all such paths  regardless of the length.

We first establish a connection between $|\mathcal{C}|$ and $\big|\mathcal{C}^{\mathsf{distinct}}\big|$. 
Define the operator $\mathsf{Proj}:\mathcal{C}\to \mathcal{C}^{\mathsf{distinct}}$ 
that maps each $\mathcal{I}\in \mathcal{C}$ to $\mathcal{I}^{\mathsf{distinct}}\in \mathcal{C}^{\mathsf{distinct}}$, 
where $\mathcal{I}^{\mathsf{distinct}}$ is composed of all distinct elements in $\mathcal{I}$ 
(in other words, this operator simply removes redundancy in $\mathcal{I}$). Let us looking at the following set 
\[
	\mathcal{B}(\mathcal{I}^{\mathsf{distinct}}) \coloneqq \big\{ \mathcal{I}\in \mathcal{C} \mid \mathsf{Proj}(\mathcal{I})=\mathcal{I}^{\mathsf{distinct}} \big\}
\]
for each $\mathcal{I}^{\mathsf{distinct}}\in \mathcal{C}^{\mathsf{distinct}}$. 
Since $\mathcal{I}^{\mathsf{distinct}}$ is a non-decreasing path with all its points being distinct, 
there are at most $MN+1$ elements in each $\mathcal{I}^{\mathsf{distinct}}$. Hence, the size of $\mathcal{B}(\mathcal{I}^{\mathsf{distinct}}) $ is at most the number of solutions to the following equations
\begin{align}
	\sum_{i=1}^{MN+1} x_i = K \qquad \text{and} \qquad x_i\in \mathbb{N} ~~\text{for all } 1\leq i \leq MN+1\nonumber.
\end{align}
Elementary combinatorial arguments then reveal that 
\begin{align*}
	\big| \mathcal{B}(\mathcal{I}^{\mathsf{distinct}})  \big | \leq 
\left( \begin{array}{c} K+MN \\ MN
\end{array}\right) 
\leq (K+MN)^{MN}\leq (2K)^{MN}
\end{align*}
for each $\mathcal{I}^{\mathsf{distinct}}$, 
provided that $K\geq MN=SAH\log_2K$. 
We then arrive at 
\begin{equation}
	|\mathcal{C}|\leq \big|\mathcal{C}^{\mathsf{distinct}}\big|\cdot (2K)^{MN}. 
	\label{eq:C-Cdistinct}
\end{equation}

Everything then boils down to bounding $|\mathcal{C}^{\mathsf{distinct}}|$. 
To do so, let us first look at the set $\mathcal{C}^{\mathsf{distinct}}(MN+1)$, 
as each path in $\mathcal{C}^{\mathsf{distinct}}$ cannot have length more than $MN+1$.  
For each $\mathcal{I}^{\mathsf{distinct}} = \{\widetilde{\mathcal{I}}^1,\widetilde{\mathcal{I}}^2,\ldots, \widetilde{\mathcal{I}}^{MN+1}\}\in \mathcal{C}^{\mathsf{distinct}}(MN+1)$, 
it is easily seen that
\begin{itemize}
	\item	$\widetilde{\mathcal{I}}^1 = [0,0,\ldots,0]^{\top}$ and $\widetilde{\mathcal{I}}^{MN+1}=[M,M,\ldots, M]^{\top}$.
	\item For each $1\leq \tau \leq MN$, $\widetilde{\mathcal{I}}^{\tau}$ and $\widetilde{\mathcal{I}}^{\tau+1}$ differ only in one element (i.e., their Hamming distance is 1). 
\end{itemize}
In other words, we can view $\mathcal{I}^{\mathsf{distinct}}$ as an $MN$-step path from  $ [0,0,\ldots,0]^{\top}$ to $[M,M,\ldots, M]^{\top}$, 
with each step moving in one dimension. 
Clearly, each step has at most $N$ directions to choose from, meaning that there are at most $N^{MN}$ such paths. This implies that 
\begin{align*}
	\big|\mathcal{C}^{\mathsf{distinct}}(MN+1)\big|\leq N^{MN}.
\end{align*}
To finish up, we further observe that for each $\mathcal{I}^{\mathsf{distinct}}\in \mathcal{C}^{\mathsf{distinct}}$, there exists some $\widetilde{\mathcal{I}}^{\mathsf{distinct}}\in \mathcal{C}^{\mathsf{distinct}}(MN+1)$ such that $\mathcal{I}^{\mathsf{distinct}}\subseteq \widetilde{\mathcal{I}}^{\mathsf{distinct}}$.  
This observation together with basic combinatorial arguments indicates that
\begin{align*}
	\big|\mathcal{C}^{\mathsf{distinct}}\big| \leq  2^{MN+1} \big|\mathcal{C}^{\mathsf{distinct}}(MN+1)\big|\leq (2N)^{MN+1}, 
\end{align*}
which taken collectively with \eqref{eq:C-Cdistinct} leads to the advertised bound  
\begin{align*}
	|\mathcal{C}|\leq (2K)^{MN} \big|\mathcal{C}^{\mathsf{distinct}}\big|
	\leq (4KN)^{MN+1} \leq (4KN)^{MN+1}.
\end{align*}

\subsection{Proof of Lemma~\ref{lemma:uniform}}\label{app:pfuniform}
Let us begin by considering any fixed total profile $\mathcal{I}\in \mathcal{C}$, 
any fixed integer $l$ obeying $2\leq l\leq\log_{2}K+1$, and any given
feasible sequence $\{X_{h,s,a}\}_{(s,a,h)\in\mathcal{S}\times\mathcal{A}\times[H]}$.
Recall that (i) $\widehat{P}_{s,a,h}^{(l)}$ is computed based on
the $l$-th batch of data comprising $2^{l-2}$ independent samples
from $\mathcal{D}^{\mathsf{expand}}$ (see Definition~\ref{def:filt2});
and (ii) each $X_{h+1,s,a}$ is given by a deterministic function of $\mathcal{I}$
and the empirical models for steps $h'\in[h+1,H]$. Consequently,
Lemma~\ref{lemma:self-norm} together with Definition~\ref{def:filt2} tells us that:
with probability at least $1-\delta'$, one has
\begin{align}
	&\sum_{s,a,h}\big\langle \widehat{P}_{s,a,h}^{(l)}-P_{s,a,h}, X_{h+1,s,a} \big\rangle \notag\\
	&\qquad 
	\leq \sqrt{\frac{8}{2^{l-2}}\sum_{s,a,h}\mathbb{V}\big(P_{s,a,h},X_{h+1,s,a}\big)\log\frac{3\log_{2}(SAHK)}{\delta'}}+\frac{4H}{2^{l-2}}\log\frac{3\log_{2}(SAHK)}{\delta'},
	\label{eq:xx1-aux}
\end{align}
where we view the left-hand side of \eqref{eq:xx1-aux} as a martingale sequence from $h=H$ back to $h=1$.

Moreover,  given that each $X_{h,s,a}$ has at most $K+1$ different
choices (since we assume $|\mathcal{X}_{h,\mathcal{I}}|\leq K+1$), 
there are no more than $(K+1)^{SAH}\leq (2K)^{SAH}$ possible choices of the feasible sequence
$\{X_{h,s,a}\}_{(s,a,h)\in\mathcal{S}\times\mathcal{A}\times[H]}$. 
In addition, it has been shown in Lemma~\ref{lemma:key2} that there are no more than $(4SAHK)^{2SAH\log_{2}K}$ possibilities of the total profile $\mathcal{I}$.
Taking the union bound over all these choices and replacing $\delta'$
in (\ref{eq:xx1-aux}) with $\delta'/\big((4SAHK)^{2SAH\log_{2}K} (2K)^{SAH}\log_2K\big)$, we can demonstrate that with probability at least $1-\delta'$, 
\begin{align}
 & 
	\sum_{s,a,h}\big\langle \widehat{P}_{s,a,h}^{(l)}-P_{s,a,h}, X_{h+1,s,a} \big\rangle
	\notag\\
	& \leq\sqrt{\frac{8}{2^{l-2}}\sum_{s,a,h}\mathbb{V}\big(P_{s,a,h},X_{h+1,s,a}\big)\left(2SAH\log_{2}K\log(4SAHK)+SAH\log (2K)+\log\frac{3\log_{2}^{2}(SAHK)}{\delta'}\right)}\nonumber\\
	& \qquad+\frac{4H}{2^{l-2}}\left(2SAH\log_{2}K\log(4SAHK)+SAH\log (2K)+\log\frac{3\log_{2}^{2}(SAHK)}{\delta'}\right) \notag\\
 & \leq\sqrt{\frac{8}{2^{l-2}}\sum_{s,a,h}\mathbb{V}\big(P_{s,a,h},X_{h+1,s,a}\big)\left(6SAH\log_{2}^{2}K+\log\frac{1}{\delta'}\right)}
	+\frac{4H}{2^{l-2}}\left(6SAH\log_{2}^{2}K+\log\frac{1}{\delta'}\right)
	\label{eq:xx1-aux-123}
\end{align}
holds simultaneously for all $\mathcal{I}\in \mathcal{C}$, all $2\leq l\leq\log_{2}K+1$, and all feasible sequences $\{X_{h,s,a}\}_{(s,a,h)\in\mathcal{S}\times\mathcal{A}\times[H]}$.

Finally, recalling our assumption $0\in \mathcal{X}_{h+1,\mathcal{I}}$, we see that 
for every total profile $\mathcal{I}$ and its associated feasible sequence  $\{X_{h,s,a}\}$, 
\[
	\sum_{s,a,h}\max\Big\{\big\langle\widehat{P}_{s,a,h}^{(l)}-P_{s,a,h},X_{h+1,s,a}\big\rangle,0\Big\}\in\bigg\{ \sum_{s,a,h}\big\langle\widehat{P}_{s,a,h}^{(l)}-P_{s,a,h},\widetilde{X}_{h+1,s,a}\big\rangle \,\Big|\, \widetilde{X}_{h+1,s,a}\in\mathcal{X}_{h+1,\mathcal{I}}, \forall (s,a,h)\bigg\} 
\]
holds true.  Consequently, the uniform upper bound on the right-hand side of \eqref{eq:xx1-aux-123} continues to be a valid upper bound on 
$\sum_{s,a,h}\max\big\{\big\langle\widehat{P}_{s,a,h}^{(l)}-P_{s,a,h},X_{h+1,s,a}\big\rangle,0\big\}$. This concludes the proof.

\subsection{Proof of Lemma~\ref{lemma:decouple}}\label{app:pfdecouple}

We begin by making the following claim, which we shall establish towards the end of this subsection.   
\begin{claim}
	\label{claim:PV-l-UB}
	With probability exceeding $1-\delta'$, 
\begin{align}
 & \sum_{s,a,h}\Big\langle\widehat{P}_{s,a,h}^{(l)}-P_{s,a,h},V_{h+1}^{k_{l,j,s,a,h}}\Big\rangle \notag\\
 & \qquad\leq\sqrt{\frac{8}{2^{l-2}}\sum_{s,a,h}\mathbb{V}\big(P_{s,a,h},V_{h+1}^{k_{l,j,s,a,h}}\big)\left(6SAH\log_{2}^{2}K+\log\frac{1}{\delta'}\right)}+\frac{4H}{2^{l-2}}\left(6SAH\log_{2}^{2}K+\log\frac{1}{\delta'}\right)
	\label{eq:claim-PV-l-UB}
\end{align}
holds simultaneously for all $l=1,\ldots,\log_{2}K$ and all $j=1,\ldots,2^{l-1}$, 
where $k_{l,j,s,a,h}$ stands for the episode index of the sample that visits $(s,a,h)$ for the $(2^{l-1}+j)$-th time in the online learning process. 
\end{claim}

Assuming the validity of Claim~\ref{claim:PV-l-UB} for the moment, 
we can combine this claim with the decomposition~\eqref{eq:PV-sum-decompose} and applying the Cauchy-Schwarz inequality to reach
\begin{align*}
 & \sum_{k=1}^{K}\sum_{h=1}^{H}\Big\langle\widehat{P}_{s_{h}^{k},a_{h}^{k},h}^{k}-P_{s_{h}^{k},a_{h}^{k},h},V_{h+1}^{k}\Big\rangle\leq\sum_{l=1}^{\log_{2}K}\sum_{j=1}^{2^{l-1}}\sum_{s,a,h}\Big\langle\widehat{P}_{s,a,h}^{(l)}-P_{s,a,h},V_{h+1}^{k_{l,j,s,a,h}}\Big\rangle+SAH^{2}\\
 & \leq\sum_{l=1}^{\log_{2}K}\sum_{j=1}^{2^{l-1}}\sqrt{\frac{8}{2^{l-2}}\sum_{s,a,h}\mathbb{V}\big(P_{s,a,h},V_{h+1}^{k_{l,j,s,a,h}}\big)\left(6SAH\log_{2}^{2}K+\log\frac{1}{\delta'}\right)}\\
 & \qquad\qquad+\sum_{l=1}^{\log_{2}K}\sum_{j=1}^{2^{l-1}}\frac{4H}{2^{l-2}}\left(6SAH\log_{2}^{2}K+\log\frac{1}{\delta'}\right)+SAH^{2}\\
 & \leq\sum_{l=1}^{\log_{2}K}\sqrt{16\sum_{j=1}^{2^{l-1}}\sum_{s,a,h}\mathbb{V}\big(P_{s,a,h},V_{h+1}^{k_{l,j,s,a,h}}\big)\left(6SAH\log_{2}^{2}K+\log\frac{1}{\delta'}\right)}\\
 & \qquad\qquad+\sum_{l=1}^{\log_{2}K}8H\left(6SAH\log_{2}^{2}K+\log\frac{1}{\delta'}\right)+SAH^{2}\\
 & \leq\sqrt{16(\log_{2}K)\sum_{l=1}^{\log_{2}K}\sum_{j=1}^{2^{l-1}}\sum_{s,a,h}\mathbb{V}\big(P_{s,a,h},V_{h+1}^{k_{l,j,s,a,h}}\big)\left(6SAH\log_{2}^{2}K+\log\frac{1}{\delta'}\right)}\\
 & \qquad\qquad+\left(48SAH^{2}\log_{2}^{3}K+8H(\log_{2}K)\log\frac{1}{\delta'}\right)+SAH^{2}\\
 & \leq\sqrt{16(\log_{2}K)\sum_{k=1}^{K}\sum_{h=1}^{H}\mathbb{V}\big(P_{s_{h}^{k},a_{h}^{k},h},V_{h+1}^{k}\big)\left(6SAH\log_{2}^{2}K+\log\frac{1}{\delta'}\right)}+49SAH^{2}\log_{2}^{3}K+8H(\log_{2}K)\log\frac{1}{\delta'}.
\end{align*}
Here, the last inequality is valid due to our assumption $V_{h+1}^k=0$ ($\forall k > K$) and the identity
\begin{align}
 & \sum_{k=1}^{K}\sum_{h=1}^{H}\mathbb{V}\big(P_{s_{h}^{k},a_{h}^{k},h},V_{h+1}^{k}\big)\nonumber\\
 & =\sum_{l=1}^{\log_{2}K}\sum_{s,a,h}\sum_{j=1}^{2^{l-1}}\mathbb{V}\big(P_{s,a,h},V_{h+1}^{k_{l,j,s,a,h}}\big)+\sum_{k=1}^{K}\sum_{h=1}^{H}\mathds{1}\Big\{ N_{h}^{k,\mathsf{all}}(s_{h}^{k},a_{h}^{k})=0\Big\}\mathbb{V}\big(P_{s_{h}^{k},a_{h}^{k},h},V_{h+1}^{k}\big).\nonumber
\end{align}
This establishes our advertised bound on $\sum_{k,h}\big\langle\widehat{P}_{s_{h}^{k},a_{h}^{k},h}^{k}-P_{s_{h}^{k},a_{h}^{k},h},V_{h+1}^{k}\big\rangle$, 
provided that Claim~\ref{claim:PV-l-UB} is valid.

Before proceeding to the proof of Claim~\ref{claim:PV-l-UB}, 
we note that the other two quantities $\sum_{k,h}\max\big\{\big\langle\widehat{P}_{s_{h}^{k},a_{h}^{k},h}^{k}-P_{s_{h}^{k},a_{h}^{k},h},V_{h+1}^{k}\big\rangle,0\big\}$ 
and $\sum_{k,h}\big\langle\widehat{P}_{s_{h}^{k},a_{h}^{k},h}^{k}-P_{s_{h}^{k},a_{h}^{k},h},\big(V_{h+1}^{k}\big)^{2}\big\rangle$ 
can be upper bounded using exactly the same arguments, which we omit for the sake of brevity. 
In particular, the latter quantity further satisfies 
\begin{align*} 
	& \sum_{k=1}^{K}\sum_{h=1}^{H}\Big\langle\widehat{P}_{s_{h}^{k},a_{h}^{k},h}^{k}-P_{s_{h}^{k},a_{h}^{k},h},\big(V_{h+1}^{k}\big)^{2}\Big\rangle\\
 & \leq\sqrt{16(\log_{2}K)\sum_{k=1}^{K}\sum_{h=1}^{H}\mathbb{V}\big(P_{s_{h}^{k},a_{h}^{k},h},\big(V_{h+1}^{k}\big)^{2}\big)\left(6SAH\log_{2}^{2}K+\log\frac{1}{\delta'}\right)}+49SAH^{3}\log_{2}^{3}K+8H^2(\log_{2}K)\log\frac{1}{\delta'}\\
 & \leq8H\sqrt{(\log_{2}K)\sum_{k=1}^{K}\sum_{h=1}^{H}\mathbb{V}\big(P_{s_{h}^{k},a_{h}^{k},h},V_{h+1}^{k}\big)\left(6SAH\log_{2}^{2}K+\log\frac{1}{\delta'}\right)}+49SAH^{3}\log_{2}^{3}K+8H^2(\log_{2}K)\log\frac{1}{\delta'},
\end{align*}
where the last inequality follows from Lemma~\ref{lemma:sqv} and
the fact that $0\leq V_{h+1}^{k}(s)\le H$ for all $s\in\mathcal{S}$.

\begin{proof}[Proof of Claim~\ref{claim:PV-l-UB}]

To invoke Lemma~\ref{lemma:decouple} to prove this claim,  we need to choose the set $\{\mathcal{X}_{h,\mathcal{I}}\}$ properly to include the true value function estimates $\{V_h^k\}$. 
To do so, we find it helpful to first introduce an auxiliary algorithm tailored to each total profile. 
Specifically, for each $\mathcal{I} \in \mathcal{C}$ (cf.~\eqref{eq:defn-C-choice}), consider the following updates operating upon the expanded sample set $\mathcal{D}^{\mathsf{expand}}$.  
%
%

\begin{algorithm}[ht]
	\DontPrintSemicolon
	\caption{Monotoinic Value Propagation for a given total profile $\mathcal{I}\in \mathcal{C}$ ($\mathtt{MVP}(\mathcal{I})$) \label{alg:main-fixed-profile}}
	
	\textbf{initialization: } set $V_{H+1}^{k,\mathcal{I}}(s)\leftarrow H$ for all $s\in \mathcal{S}$ and $1\leq k\leq K$. \\
	\For{$k=1,2,\ldots,K$} {
			\For{$h=H,H-1,...,1$} {
				\For{$(s,a)\in \mathcal{S}\times \mathcal{A}$} { \vspace{-1ex}
					\begin{align*} 
						\vspace{-3ex}
						j & \leftarrow I_{s,a,h}^k,~~ n \leftarrow 2^{j-2}, \\
						b_h(s,a) &\leftarrow c_1 \sqrt{\frac{   \mathbb{ V}\big(\widehat{P}^{(j)}_{s,a,h} ,V_{h+1}^{k,\mathcal{I}}\big) \log \frac{1}{\delta'}  }{ \max\{n,1 \} }}+c_2 \sqrt{\frac{\big(\widehat{\sigma}^{(j)}_h(s,a)- (\widehat{r}^{(j)}_h(s,a))^2 \big)\log \frac{1}{\delta'}}{\max\{n,1\}}} 
						 +c_3\frac{H\log \frac{1}{\delta'}}{ \max\{n,1\}  },  
						\\
						Q_h^{k,\mathcal{I}}(s,a) &\leftarrow \min\Big\{    \widehat{r}^{(j)}_h(s,a)+\langle \widehat{P}^{(j)}_{s,a,h}, V_{h+1}^{k,\mathcal{I}} \rangle +b_h(s,a)    ,H\Big\}, \\
						V^{k,\mathcal{I}}_{h}(s) &\leftarrow \max_{a}Q^{k,\mathcal{I}}_{h}(s,a).
					\end{align*}
					\vspace{-3ex}
				}
			}
	}
\end{algorithm}
If we construct
\begin{align}
	\mathcal{X}_{h,\mathcal{I}} \coloneqq \Big\{ V_h^{k,\mathcal{I}} \mid 1\leq k\leq K\Big\} \cup \{0\}, \qquad \forall h\in [H] \text{ and }\mathcal{I}\in \mathcal{C}, 
	\label{eq:X-I-ours}
\end{align}
then it can be easily seen that $\{\mathcal{X}_{h,\mathcal{I}}\}$ satisfies the properties stated right before  Lemma~\ref{lemma:uniform}.  
As a consequence, applying Lemma~\ref{lemma:uniform} yields
\begin{align}
 & \sum_{s,a,h}\Big\langle\widehat{P}_{s,a,h}^{(l)}-P_{s,a,h},X_{h+1,s,a}\Big\rangle \notag\\
 & \qquad\leq\sqrt{\frac{8}{2^{l-2}}\sum_{s,a,h}\mathbb{V}\big(P_{s,a,h},X_{h+1,s,a}\big)\left(6SAH\log_{2}^{2}K+\log\frac{1}{\delta'}\right)}+\frac{4H}{2^{l-2}}\left(6SAH\log_{2}^{2}K+\log\frac{1}{\delta'}\right)
	\label{eq:claim-PV-l-UB-temp}
\end{align}
simultaneously for all $l=1,\ldots,\log_{2}K$, all $\mathcal{I}\in \mathcal{C}$, and all sequences $\{X_{h,s,a}\}$ obeying $X_{h,s,a}\in \mathcal{X}_{h,\mathcal{I}}$, $\forall (s,a,h)$.

To finish up, denote by $\mathcal{I}^{\mathsf{true}}$ the true total profile resulting from the online learning process. 
Given the way we couple $\mathcal{D}^{\mathsf{expand}}$ and $\mathcal{D}^{\mathsf{original}}$ (see the beginning of Section~\ref{sec:decoupling-all}), 
	we can easily see that the true value function estimate $\{V_{h}^{k}\}$ obeys
\begin{equation}
	V_{h}^{k} = V_{h}^{k,\mathcal{I}^{\mathsf{true}}} \in \mathcal{X}_{h,\mathcal{I}^{\mathsf{true}}},\qquad 1\leq k \leq K. 
	\label{eq:Vk-Itrue}
\end{equation}
The claimed result then follows immediately from \eqref{eq:Vk-Itrue} and the uniform bound \eqref{eq:claim-PV-l-UB-temp}. 
\end{proof}

\section{Proofs of auxiliary lemmas in Section~\ref{app:thmmain}}\label{app:proof-lem-main}

\subsection{Proof of Lemma~\ref{lemma:opt}} \label{sec:proof-lemma:opt}
 To begin with, we find it helpful to define the following function 
 $$
	 f(p,v,n) \coloneqq \langle p, v\rangle + \max\Bigg\{\frac{20}{3}\sqrt{\frac{\mathbb{V}(p,v)\log \frac{1}{\delta'}} {n}} ,\frac{400}{9} \frac{H\log \frac{1}{\delta'}}{n} \Bigg\}
 $$ 
 for any vector $p\in \Delta^{S}$, any non-negative vector $v\in \mathbb{R}^S$ obeying $\|v\|_{\infty}\leq H $, and any positive integer $n$. 
 We claim that 
\begin{equation}
	f(p,v,n)  \text{ is non-decreasing in each entry of }v.
	\label{eq:nondecreasing-f-property}
\end{equation}
To justify this claim, consider any $1\leq s\leq S$, and let us freeze $p$, $n$ and all but the $s$-th entries of $v$.  
It then suffices to observe that (i) $f$ is a continuous function, and (ii) except for at most two possible choices of $v(s)$ that obey $\frac{20}{3}\sqrt{\frac{\mathbb{V}(p,v)\log \frac{1}{\delta'}} {n}}  =\frac{400}{9} \frac{H\log \frac{1}{\delta'}}{n} $, one can use the properties of $p$ and $v$ to calculate
 \begin{align}
	\frac{\partial f(p,v,n)}{\partial v(s)} & =p(s)+\frac{20}{3}\mathds{1}\Bigg\{\frac{20}{3}\sqrt{\frac{\mathbb{V}(p,v)\log\frac{1}{\delta'}}{n}}\geq\frac{400}{9}\frac{H\log\frac{1}{\delta'}}{n}\Bigg\}\frac{p(s)\big(v(s)-\langle p,v\rangle\big)\sqrt{\log\frac{1}{\delta'}}}{\sqrt{n\mathbb{V}(p,v)}}\nonumber\\
 & =p(s)+\mathds{1}\Bigg\{\sqrt{n\mathbb{V}(p,v)\log\frac{1}{\delta'}}\geq\frac{20}{3}H\log\frac{1}{\delta'}\Bigg\}\frac{\frac{20}{3}H\log\frac{1}{\delta'}}{\sqrt{n\mathbb{V}(p,v)\log\frac{1}{\delta'}}}\cdot\frac{p(s)\big(v(s)-\langle p,v\rangle\big)}{H}\nonumber\\
 & \geq\min\bigg\{ p(s)+p(s)\frac{\big(v(s)-\langle p,v\rangle\big)}{H},p(s)\bigg\}\nonumber\\
 & \geq p(s)\min\bigg\{\frac{H+v(s)-\langle p,v\rangle}{H},1\bigg\}\geq0,\nonumber
\end{align}
thus establishing the claim \eqref{eq:nondecreasing-f-property}.

We now proceed to the proof of Lemma~\ref{lemma:opt}. 
	 Consider any $(h,k,s,a)$, and we divide into two cases. 

\paragraph{Case 1: $N_{h}^k(s,a) \leq 2$.}  
In this case, the following trivial bounds arise directly from the update rule \eqref{eq:update1}:  
$$
	 Q_h^k(s,a)=H \geq Q_h^{\star}(s,a) \qquad \text{and} \qquad V_h^k(s)=H\geq V_h^{\star}(s). 
$$ 

\paragraph{Case 2: $N_{h}^k(s,a) > 2$.}
Suppose now that $Q_{h+1}^k\geq Q_{h+1}^{\star}$, which also implies that $V_{h+1}^k \geq V_{h+1}^{\star}$. 
If $Q_h^k(s,a)=H$, then $Q_h^k(s,a)\geq Q_h^{\star}(s,a)$ holds trivially, 
and hence it suffices to look at the case with $Q_h^k(s,a)<H$. 
According to the update rule in \eqref{eq:update1}, it holds that
\begin{align}
 & Q_{h}^{k}(s,a)\nonumber\\
 & =\widehat{r}_{h}^{k}(s,a)+\big\langle\widehat{P}_{s,a,h}^{k},V_{h+1}^{k}\big\rangle\nonumber
 \\ & \qquad \qquad +c_{1}\sqrt{\frac{\mathbb{V}(\widehat{P}_{s,a,h}^{k},V_{h+1}^{k})\log\frac{1}{\delta'}}{N_{h}^{k}(s,a)}}+c_{2}\sqrt{\frac{\left(\widehat{\sigma}_{h}^{k}(s,a)-\big(\widehat{r}_{h}^{k}(s,a)\big)^{2}\right)\log\frac{1}{\delta'}}{N_{h}^{k}(s,a)}}+c_{3}\frac{H\log\frac{1}{\delta'}}{N_{h}^{k}(s,a)}\nonumber\\
 & \geq\widehat{r}_{h}^{k}(s,a)+2\sqrt{2}\sqrt{\frac{\left(\widehat{\sigma}_{h}^{k}(s,a)-\big(\widehat{r}_{h}^{k}(s,a)\big)^{2}\right)\log\frac{1}{\delta'}}{N_{h}^{k}(s,a)}}+\frac{48H\log\frac{1}{\delta'}}{3N_{h}^{k}(s,a)}+f\big(\widehat{P}_{s,a,h}^{k},V_{h+1}^{k},N_{h}^{k}(s,a)\big)\nonumber\\
 & \geq\widehat{r}_{h}^{k}(s,a)+2\sqrt{2}\sqrt{\frac{\left(\widehat{\sigma}_{h}^{k}(s,a)-\big(\widehat{r}_{h}^{k}(s,a)\big)^{2}\right)\log\frac{1}{\delta'}}{N_{h}^{k}(s,a)}}+\frac{48H\log\frac{1}{\delta'}}{3N_{h}^{k}(s,a)}+f\big(\widehat{P}_{s,a,h}^{k},V_{h+1}^{\star},N_{h}^{k}(s,a)\big)
	\label{eq:Qhk-sa-LB1}
 \end{align}
for any $(s,a)$, where the last inequality results from the claim \eqref{eq:nondecreasing-f-property} and the property $V_{h+1}^k \geq V_{h+1}^{\star}$.
Moreover, applying Lemma \ref{empirical bernstein} and recalling the definition of $\widehat{\sigma}_h^k(s,a)$, we have 
\begin{subequations}
\label{eq_lemma1_ref.5+ref1}
\begin{align}
 & \mathbb{P}\left\{ \Big|\big\langle\widehat{P}_{s,a,h}^{k}-P_{s,a,h},\,V_{h+1}^{\star}\big\rangle\Big|>2\sqrt{\frac{\mathbb{V}\big(\widehat{P}_{s,a,h}^{k},V_{h+1}^{\star}\big)\log \frac{1}{\delta'}}{N_{h}^{k}(s,a)}}+\frac{14H\log \frac{1}{\delta'}}{3N_h^{k}(s,a)}\right\} \nonumber\\
 & \quad\leq\mathbb{P}\left\{ \Big|\big\langle\widehat{P}_{s,a,h}^{k}-P_{s,a,h},\,V_{h+1}^{\star}\big\rangle\Big|>\sqrt{\frac{2\mathbb{V}\big(\widehat{P}_{s,a,h}^{k},V_{h+1}^{\star}\big)\log \frac{1}{\delta'}}{N_{h}^{k}(s,a)-1}}+\frac{7H\log \frac{1}{\delta'}}{3N_{h}^{k}(s,a)-1}\right\} \leq2\delta'
	\label{eq_lemma1_ref.5}	
\end{align}
and
\begin{align}
\mathbb{P}\left\{ \Big|\widehat{r}_{h}^{k}(s,a)-r_{h}(s,a)\Big|>2\sqrt{\frac{\left(\widehat{\sigma}_{h}^{k}(s,a)-\big(\widehat{r}_{h}^{k}(s,a)\big)^{2}\right)\log \frac{1}{\delta'}}{N_{h}^{k}(s,a)}}+\frac{28H\log \frac{1}{\delta'}}{3N_{h}^{k}(s,a)}\right\}  & \leq2\delta' .
	\label{eq_lemma1_ref1}	
\end{align}
\end{subequations}
%
%
These two inequalities imply that with probability exceeding $1-4\delta'$, 
\begin{align*}
r_{h}(s,a) & \leq\widehat{r}_{h}^{k}(s,a)+2\sqrt{2}\sqrt{\frac{\left(\widehat{\sigma}_{h}^{k}(s,a)-\big(\widehat{r}_{h}^{k}(s,a)\big)^{2}\right)\log\frac{1}{\delta'}}{N_{h}^{k}(s,a)}}+\frac{28H\log\frac{1}{\delta'}}{3N_{h}^{k}(s,a)};\\
f\big(\widehat{P}_{s,a,h}^{k},V_{h+1}^{\star},N_{h}^{k}(s,a)\big) & =\big\langle P_{s,a,h},V_{h+1}^{\star}\big\rangle+\big\langle\widehat{P}_{s,a,h}^{k}-P_{s,a,h},V_{h+1}^{\star}\big\rangle\\
 & \qquad+\max\Bigg\{\frac{20}{3}\sqrt{\frac{\mathbb{V}(\widehat{P}_{s,a,h}^{k},V_{h+1}^{\star})\log\frac{1}{\delta'}}{N_{h}^{k}(s,a)}},\frac{400}{9}\frac{H\log\frac{1}{\delta'}}{N_{h}^{k}(s,a)}\Bigg\}\\
 & \geq\big\langle P_{s,a,h},V_{h+1}^{\star}\big\rangle.
\end{align*}
Substitution into \eqref{eq:Qhk-sa-LB1} gives: with probability at least $1-4\delta'$,  
$$	
Q_h^k(s,a)\geq r_h(s,a)+ \big\langle P_{s,a,h}, V_{h+1}^{\star} \big\rangle= Q_h^{\star}(s,a).
$$

\paragraph{Putting all this together.} 
With the above two cases in place,  one can invoke standard induction arguments to deduce that: with probability at least $1-4SAHK\delta'$, one has $Q_h^k(s,a)\geq Q_h^{\star}(s,a)$ 
and $V_h^k = \max_a Q_h^k(s,a)\geq \max_a Q_h^{\star}(s,a)= V_h^{\star}(s)$ for every $(s,a,h,k)$. The proof is thus completed.
%
%

\subsection{Proof of Lemma~\ref{lem:bound-T234}}\label{app:pflem:bound-T234}

\subsubsection{Bounding $T_2$} 
We first establish the bound \eqref{eq:boundt2} on $T_2$. To begin with, $T_2$ can be decomposed using the definition \eqref{eq:update1} of the bonus term:  
\begin{align}
T_{2} & =\sum_{k=1}^{K}\sum_{h=1}^{H}b_{h}^{k}(s_{h}^{k},a_{h}^{k})\nonumber\\
 & =\frac{460}{9}\sum_{k,h}\sqrt{\frac{\mathbb{V}\big(\widehat{P}_{s_{h}^{k},a_{h}^{k},h}^{k},V_{h+1}^{k}\big)\log\frac{1}{\delta'}}{N_{h}^{k}(s_{h}^{k},a_{h}^{k})}}+2\sqrt{2}\sum_{k,h}\sqrt{\frac{\left(\widehat{\sigma}_{h}^{k}(s_{h}^{k},a_{h}^{k})-\big(\widehat{r}_{h}^{k}(s_{h}^{k},a_{h}^{k})\big)^{2}\right)\log\frac{1}{\delta'}}{N_{h}^{k}(s_{h}^{k},a_{h}^{k})}}\nonumber\\
 & \qquad\qquad\qquad\qquad\qquad\qquad\qquad\qquad\qquad\qquad\qquad\qquad+\frac{544}{9}\sum_{k,h}\frac{H\log\frac{1}{\delta'}}{N_{h}^{k}(s_{h}^{k},a_{h}^{k})}.
	\label{eq:T2-decomposition-main1}
\end{align}
Applying the Cauchy-Schwarz inequality and invoking Lemma~\ref{lemma:doubling}, we obtain
\begin{align}
T_{2} & \leq\frac{460}{9}\sqrt{\sum_{k,h}\frac{\log\frac{1}{\delta'}}{N_{h}^{k}(s_{h}^{k},a_{h}^{k})}}\sqrt{\sum_{k,h}\mathbb{V}\big(\widehat{P}_{s_{h}^{k},a_{h}^{k},h}^{k},V_{h+1}^{k}\big)}\nonumber\\
 & \qquad+2\sqrt{2}\sqrt{\sum_{k,h}\frac{\log\frac{1}{\delta'}}{N_{h}^{k}(s_{h}^{k},a_{h}^{k})}}\sqrt{\sum_{k,h}\left(\widehat{\sigma}_{h}^{k}(s_{h}^{k},a_{h}^{k})-\big(\widehat{r}_{h}^{k}(s_{h}^{k},a_{h}^{k})\big)^{2}\right)}+\frac{544H\log\frac{1}{\delta'}}{9}\sum_{k,h}\frac{1}{N_{h}^{k}(s_{h}^{k},a_{h}^{k})}\nonumber\\
 & \leq\frac{460}{9}\sqrt{2SAH(\log_{2}K)\Big(\log\frac{1}{\delta'}\Big)\sum_{k,h}\mathbb{V}\big(\widehat{P}_{s_{h}^{k},a_{h}^{k},h}^{k},V_{h+1}^{k}\big)}\nonumber\\
 & \qquad+4\sqrt{SAH(\log_{2}K)\log\frac{1}{\delta'}}\sqrt{\sum_{k,h}\left(\widehat{\sigma}_{h}^{k}(s_{h}^{k},a_{h}^{k})-\big(\widehat{r}_{h}^{k}(s_{h}^{k},a_{h}^{k})\big)^{2}\right)}+\frac{1088}{9}SAH^{2}(\log_{2}K)\log\frac{1}{\delta'}.
	\label{eq:boundt2o-temp}
\end{align}
Using the basic fact $\widehat{\sigma}_h^k(s_h^k,a_h^k) \leq H\widehat{r}_h^k(s,a)$ (since each immediate reward is at most $H$) and the definition \eqref{eq:defn-T5-proof} of $T_5$,  we can continue the bound in \eqref{eq:boundt2o-temp} to derive
\begin{align}
T_{2} & \leq\frac{460}{9}\sqrt{2SAH(\log_{2}K)\Big(\log\frac{1}{\delta'}\Big)T_{5}}\nonumber\\
 & \qquad+4\sqrt{SAH^{2}(\log_{2}K)\log\frac{1}{\delta'}}\sqrt{\sum_{k,h}\widehat{r}_{h}^{k}(s_{h}^{k},a_{h}^{k})}+\frac{1088}{9}SAH^{2}(\log_{2}K)\log\frac{1}{\delta'}.
	\label{eq:boundt2o}
\end{align}
%
%
Applying Lemma~\ref{lemma:bdempr} to bound $\sum_{k,h}\widehat{r}_h^k(s_h^k,a_h^k)$ 
and using the basic fact $\sum_{k,h} r_h(s_h^k,a_h^k)\leq KH$, 
we can employ a little algebra to deduce that 
\begin{align*}
	T_2\leq 61\sqrt{2SAH(\log_2 K)\Big(\log\frac{1}{\delta'}\Big)T_5} +  8\sqrt{SAH^3K(\log_2 K)\log\frac{1}{\delta'}}+
	155SAH^2(\log_2K)\log\frac{1}{\delta'}
\end{align*}
with probability exceeding $1-2SAHK\delta'$.

%

\subsubsection{Bounding $T_3$} 
Next, let us prove the bound \eqref{eq:boundt3} on $|T_3|$. 
Recall that $V_{h+1}^{k}(s)$ denotes the value function estimate
of state $s$ \emph{before} the $k$-th episode, which corresponds
to the value estimate computed at the end of the \emph{previous epoch}.
This important fact implies that conditional on  $(s_{h}^{k},a_{h}^{k})$, the vector $e_{s_{h+1}^{k}}$ is statistically independent of $V_{h+1}^{k}$ and has conditional mean $P_{s_{h}^{k},a_{h}^{k},h}$,  
allowing us to invoke the Freedman inequality for
martingales (see Lemma~\ref{lemma:self-norm}) to control the sum
of $\big\langle P_{s_{h}^{k},a_{h}^{k},h}-e_{s_{h+1}^{k}},V_{h+1}^{k}\big\rangle$. 
Recalling the definition of $T_6$ in \eqref{eq:defn-T6-proof}, 
we can see from Lemma~\ref{lemma:self-norm} that
\begin{align}
 |T_3|  \leq 2\sqrt{2}\cdot \sqrt{  T_6 \log \frac{1}{\delta'} } + \log \frac{1}{\delta'} + 2H\log \frac{1}{\delta'} 
	\leq 2\sqrt{2}\cdot \sqrt{  T_6 \log \frac{1}{\delta'}  } + 3H\log \frac{1}{\delta'} \label{eq:boundt3-proof}
\end{align}
with probability at least $1-10SAH^2K^2\delta'$.

\subsubsection{Bounding $T_4$} 
We now turn attention to the bound \eqref{eq:bdt_4f} on $|T_4|$. 
Recall that
\begin{align}
T_4 = \sum_{k=1}^K\sum_{h=1}^H \left(\widehat{r}_h^k(s_h^k,a_h^k) - r_h(s_h^k,a_h^k)  \right) + \sum_{k=1}^K \left( \sum_{h=1}^H r_h(s_h^k,a_h^k) - V_1^{\pi^k}(s_1^k) \right),
	\label{eq:T4-defn-repeat}
\end{align}
and we shall bound the two terms above separately. 
\begin{itemize}
	\item Regarding the first term on the right-hand side of \eqref{eq:T4-defn-repeat}, we can apply Lemma~\ref{lemma:bdempr} and the fact $\sum_{k,h} r_h(s_h^k,a_h^k)\leq KH$ to show that
%
\begin{align}
  \left| \sum_{k=1}^K\sum_{h=1}^H \big(\widehat{r}_h^k(s_h^k,a_h^k)-r_{h}(s_h^k,a_h^k)\big)\right|
	\leq  4\sqrt{2SAH^3K (\log_2K)\log\frac{1}{\delta'}}+52 SAH^2 (\log_2K)\log \frac{1}{\delta'}\label{eq:bdt4_2}
\end{align}
holds with probability at least $1-2SAHK\delta'$. 
	\item
With regards to the second term on the right-hand side of \eqref{eq:T4-defn-repeat}, 
we note that conditional on $\pi^k$,  $E_k:=\sum_{h=1}^H r_{h}(s_h^k,a_h^k)- V_1^{\pi^k}(s_1^k)$ is a zero-mean random variable bounded in magnitude by $H$. 
According to Lemma~\ref{lemma:self-norm}, 
\begin{align}
	\left| \sum_{k=1}^K E_k \right|
	&\leq 2\sqrt{2}\cdot \sqrt{\sum_{k=1}^K \mathsf{Var}(E_k)\log \frac{1}{\delta'} } + 3H^2\log \frac{1}{\delta'} \nonumber\\
	&\leq 2\sqrt{2KH^2\log \frac{1}{\delta'}} + 3H^2 \log \frac{1}{\delta'}
	\label{eq:bdt4_1}
\end{align}
holds with probability exceeding $1- 4\delta' \log_2(KH)$, 
where $\mathsf{Var}(E_k)$ denotes the variance of $E_k$  conditioned on what happens before the $k$-th episode, 
and the last inequality follows since $|E_k|\leq H$ always holds. 
\end{itemize}

Substituting \eqref{eq:bdt4_2} and \eqref{eq:bdt4_1} into \eqref{eq:T4-defn-repeat} reveals that with probability at least $1-3SAHK\delta'$, 
\begin{align}
	|T_4| \leq 6\sqrt{2SAH^3K(\log_2K)\log \frac{1}{\delta'} } +  55SAH^2(\log_2 K)\log \frac{1}{\delta'}.\label{eq:bdt_4f-proof}
\end{align}

\subsection{Proof of Lemma~\ref{lem:bound-T56}} \label{sec:pflem:bound-T56}
Regarding the term $T_5$, direct calculation gives 
\begin{align}
T_{5} & =\sum_{k=1}^{K}\sum_{h=1}^{H}\mathbb{V}\big(\widehat{P}_{s_{h}^{k},a_{h}^{k},h}^{k},V_{h+1}^{k}\big)=\sum_{k=1}^{K}\sum_{h=1}^{H}\left(\Big\langle\widehat{P}_{s_{h}^{k},a_{h}^{k},h}^{k},\big(V_{h+1}^{k}\big)^{2}\Big\rangle-\big(\big\langle\widehat{P}_{s_{h}^{k},a_{h}^{k},h}^{k},V_{h+1}^{k}\big\rangle\big)^{2}\right)\nonumber\\
 & =\underset{=\:T_{7}}{\underbrace{\sum_{k=1}^{K}\sum_{h=1}^{H}\Big\langle\widehat{P}_{s_{h}^{k},a_{h}^{k},h}^{k}-P_{s_{h}^{k},a_{h}^{k},h},\big(V_{h+1}^{k}\big)^{2}\Big\rangle}}+\underset{=\:T_{8}}{\underbrace{\sum_{k=1}^{K}\sum_{h=1}^{H}\Big\langle P_{s_{h}^{k},a_{h}^{k},h}-e_{s_{h+1}^{k}},\big(V_{h+1}^{k}\big)^{2}\Big\rangle}}\nonumber\\
 & \qquad\qquad\qquad\qquad\qquad\qquad+\sum_{k=1}^{K}\sum_{h=1}^{H}\big(V_{h+1}^{k}(s_{h+1}^{k})\big)^{2}-\sum_{k=1}^{K}\sum_{h=1}^{H}\big(\big\langle\widehat{P}_{s_{h}^{k},a_{h}^{k},h}^{k},V_{h+1}^{k}\big\rangle\big)^{2}\notag\\
 & =T_{7}+T_{8}+\sum_{k=1}^{K}\sum_{h=2}^{H}\big(V_{h}^{k}(s_{h}^{k})\big)^{2}-\sum_{k=1}^{K}\sum_{h=1}^{H}\big(\big\langle\widehat{P}_{s_{h}^{k},a_{h}^{k},h}^{k},V_{h+1}^{k}\big\rangle\big)^{2}\notag\\
 & \leq T_{7}+T_{8}+2H\sum_{k=1}^{K}\sum_{h=1}^{H}\max\Big\{ V_{h}^{k}(s_{h}^{k})-\big\langle\widehat{P}_{s_{h}^{k},a_{h}^{k},h}^{k},V_{h+1}^{k}\big\rangle,0\Big\}\nonumber\\
 & \leq T_{7}+T_{8}+2H\sum_{k=1}^{K}\sum_{h=1}^{H}b_{h}^{k}(s_{h}^{k},a_{h}^{k})+2H\sum_{k=1}^{K}\sum_{h=1}^{H}\widehat{r}^k_{h}(s_{h}^{k},a_{h}^{k})\label{eq:boundt5-intermediate-13}\\
 & \leq T_{7}+T_{8}+2HT_{2}+6KH^{2}\label{eq:boundt5-intermediate}
\end{align}
with probability at least $1-3\delta'\log(KH^{3})$. Here, the third line utilizes the fact that $V_{H+1}^k=0$,  
the first inequality holds since 
\begin{align*}
\big(V_{h}^{k}(s_{h}^{k})\big)^{2}-\big(\big\langle\widehat{P}_{s_{h}^{k},a_{h}^{k},h}^{k},V_{h+1}^{k}\big\rangle\big)^{2} 
	& =\Big(V_{h}^{k}(s_{h}^{k})+\big\langle\widehat{P}_{s_{h}^{k},a_{h}^{k},h}^{k},V_{h+1}^{k}\big\rangle\Big)\Big(V_{h}^{k}(s_{h}^{k})-\big\langle\widehat{P}_{s_{h}^{k},a_{h}^{k},h}^{k},V_{h+1}^{k}\big\rangle\Big)\\
 & \leq2H\max\left\{ V_{h}^{k}(s_{h}^{k})-\big\langle\widehat{P}_{s_{h}^{k},a_{h}^{k},h}^{k},V_{h+1}^{k}\big\rangle,\,0\right\} ,
\end{align*}
 the penultimate line makes use of the property $V_{h}^{k}(s_{h}^{k})=Q_{h}^{k}(s_{h}^{k},a_{h}^{k})$ and the update rule \eqref{eq:updateq}, 
 whereas the last line applies property \eqref{eq:sum-empirical-r-UB} and the definition \eqref{eq:boundt2} of $T_2$.


Akin to the above bound on $T_5$, we can show that with probability at least $1-3SAHK\delta'$, 
\begin{align}
T_{6} & =\sum_{k=1}^{K}\sum_{h=1}^{H}\mathbb{V}\big(P_{s_{h}^{k},a_{h}^{k},h},V_{h+1}^{k}\big)=\sum_{k=1}^{K}\sum_{h=1}^{H}\Big\langle P_{s_{h}^{k},a_{h}^{k},h},\big(V_{h+1}^{k}\big)^{2}\Big\rangle-\sum_{k=1}^{K}\sum_{h=1}^{H}\Big(\big\langle P_{s_{h}^{k},a_{h}^{k},h},V_{h+1}^{k}\big\rangle\Big)^{2}\nonumber\\
 & =\sum_{k=1}^{K}\sum_{h=1}^{H}\Big\langle P_{s_{h}^{k},a_{h}^{k},h}-e_{s_{h+1}^{k}},\big(V_{h+1}^{k}\big)^{2}\Big\rangle+\sum_{k=1}^{K}\sum_{h=2}^{H}\big(V_{h}^{k}(s_{h}^{k})\big)^{2}-\sum_{k=1}^{K}\sum_{h=1}^{H}\Big(\big\langle P_{s_{h}^{k},a_{h}^{k},h},V_{h+1}^{k}\big\rangle\Big)^{2}\nonumber\\
 & \leq T_{8}+2H\sum_{k=1}^{K}\sum_{h=1}^{H}\max\Big\{ V_{h}^{k}(s_{h}^{k})-\big\langle P_{s_{h}^{k},a_{h}^{k},h},V_{h+1}^{k}\big\rangle,0\Big\}\notag\\
 & \leq T_{8}+2H\sum_{k=1}^{K}\sum_{h=1}^{H}\max\Big\{ V_{h}^{k}(s_{h}^{k})-\big\langle\widehat{P}_{s_{h}^{k},a_{h}^{k},h},V_{h+1}^{k}\big\rangle,0\Big\}+2H\sum_{k=1}^{K}\sum_{h=1}^{H}\max\Big\{\big\langle\widehat{P}_{s_{h}^{k},a_{h}^{k},h}^{k}-P_{s_{h}^{k},a_{h}^{k},h},V_{h+1}^{k}\big\rangle,0\Big\}\nonumber\\
 & \leq T_{8}+2H\sum_{k=1}^{K}\sum_{h=1}^{H}b_{h}^{k}(s_{h}^{k},a_{h}^{k})+2H\sum_{k=1}^{K}\sum_{h=1}^{H}\widehat{r}_{h}^{k}(s_{h}^{k},a_{h}^{k})+2HT_{9}\label{eq:boundt6-intermediate-135}\\
 & \leq T_{8}+2HT_{2}+6KH^2+2HT_{9}.\label{eq:boundt6-intermediate}
\end{align}
%
%


Finally, note that the above bounds on $T_5$ and $T_6$ both depend on the term $T_8$ (cf.~\eqref{eq:defn-T8-proof}), 
which we would like to cope with now. 
Using Freedman's inequality (cf.~Lemma~\ref{lemma:self-norm}) and the fact that $\mathsf{Var}(X^2)\leq 4 H^2\mathsf{Var}(X)$ for any random variable $X$ with support on $[-H,H]$ (cf.~Lemma~\ref{lemma:sqv}), we reach
\begin{align}
	|T_{8}|\leq2\sqrt{2}\sqrt{\sum_{k,h}\mathbb{V}\Big(\widehat{P}_{s_{h}^{k},a_{h}^{k},h}^{k},\big(V_{h+1}^{k}\big)^{2}\Big)\log\frac{1}{\delta'}}+3H^{2}\log\frac{1}{\delta'}\leq \sqrt{32H^{2}T_{6}\log\frac{1}{\delta'}}+3H^{2}\log\frac{1}{\delta'}	 
\end{align}
with probability at least $1-3\delta'\log(KH^{3})$. 
Substitution into \eqref{eq:boundt5-intermediate} and \eqref{eq:boundt6-intermediate} establishes \eqref{eq:boundt56}. 
%
%

\section{Proof of the value-based regret bound (proof of Theorem~\ref{thm:first})}\label{sec:appfirst}

Recall that 
\begin{equation}
	B=4000 (\log_2 K)^3 \log(3SAH)\log\frac{1}{\delta'} 
	\qquad \text{with }\delta' = \frac{\delta}{200SAH^2K^2}. 
	\label{eq:defn-B-first}
\end{equation}
Consider first the scenario where $K\leq \frac{BSAH^2}{v^{\star}}$: 
the regret bound can be upper bounded by 
\begin{align}
\mathbb{E}\big[\mathsf{Regret}(K)\big] & =\mathbb{E}\left[\sum_{k=1}^{K}\Big(V_{1}^{\star}(s_{1}^{k})-V_{1}^{\pi^{k}}(s_{1}^{k})\Big)\right]\leq\mathbb{E}\left[\sum_{k=1}^{K}V_{1}^{\star}(s_{1}^{k})\right]=K\mathbb{E}_{s_{1}\sim\mu}\big[V_{1}^{\star}(s_{1})\big]\notag\\
 & =Kv^{\star}=\min\Big\{\sqrt{BSAH^{2}Kv^{\star}}
	,Kv^{\star}\Big\}.
	\label{eq:E-regret-UB-easy-first}
\end{align}
As a result, 
the remainder of the proof is dedicated to the the case with
\begin{equation}
	K\geq \frac{BSAH^2}{v^{\star}} .
	\label{eq:K-focus-first}
\end{equation}
%


To begin with, recall that the proof of Theorem~\ref{thm1} in Section~\ref{app:thmmain} consists of bounding the quantities $T_1,\ldots,T_9$ (see \eqref{eq:decomposition}, \eqref{eq:defn-T56-proof} and \eqref{eq:defn-T789-proof}) and recall that $\delta' = \frac{\delta}{200SAH^2K^2}$. 
%
%
In order to establish Theorem~\ref{thm:first}, we need to develop tighter bounds on some of these quantities (i.e., $T_2$, $T_4$, $T_5$ and $T_6$)  to reflect their dependency on $v^{\star}$ (cf.~\eqref{eq:defn-vstar-formal}). 



\paragraph{Bounding $T_2$.}
Recall that we have shown in \eqref{eq:boundt2o} that
\begin{align}
 & T_{2}\leq\frac{460}{9}\sqrt{2SAH(\log_{2}K)\Big(\log\frac{1}{\delta'}\Big)T_{5}}\nonumber\\
 & \qquad+4\sqrt{SAH^{2}(\log_{2}K)\log\frac{1}{\delta'}}\sqrt{\sum_{k,h}\widehat{r}_{h}^{k}(s_{h}^{k},a_{h}^{k})}+\frac{1088}{9}SAH^{2}(\log_{2}K)\log\frac{1}{\delta'}.
\nonumber
\end{align}
In view of the definition of $T_4$ (cf.~\eqref{eq:decomposition}) as well as the fact that $\sum_{k=1}^K V_1^{\star}(s_1^k)\leq 3Kv^{\star} + H\log \frac{1}{\delta'}$ holds with probability at least $1-\delta'$ (see Lemma~\ref{lemma:con}),
we arrive at
\begin{align}
\sum_{k,h}\widehat{r}_{h}^{k}(s_{h}^{k},a_{h}^{k})\leq T_{4}+\sum_{k}V_{1}^{\pi_{k}}(s_{1}^{k})\leq T_{4}+\sum_{k}V_1^{\star}(s_{1}^{k})\leq T_{4}+3Kv^{\star}+H\log\frac{1}{\delta'},
	\label{eq:sum-rhat-vstar}
\end{align}
which in turn gives
 \begin{align}
T_{2} & \leq\frac{460}{9}\sqrt{2SAH(\log_{2}K)\Big(\log\frac{1}{\delta'}\Big)T_{5}}\nonumber\\
 & \qquad\qquad+4\sqrt{SAH^{2}(\log_{2}K)\log\frac{1}{\delta'}}\sqrt{T_{4}+3Kv^{\star}}+130SAH^{2}(\log_{2}K)\log\frac{1}{\delta'}.
\label{eq:nbft2}
 \end{align}

\paragraph{Bounding $T_4$.}
When it comes to the quantity $T_4$ (cf.~\eqref{eq:decomposition}), we make the observation that 
\begin{align}
 T_4 
	& 
	=\underset{\eqqcolon\,\widecheck{T}_{1}}{\underbrace{\sum_{k=1}^{K}\left(\sum_{h=1}^{H}\widehat{r}_{h}^{k}(s_{h}^{k},a_{h}^{k})-r_{h}(s_{h}^{k},a_{h}^{k})\right)}}+\underset{\eqqcolon\,\widecheck{T}_{2}}{\underbrace{\sum_{k=1}^{K}\left(\sum_{h=1}^{H}r_{h}(s_{h}^{k},a_{h}^{k})-V_{1}^{\pi^{k}}(s_{1}^{k})\right)}}.
	\label{eq:T4-decompose-T12-check}
 \end{align}
%
%
Repeating the arguments for \eqref{eq:sum-rhat-vstar} yields
\begin{align}
	\sum_{k,h} r_{h}(s_{h}^{k},a_{h}^{k})\leq \widecheck{T}_{2}+\sum_{k}V_{1}^{\pi_{k}}(s_{1}^{k})\leq \widecheck{T}_{2}+\sum_{k}V_1^{\star}(s_{1}^{k})\leq \widecheck{T}_{2}+3Kv^{\star}+H\log\frac{1}{\delta'}
	\label{eq:sum-rnohat-vstar}
\end{align}
 with probability at least $1-\delta'$. 
Combining this with Lemma~\ref{lemma:bdempr},
we see that   
\begin{align}
	\widecheck{T}_1 & \leq 4\sqrt{2SAH^2\log_2 K \log \frac{1}{\delta'}} \sqrt{\sum_{k=1}^K\sum_{h=1}^H r_h(s_h^k,a_h^k)} + 52SAH^2(\log_2 K)\log \frac{1}{\delta'} \nonumber
	\\ & \leq 4\sqrt{2SAH^2\log_2 K \log \frac{1}{\delta'}} \sqrt{\widecheck{T}_2 + 3Kv^{\star}} + 60SAH^2(\log_2 K)\log \frac{1}{\delta'}\label{eq:ct1}
\end{align}
with probability exceeding $1-3SAHK\delta'$. 
In addition,  Lemma~\ref{lemma:self-norm} tells us that
\begin{align}
\widecheck{T}_2 & \leq 2\sqrt{2\sum_{k=1}^K \mathbb{E}_{\pi^k,s_1\sim \mu}\left[\left(\sum_{h=1}^H r_h(s_h,a_h) \right)^2 \right]\log \frac{1}{\delta'}}+3H^2\log \frac{1}{\delta'} \nonumber
\\ & \leq 2\sqrt{2H\sum_{k=1}^K \mathbb{E}_{\pi^k,s_1\sim \mu}\left[\sum_{h=1}^H r_h(s_h,a_h)  \right]\log \frac{1}{\delta'}}+3H\log \frac{1}{\delta'} \nonumber
\\ & \leq 2\sqrt{2KHv^{\star}\log \frac{1}{\delta'}}+3H\log \frac{1}{\delta'} \label{eq:ct1.5}
\\ & \leq 2Kv^{\star} + 5H\log \frac{1}{\delta'}\label{eq:ct2}
\end{align}
with probability at least $1-2SAHK\delta'$, 
where the expectation operator $\mathbb{E}_{\pi^k,s_1\sim \mu}[\cdot]$ is taken over the randomness of a trajectory $\{(s_h,a_h)\}$ 
generated under policy $\pi^k$ and initial state $s_1\sim \mu$, 
the last line arises from the AM-GM inequality, 
and the penultimate line makes use of Assumption~\ref{assum1} and the fact that 
\[
\mathbb{E}_{\pi^{k},s_{1}\sim\mu}\left[\sum_{h=1}^{H}r_{h}(s_{h},a_{h})\right]=\mathbb{E}_{s_{1}\sim\mu}\left[V_{1}^{\pi^{k}}(s_{1})\right]\leq\mathbb{E}_{s_{1}\sim\mu}\left[V_{1}^{\star}(s_{1})\right]=v^{\star}.
\]
Taking \eqref{eq:ct1}, \eqref{eq:ct1.5} and \eqref{eq:ct2} together, we can demonstrate that with probability exceeding $1-5SAHK\delta'$,
\begin{subequations}
\begin{align}
	& \widecheck{T}_1\leq     13\sqrt{SAH^2Kv^{\star}(\log_2 K)\log \frac{1}{\delta'}}  + 80SAH^2(\log_2 K)\log\frac{1}{\delta'},\label{eq:check-T-bound-first}
\\ & \widecheck{T}_2 \leq  2\sqrt{2KHv^{\star}\log \frac{1}{\delta'}}+3H\log\frac{1}{\delta'} .
\end{align}
\end{subequations}
Substitution into \eqref{eq:T4-decompose-T12-check} reveals that: with probability exceeding $1-5SAHK\delta'$, 
\begin{align}
	T_4 \leq 15\sqrt{SAH^2Kv^{\star}(\log_2K)\log \frac{1}{\delta'}}  + 83SAH^2(\log_2K)\log\frac{1}{\delta'}.\label{eq:fnbt4}
\end{align}

\paragraph{Bounding $T_5$.}
Recall that we have proven in \eqref{eq:boundt5-intermediate-13} that
\begin{align}
	T_{5} & \leq T_{7}+T_{8}+2HT_2+2H\sum_{k=1}^{K}\sum_{h=1}^{H}\widehat{r}^k_{h}(s_{h}^{k},a_{h}^{k}).\label{eq:T5-UB-first-123}
\end{align}
%
%
With \eqref{eq:sum-rnohat-vstar} and \eqref{eq:ct2} in place, we can deduce that, with probability at least $1-3SAHK\delta'$, 
\begin{align}
\sum_{k,h}r_{h}(s_{h}^{k},a_{h}^{k}) & \leq\widecheck{T}_{2}+3Kv^{\star}+H\log\frac{1}{\delta'}\leq5Kv^{\star}+6H\log\frac{1}{\delta'}. 
	\label{eq:sum-r-UB-first}
\end{align}
Moreover, under the assumption~\eqref{eq:K-focus-first}, we can further bound \eqref{eq:check-T-bound-first} as
\[
	\widecheck{T}_{1}\leq\sqrt{BSAH^{2}Kv^{\star}}+BSAH^{2}\leq2Kv^{\star}
\]
with probability exceeding $1-3SAHK\delta'$, which combined with \eqref{eq:sum-r-UB-first} and the assumption~\eqref{eq:K-focus-first} results in
\begin{align}
\sum_{k,h}\widehat{r}_{h}^{k}(s_{h}^{k},a_{h}^{k})=\sum_{k,h}r_{h}(s_{h}^{k},a_{h}^{k})+\widecheck{T}_{1} & \leq7Kv^{\star}+6H\log\frac{1}{\delta'}\leq8Kv^{\star}.
	\label{eq:sum-hat-r-UB-first}
\end{align}
Substitution into \eqref{eq:T5-UB-first-123} indicates that: with probability exceeding $1-6SAHK\delta'$, 
\begin{align}
T_5\leq  T_7 + T_8+2HT_2 + 16HKv^{\star} .\label{eq:fnbt5}
\end{align}

\paragraph{Bounding $T_6$.}
Making use of our bounds \eqref{eq:boundt6-intermediate-135}, \eqref{eq:boundt8} and \eqref{eq:sum-hat-r-UB-first}, we can readily derive
\begin{align}
T_{6} & \leq T_{8}+2HT_{2}+2HT_{9}+2H\sum_{k=1}^{K}\sum_{h=1}^{H}\widehat{r}_{h}(s_{h}^{k},a_{h}^{k})\nonumber\\
 & \leq \sqrt{32T_{6}\log\frac{1}{\delta'}}+2HT_{9}+16HKv^{\star}+3H^{2}\log\frac{1}{\delta'}+2HT_{2}
	\label{eq:fnbt6-first}
\end{align}
with probability at least $1-16SAH^2K^2\delta'$.

\paragraph{Putting all pieces together.}
%
%
Recalling our choice of $B$ (cf.~\eqref{eq:defn-B-first}), 
we can see from \eqref{eq:nbft2}, \eqref{eq:boundt3}, \eqref{eq:fnbt4}, \eqref{eq:fnbt5}, \eqref{eq:fnbt6-first}, \eqref{eq:boundt8}, \eqref{eq:boundt1} and \eqref{eq:boundt7} that
\begin{subequations}
	\label{eq:all-T-bounds-first}
\begin{align}
T_{2} & \leq\sqrt{BSAHT_{5}}+\sqrt{BSAH^{2}(T_{4}+3Kv^{\star})}+BSAH^{2},\\
T_{3} & \leq\sqrt{BT_{6}}+BH,\\
T_{4} & \leq\sqrt{BSAH^{2}Kv^{\star}}+BSAH^{2},\\
T_{5} & \leq T_{7}+T_{8}+2HT_{2}+16HKv^{\star},\\
T_{6} & \leq\sqrt{BT_{6}}+2HT_{9}+16HKv^{\star}+BH^{2}+2HT_{2},\\
T_{8} & \leq\sqrt{BH^{2}T_{6}}+BH^{2},\\
T_{1}\leq T_{9} & \leq\sqrt{BSAHT_{6}}+BSAH^{2},\\
T_{7} & \leq H\sqrt{BSAHT_{6}}+BSAH^{3}.
\end{align}
\end{subequations}
Solving \eqref{eq:all-T-bounds-first} under the assumption $K\geq \frac{BSAH^2}{v^{\star}}$ allows us to demonstrate that
\begin{subequations}
\begin{align}
	T_6 &\lesssim BHKv^{\star} \\
	T_1\leq T_9 &\lesssim \sqrt{B^2SAH^2Kv^{\star}}\\
	T_7+T_8 &\lesssim \sqrt{B^2SAH^4Kv^{\star}} \\
	T_5 &\lesssim BHKv^{\star} \\
	T_2 &\lesssim \sqrt{B^2SAH^2Kv^{\star}} \\
	T_3 &\lesssim \sqrt{B^2HKv^{\star}} \\
	T_4 & \lesssim \sqrt{BSAH^{2}Kv^{\star}}
\end{align}
\end{subequations}
with probability exceeding $1-200SAH^2K^2\delta'$. 
Putting these bounds together with \eqref{eq:decomposition}, we arrive at
\[
	\mathsf{Regret}(K)\leq T_{1}+T_{2}+T_{3}+T_{4} 
	\lesssim B\sqrt{SAH^2Kv^{\star}} 
\]
with probability exceeding $1-200SAH^2K^2\delta'$. 
Replacing $\delta'$ with $\frac{\delta}{200SAH^2K^2}$ and taking $\delta=\frac{1}{2KH}$ give
\begin{align*}
	\mathbb{E}\big[\mathsf{Regret}(K)\big]
	&\lesssim(1-\delta)B\sqrt{SAH^{2}Kv^{\star}}+\delta Kv^{\star}\lesssim B\sqrt{SAH^{2}Kv^{\star}}+1 \asymp B\sqrt{SAH^{2}Kv^{\star}} \\
	&  
	\asymp \min\big\{ B\sqrt{SAH^{2}Kv^{\star}}, B Kv^{\star} \big\} 
	\asymp  \min\big\{ \sqrt{SAH^{2}Kv^{\star}}, Kv^{\star}\big\} \log^{5}(SAHK) 
	,
\end{align*}
provided that $K\geq \frac{BSAH^2}{v^{\star}}$. 
Taking this collectively with \eqref{eq:E-regret-UB-easy-first} concludes the proof.

\section{Proof of the cost-based regret bound (proof of Theorem~\ref{thm:cost})}\label{app:cost}

We now turn to the proof of Theorem~\ref{thm:cost}. 
For notational convenience, we shall use $r$ to denote the negative cost (namely, $r_h= -c_h$, $\widehat{r}_h=-\widehat{c}_h$, and so on) throughout this section. 
We shall also use the following notation (and similar quantities like $Q_h^k$, $V_h^k$, $\ldots$)
\begin{align}
	Q_h(s,a)&\leftarrow \max\left\{\min \Big\{ \widehat{r}_h(s,a) + \widehat{P}_{s,a,h}V_{h+1}+b_h(s,a), 0 \Big\} ,-H \right\},\nonumber \\
	V_h(s) & \leftarrow \max_a Q_h(s,a),
	\nonumber
\end{align}
in order to be consistent with the reward-based setting.

Akin to the proof of Theorem~\ref{thm:first},  
we need to bound the quantities $T_1,\ldots,T_9$ introduced previously (see \eqref{eq:decomposition}, \eqref{eq:defn-T56-proof} and \eqref{eq:defn-T789-proof}). 
We note that the analysis for $T_1$, $T_3$, $T_7$, $T_8$ and $T_9$ in Appendix~\ref{sec:appfirst} readily applies to the negative reward case herein. 
Thus, it suffices to develop bounds on $T_2$, $T_4$, $T_5$ and $T_6$ to capture their dependency on $c^{\star}$, 
which forms the main content of the remainder of  this section.

\paragraph{Bounding $T_2$.}
Recall from \eqref{eq:T2-decomposition-main1} that 
\begin{align}
T_{2} 
 & =\frac{460}{9}\sum_{k,h}\sqrt{\frac{\mathbb{V}\big(\widehat{P}_{s_{h}^{k},a_{h}^{k},h}^{k},V_{h+1}^{k}\big)\log\frac{1}{\delta'}}{N_{h}^{k}(s_{h}^{k},a_{h}^{k})}}+\nonumber\\
 & \qquad2\sqrt{2}\sum_{k,h}\sqrt{\frac{\left(\widehat{\sigma}_{h}^{k}(s_{h}^{k},a_{h}^{k})-\big(\widehat{r}_{h}^{k}(s_{h}^{k},a_{h}^{k})\big)^{2}\right)\log\frac{1}{\delta'}}{N_{h}^{k}(s_{h}^{k},a_{h}^{k})}} +\frac{544}{9}\sum_{k,h}\frac{H\log\frac{1}{\delta'}}{N_{h}^{k}(s_{h}^{k},a_{h}^{k})}.
\label{eq:cc1}
\end{align}
In what follows, let us bound the three terms on the right-hand side of \eqref{eq:cc1} separately. 
\begin{itemize}
	\item 
For the first and the third terms on the right-hand side of \eqref{eq:cc1}, invoking the Cauchy-Schwarz inequality and Lemma~\ref{lemma:doubling} gives
\begin{align}
\sum_{k,h}\sqrt{\frac{\mathbb{V}\big(\widehat{P}_{s_{h}^{k},a_{h}^{k},h}^{k},V_{h+1}^{k}\big)\log\frac{1}{\delta'}}{N_{h}^{k}(s_{h}^{k},a_{h}^{k})}} & \leq\sqrt{2SAH(\log_{2}K)\Big(\log\frac{1}{\delta'}\Big)\sum_{k,h}\mathbb{V}\big(\widehat{P}_{s_{h}^{k},a_{h}^{k},h}^{k},V_{h+1}^{k}\big)}\nonumber\\
 & =\sqrt{2SAH(\log_{2}K)\Big(\log\frac{1}{\delta'}\Big)T_{5}}\label{eq:cterm1} 
 \end{align}
with $T_5$ defined in \eqref{eq:defn-T5-proof},  and in addition, 
 \begin{align}
	 \sum_{k,h} \frac{H\log \frac{1}{\delta'}}{N_h^k(s_h^k,a_h^k)} \leq 2SAH^2(\log_2K)\log \frac{1}{\delta'}.\label{eq:cterm2}
\end{align}

\item 
Let us turn to the second term on the right-hand side of \eqref{eq:cc1}. Observing the basic fact that
\begin{align}
 \widehat{\sigma}_h^k(s_h^k,a_h^k) - \big(\widehat{r}_h^k(s_h^k,a_h^k)\big)^2 \leq  -H \widehat{r}_h^k(s_h^k,a_h^k),\nonumber
\end{align}
we can combine it with Lemma~\ref{lemma:doubling} to derive 
\begin{align}
 & \sqrt{\frac{\left(\widehat{\sigma}_{h}^{k}(s_{h}^{k},a_{h}^{k})-\big(\widehat{r}_{h}^{k}(s_{h}^{k},a_{h}^{k})\big)^{2}\right)\log\frac{1}{\delta'}}{N_{h}^{k}(s_{h}^{k},a_{h}^{k})}}\leq\sqrt{2SAH(\log_{2}K)\log\frac{1}{\delta'}}\sqrt{H\sum_{k,h}-\widehat{r}_{h}^{k}(s_{h}^{k},a_{h}^{k})}\nonumber\\
 & \leq\sqrt{2SAH^{2}(\log_{2}K)\log\frac{1}{\delta'}}\sqrt{-T_{4}+3Kc^{\star}+\sum_{k=1}^{K}\Big(-V_{1}^{\pi^{k}}(s_{1}^{k})+V_{1}^{\star}(s_{1}^{k})\Big)+\sum_{k=1}^{K}\Big(-V_{1}^{\star}(s_{1}^{k})-3c^{\star}\Big)},\label{eq:cterm3}
\end{align}
where the last inequality invokes the definition of $T_4$ (see \eqref{eq:decomposition}). 
By virtue of Lemma~\ref{lemma:con} and the definition \eqref{eq:defn-cstar-formal} of $c^{\star}$, one can show that
\begin{align}
\sum_{k=1}^K -V_1^{\star}(s_1^k)\leq 3Kc^{\star} + H\log \frac{1}{\delta'}\nonumber
\end{align}
with probability exceeding $1-\delta'$. 
In addition, we note that
\begin{align}
\sum_{k=1}^K \Big(-V_1^{\pi^k}(s_1^k) +V_1^{\star}(s_1^k)  \Big) = \mathsf{Regret}(K) = T_1+T_2+T_3+T_4.
\end{align}
Taking these properties together with \eqref{eq:cterm3} yields
\begin{align*}
 & \sqrt{\frac{\left(\widehat{\sigma}_{h}^{k}(s_{h}^{k},a_{h}^{k})-\big(\widehat{r}_{h}^{k}(s_{h}^{k},a_{h}^{k})\big)^{2}\right)\log\frac{1}{\delta'}}{N_{h}^{k}(s_{h}^{k},a_{h}^{k})}}\nonumber\\
 & \qquad\leq\sqrt{2SAH^{2}(\log_{2}K)\log\frac{1}{\delta'}}\sqrt{T_{1}+T_{2}+T_{3}+2|T_{4}|+3Kc^{\star}+H\log\frac{1}{\delta'}}
\end{align*}
\end{itemize}
Putting the above results together, we can deduce that, with probability exceeding $1-\delta'$,
\begin{align}
	T_2 &  \leq 90\sqrt{SAH(\log_2K)\Big(\log\frac{1}{\delta'}\Big)T_5}  \nonumber\\ &  + 4\sqrt{SAH^2(\log_2K)\log\frac{1}{\delta'}}\sqrt{T_1+T_2+T_3+2|T_4|+3Kc^{\star}+H\log\frac{1}{\delta'}} + 130 SAH^2(\log_2 K)\log \frac{1}{\delta'}.\label{eq:ccbt2}
\end{align}

\paragraph{Bounding $T_4$.}
When it comes to the quantity $T_4$, we recall that
\begin{align}
 T_4 
	& = \underset{\eqqcolon\,\widecheck{T}_1}{\underbrace{ \sum_{k=1}^K \left( \sum_{h=1}^H \widehat{r}_h^k(s_h^k,a_h^k)-r_{h}(s_h^k,a_h^k)\right) }} 
	+ \underset{\eqqcolon\,\widecheck{T}_2}{\underbrace{ \sum_{k=1}^K \left( \sum_{h=1}^H r_h(s_h^k,a_h^k) - V_{1}^{\pi^k}(s_1^k) \right) }}.
	\label{eq:T4-1234567}
 \end{align}
%
%
To control $T_4$, we first make note of the following result that bounds the empirical reward (for the case with negative rewards), 
which assists in bounding the term $\widecheck{T}_1$.
\begin{lemma}\label{lemma:bdempc}
With probability at least $1-2SAHK\delta'$, it holds that
\begin{align}
 & \sum_{k=1}^{K}\sum_{h=1}^{H}\left|\widehat{r}_{h}^{k}(s_{h}^{k},a_{h}^{k})-r_{h}(s_{h}^{k},a_{h}^{k})\right|\nonumber\\
	& \leq 4\sqrt{2SAH^{2}(\log_{2}K)\log\frac{1}{\delta'}}\cdot\sqrt{\sum_{k=1}^{K}\sum_{h=1}^{H} \big(-r_{h}(s_{h}^{k},a_{h}^{k}) \big)}+52SAH^{2}(\log_{2}K)\log\frac{1}{\delta'}.\nonumber
\end{align}
\end{lemma}
\begin{proof}
The proof basically follows the same arguments as in the proof of Lemma~\ref{lemma:bdempr}, except that $r$ is now replaced with $-r$.
\end{proof}
\noindent 
Lemma~\ref{lemma:bdempc} tells us that with probability at least $1-3SAHK\delta'$,  
\begin{align}
	|\widecheck{T}_1 |& \leq 4\sqrt{2SAH^{2}(\log_{2}K)\log\frac{1}{\delta'}}\cdot\sqrt{\sum_{k=1}^{K}\sum_{h=1}^{H}\big(-r_{h}(s_{h}^{k},a_{h}^{k}) \big)}+52SAH^{2}(\log_{2}K)\log\frac{1}{\delta'} \nonumber
	\\ & \leq 4\sqrt{2SAH^2(\log_2 K)}\cdot \sqrt{-\widecheck{T}_2 + 3Kc^{\star}+\sum_{k=1}^K \big(-V_1^{\star}(s_1^k)- 3c^{\star}\big) } + 52SAH^2(\log_2 K)\log \frac{1}{\delta'} \nonumber
  \\ & \leq 4\sqrt{2SAH^2(\log_2 K)}\cdot \sqrt{\widecheck{T}_2 + 3Kc^{\star} } + 60SAH^2(\log_2 K)\log \frac{1}{\delta'} .\label{eq:cct1}
\end{align}
Here, the last line uses the fact (see Lemma~\ref{lemma:con}) that, with probability exceeding $1-\delta'$, 
\begin{align}
	\sum_{k=1}^K \big(-V_1^{\star}(s_1^k) \big) \leq 3Kc^{\star} + H\log \frac{1}{\delta'}. \label{eq:addc}
\end{align}

In addition, the Freedman inequality in Lemma~\ref{lemma:self-norm} combined with \eqref{eq:addc} implies that, with probability at least $1-3SAHK\delta$,
\begin{align}
|\widecheck{T}_2| & \leq 2\sqrt{2\sum_{k=1}^K \mathbb{E}_{\pi^k}\left[\left(\sum_{h=1}^H r_h(s_h,a_h) \right)^2 \,\Big|\, s_1=s_1^k\right]\log \frac{1}{\delta} }+3H\log \frac{1}{\delta} \nonumber
\\ & \leq 2\sqrt{2H\sum_{k=1}^K \mathbb{E}_{\pi^k}\left[\sum_{h=1}^H  -r_h(s_h,a_h) \,\Big|\, s_1=s_1^k \right]\log \frac{1}{\delta} }+3H\log \frac{1}{\delta} \nonumber
\\ & = 2\sqrt{2H\left( \sum_{k=1}^K \left(- V_1^{\pi^k}(s_1^k) +V_1^{\star}(s_1^k) \right) + \sum_{k=1}^K\left(-V_1^{\star}(s_1^k)-3c^{\star}\right)+3Kc^{\star}\right)\log \frac{1}{\delta}}
	+3H\log \frac{1}{\delta} \label{eq:cct1.5}
\\ & \leq 3Kc^{\star}  +T_1+T_2+T_3+T_4 + 9H\log \frac{1}{\delta}.
	\label{eq:cct2}
\end{align}
Combining \eqref{eq:cct1}, \eqref{eq:cct1.5} with \eqref{eq:cct2} reveals that, with probability at least $1-4SAHK\delta$,
\begin{align}
	& |\widecheck{T}_1|\leq     16\sqrt{SAH^2(Kc^{\star}+T_1+T_2+T_3+T_4)(\log_2 K)\log \frac{1}{\delta}}  + 200SAH^2(\log_2K)\log\frac{1}{\delta}\nonumber
\\ & |\widecheck{T}_2| \leq  2\sqrt{2H(3Kc^{\star}+T_1+T_2+T_3+T_4)\log \frac{1}{\delta}}+9H\log \frac{1}{\delta} .\nonumber
\end{align}
As a result, substitution into \eqref{eq:T4-1234567} leads to
\begin{align}
	|T_4| \leq 22\sqrt{SAH^2(Kc^{\star} +T_1+T_2+T_3+T_4) (\log_2 K)\log \frac{1}{\delta}}  + 209SAH^2 (\log_2 K)\log\frac{1}{\delta}.\label{eq:ccbt4}
\end{align}

\paragraph{Bounding $T_5$.}

Invoking the arguments in  \eqref{eq:boundt5} and recalling the update rule \eqref{eq:updatecost}, we obtain
\begin{align}
	T_{5} & \leq\sum_{k=1}^{K}\sum_{h=1}^{H}\Big\langle\widehat{P}_{s_{h}^{k},a_{h}^{k},h}^{k}-P_{s_{h}^{k},a_{h}^{k},h},\,\big(V_{h+1}^{k}\big)^{2}\Big\rangle+\sum_{k=1}^{K}\sum_{h=1}^{H}\Big\langle P_{s_{h}^{k},a_{h}^{k},h}-e_{s_{h+1}^{k}},\,\big(V_{h+1}^{k}\big)^{2}\Big\rangle\\
 & \qquad+2H\sum_{k=1}^{K}\sum_{h=1}^{H}\big[-r_{h}(s_{h}^{k},a_{h}^{k})\big].\nonumber
\end{align}
Moreover, we recall that
\begin{align}
	\sum_{k=1}^K \sum_{h=1}^H  \big[-r_h(s_h^k,a_h^k) \big] &  = - \widecheck{T}_2 -\sum_{k=1}^K V_1^{\pi^k}(s_1) \leq -\widecheck{T}_2 + \sum_{k=1}^K V_1^{\star}(s_1^k).\label{eq:cx1}
\end{align}
By virtue of \eqref{eq:addc}, one sees that with probability at least $1-5SAHK\delta$,
\begin{align}
	\sum_{k=1}^K \sum_{h=1}^H \big[ -r_h(s_h^k,a_h^k) \big] \leq  2\sqrt{2H(3Kc^{\star} + T_1+T_2+T_3+T_4)\log \frac{1}{\delta} } + 3Kc^{\star} + 10H\log \frac{1}{\delta}.
\end{align}
Consequently, we arrive at
\begin{align}
T_5 & \leq T_7 + T_8+2HT_2 + 4\sqrt{2H^3(3Kc^{\star} +T_1+T_2+T_3+T_4)\log \frac{1}{\delta}}+ 6HKc^{\star} + 20H^2\log \frac{1}{\delta}\label{eq:ccbt5}
\end{align}
with probability exceeding $1-5SAHK\delta$.

\paragraph{Bounding $T_6$.}
Invoking the arguments in  \eqref{eq:boundt5}, \eqref{eq:addc} and \eqref{eq:cx1}, and recalling the update rule \eqref{eq:updatecost}, 
we can demonstrate that 
\begin{align}
T_{6} & \leq2\sqrt{8T_{6}\log\frac{1}{\delta}}+3H^{2}\log\frac{1}{\delta}+2H\sum_{k=1}^{K}\sum_{h=1}^{H}\max\big\{\big\langle P_{s_{h}^{k},a_{h}^{k},h},V_{h+1}^{k}\big\rangle-V_{h}^{k}(s_{h}^{k}),0\big\}\nonumber\\
 & \leq2\sqrt{8T_{6}\log\frac{1}{\delta}}+3H^{2}\log\frac{1}{\delta}+2HT_{9}+2H\sum_{k=1}^{K}\sum_{h=1}^{H}\left[-r_{h}(s_{h}^{k},a_{h}^{k})\right]\nonumber\\
 & \leq2\sqrt{8T_{6}\log\frac{1}{\delta}}+3H^{2}\log\frac{1}{\delta}+2HT_{9}\nonumber\\
 & \qquad\qquad\qquad+2H\left(2\sqrt{2H(3Kc^{\star}+T_{1}+T_{2}+T_{3}+T_{4})\log\frac{1}{\delta}}+3Kc^{\star}+10H\log\frac{1}{\delta}\right)
	\label{eq:ccbt6}
\end{align}
with probability at least $1-3SAHK\delta$.

\paragraph{Putting all this together.}

Armed with the preceding bounds, we are ready to establish the claimed regret bound. 
By solving \eqref{eq:ccbt2},\eqref{eq:boundt3},\eqref{eq:ccbt4},\eqref{eq:ccbt5},\eqref{eq:ccbt6},\eqref{eq:boundt8},\eqref{eq:boundt1} and \eqref{eq:boundt7}, 
we can show that, with probability exceeding $1-100SAH^2K\delta$, 
\begin{align*}
	T_6 &\lesssim HKc^{\star}+BSAH^3, \\
	T_1 &\lesssim  \sqrt{BSAH^2Kc^{\star}}+BSAH^2, \\
	T_7+T_8 &\lesssim  \sqrt{BSAH^4Kc^{\star}}+BSAH^3, \\
	T_5 &\lesssim  HKc^{\star}+BSAH^2, \\
	T_2 &\lesssim  \sqrt{BSAH^2Kc^{\star}}+BSAH^2, \\
	T_3 &\lesssim  \sqrt{BHKc^{\star}}+BSAH^2. 
\end{align*}
%
%
We then readily conclude that the total regret is bounded by $$O\big(\sqrt{BSAH^2Kc^{\star}}+BSAH^2 \big).$$ 
In addition, the regret bound is trivially upper bounded by $O\big(K(H-c^{\star})\big)$. 
The proof is thus completed by combining these two regret bounds and replacing $\delta'$ with $\frac{\delta}{100SAH^2K}$.

\section{Proof of the variance-dependent regret bounds (proof of Theorem~\ref{thm:var})}\label{app:var}


In this section, we turn to establishing Theorem~\ref{thm:var}. The proof primarily contains two parts, as summarized in the following lemmas. 
%
%
\begin{lemma}\label{lemma:var1}
 With probability exceeding $1-\delta/2$, Algorithm~\ref{alg:main} obeys
	$$
		\mathsf{Regret}(K) \leq \widetilde{O}\Big(\min\big\{\sqrt{SAHK\mathrm{var}_1}+SAH^2,KH\big\}\Big). 
	$$
\end{lemma}
\begin{lemma}\label{lemma:var2}
 With probability at least $1-\delta/2$, Algorithm~\ref{alg:main} satisfies 
	$$
		\mathsf{Regret}(K) \leq \widetilde{O}\Big(\min\big\{\sqrt{SAHK\mathrm{var}_2}+SAH^2,KH\big\}\Big).
	$$
\end{lemma}

Putting these two regret bounds together and rescaling $\delta$ to $\delta/2$, we immediately conclude the proof of Theorem~\ref{thm:var}. 
The remainder of this section is thus devoted to establishing Lemma~\ref{lemma:var1} and Lemma~\ref{lemma:var2}. 

\subsection{Proof of Lemma~\ref{lemma:var1}}

Before proceeding, we recall that 
\begin{align}
&T_4 = \sum_{k=1}^K \left( \sum_{h=1}^H \widehat{r}_h^k(s_h^k,a_h^k) - V_1^{\pi^k}(s_1^k) \right),\nonumber
\\&T_5 = \sum_{k=1}^K \sum_{h=1}^H \mathbb{V}\big(\widehat{P}_{s_h^k,a_h^k,h},V_{h+1}^k\big),\nonumber
\\ &T_6 = \sum_{k=1}^K \sum_{h=1}^H \mathbb{V}\big(P_{s_h^k,a_h^k,h},V_{h+1}^k\big),\nonumber
 \end{align}
and that 
$$
	B =4000(\log_2K)^3\log(3SAH)\log\frac{1}{\delta'}
	\qquad\text{and}\qquad
	\delta' = \frac{\delta}{200SAH^2K^2}.
$$

\subsubsection{Bounding $T_2$}
Recall that when proving \eqref{eq:boundt2}, we have demonstrated that (see \eqref{eq:boundt2o-temp}) 
\begin{align}
	T_2 &\leq \frac{460}{9} \sqrt{2SAH (\log_2K) \Big( \log \frac{1}{\delta'} \Big) T_5 } \nonumber
\\ & \qquad  +4\sqrt{SAH(\log_2K)\log \frac{1}{\delta'}}\sqrt{\sum_{k,h}\left(\widehat{\sigma}_h^k(s_h^k,a_h^k)- \big(\widehat{r}_h^k(s_h^k,a_h^k)\big)^2\right)} 
	+ \frac{1088}{9} SAH^2(\log_2K)\log \frac{1}{\delta'} .\label{eq:local3}
 \end{align}
This motivates us to bound the sum $\sum_{k,h}\big(\widehat{\sigma}_h^k(s_h^k,a_h^k)- \big(\widehat{r}_h^k(s_h^k,a_h^k)\big)^2\big)$, 
which we accomplish via the following lemma.
\begin{lemma}\label{lemma:bdrv}
With probability at least $1-4SAHK\delta'$, one has
\begin{align}
\sum_{k,h}\left(\widehat{\sigma}_h^k(s_h^k,a_h^k)- \big(\widehat{r}_h^k(s_h^k,a_h^k) \big)^2\right)\leq  6K\mathrm{var}_1 + 242SAH^3(\log_2K)\log \frac{1}{\delta'}.
\end{align}
\end{lemma}
Combining Lemma~\ref{lemma:bdrv} with \eqref{eq:local3}, we can readily derive
\begin{align}
	T_2&\leq \frac{460}{9} \sqrt{2SAH (\log_2K) \Big( \log \frac{1}{\delta'} \Big) T_5 }  +12\sqrt{SAH(\log_2K)\log \frac{1}{\delta'}}\sqrt{2K\mathrm{var}_1} \notag\\
	& \qquad\qquad + 157SAH^2(\log_2K)\log \frac{1}{\delta'} \label{eq:nbt2}
\end{align}
with probability at least $1-4SAHK\delta'$. 

\begin{proof}[Proof of Lemma~\ref{lemma:bdrv}]
For notational convenience, let us define the variance of $R_h(s,a)$ as $v_h(s,a)$.

Firstly, we control each $\widehat{\sigma}_h^k(s_h^k,a_h^k)- (\widehat{r}_h^k(s_h^k,a_h^k))^2$ with $v_h(s,a)$. Fix $(s,a,h,k)$. 
Applying Lemma~\ref{lemma:con} shows that, with probability at least $1-2\delta'$,
\begin{align}
N_h^k(s,a)\left(\widehat{\sigma}_h^k(s_h^k,a_h^k)- \big(\widehat{r}_h^k(s_h^k,a_h^k)\big)^2\right)\leq 3N_h^k v_h(s,a)+ H^2\log \frac{1}{\delta'}.
\end{align}
This allows us to deduce that, with probability exceeding $1-2SAHK\delta'$,
\begin{align}
\sum_{k,h}\left(\widehat{\sigma}_h^k(s_h^k,a_h^k)- \big(\widehat{r}_h^k(s_h^k,a_h^k)\big)^2\right)
 &\leq 3\sum_{k,h}v_h(s_h^k,a_h^k) + \sum_{k,h}\frac{H^2\log \frac{1}{\delta'} }{N_h^k(s_h^k,a_h^k)} \notag\\
	& \leq 3\sum_{k,h}v_h(s_h^k,a_h^k) 
	+2SAH^3(\log_2K)\log \frac{1}{\delta'}.\label{eq:bdvr1}
\end{align}

It then suffices to bound the sum $\sum_{k,h}v_h(s_h^k,a_h^k)$. 
Towards this end, let 
$$
	\widetilde{V}_h^k(s) \coloneqq \mathbb{E}_{\pi^k}\left[\sum_{h'=h}^H v_{h'}(s_{h'},a_{h'}) \,\Big|\, s_h = s \right]
$$ 
be the value function with rewards taken to be $\{v_h(s,a)\}$ and the policy selected as $\pi^k$. 
It is clearly seen that $$\widetilde{V}^k_h(s,a)\leq H^2.$$
In view of Lemma~\ref{lemma:self-norm}, we can obtain
\begin{align}
\sum_{k=1}^K\sum_{h=1}^Hv_h(s_h^k,a_h^k) - \sum_{k=1}^K \widetilde{V}_1^k(s_1^k) 
&= \sum_{k=1}^K \left( \sum_{h=1}^H \big\langle e_{s_{h+1}^k}-P_{s_h^k,a_h^k,h},\,\widetilde{V}^k_{h+1} \big\rangle \right) \nonumber
\\ & \leq 2\sqrt{2\sum_{k=1}^K \sum_{h=1}^H \mathbb{V}\big(P_{s_h^k,a_h^k,h},\widetilde{V}_{h+1}^k \big)\log\frac{1}{\delta'}} + 3H^2 \log\frac{1}{\delta'} \label{eq:local0}
\end{align}
with probability at least $1-2SAHK\delta'$. 
Moreover, invoking Lemma~\ref{lemma:self-norm} once again reveals that 
\begin{align}
&\sum_{k=1}^K \sum_{h=1}^H \mathbb{V}\big(P_{s_h^k,a_h^k,h},\widetilde{V}_{h+1}^k\big)\nonumber
 \\ & = \sum_{k=1}^K \sum_{h=1}^H \big\langle P_{s_h^k,a_h^k,h} - e_{s_{h+1}^k}  ,\, (\widetilde{V}_{h+1}^k)^2 \big\rangle  \nonumber
 \\ & \qquad  + \sum_{k=1}^H \sum_{h=1}^H \left(\big(\widetilde{V}_{h+1}^k(s_{h+1}^k)\big)^2- \big(\widetilde{V}_h^k(s_h^k)\big)^2 \right)
	+\sum_{k=1}^K \sum_{h=1}^H \left( \big(\widetilde{V}_h^k(s_h^k)\big)^2 - \big(\big\langle P_{s_h^k,a_h^k,h}, \widetilde{V}_{h+1}^k \big\rangle\big)^2\right) \nonumber
 \\ & \leq 2\sqrt{8H^4\sum_{k=1}^K \sum_{h=1}^H \mathbb{V}\big(P_{s_h^k,a_h^k,h},\widetilde{V}_{h+1}^k \big)\log \frac{1}{\delta'}   } +2H^2\sum_{k=1}^K\sum_{h=1}^H v_h(s_h^k,a_h^k) 
	+ 3H^4\log \frac{1}{\delta'} \nonumber
 \\ & \leq 4H^2\sum_{k=1}^K\sum_{h=1}^H v_h(s_h^k,a_h^k)+42H^4\log \frac{1}{\delta'} \label{eq:local1}
\end{align}
with probability at least $1-2SAHK\delta'$. 
Combine \eqref{eq:local0} and \eqref{eq:local1} to yield 
\begin{align}
\sum_{k=1}^K\sum_{h=1}^K v_h(s_h^k,a_h^k) & \leq \sum_{k=1}^K \widetilde{V}_1^k(s_1^k) + 2\sqrt{8H^2\sum_{k=1}^K\sum_{h=1}^Hv_h(s_h^k,a_h^k) \log\frac{1}{\delta'}+ 84H^4\log^2 \frac{1}{\delta'}}+3H^2\log \frac{1}{\delta'} \nonumber
\\ & \leq   2\sum_{k=1}^K \widetilde{V}_1^k(s_1^k)  +80H^2\log \frac{1}{\delta'} \nonumber
\\ & \leq 2K\mathrm{var}_1 +80H^2\log \frac{1}{\delta'} \label{eq:bdddv}
\end{align}
with probability exceeding $1-4SAHK\delta'$. 
\end{proof}

\subsubsection{Bounding $T_4$}

We now move on to the term $T_4$, 
which can be written as $T_4 = \widecheck{T}_1+\widecheck{T}_2$ with
\begin{align*}
	\widecheck{T}_1 = \sum_{k=1}^K \sum_{h=1}^H  \left( \widehat{r}_h^k(s_h^k,a_h^k)-r_h(s_h^k,a_h^k) \right) \\
	\widecheck{T}_2 = \sum_{k=1}^K \left( \sum_{h=1}^H r_h(s_h^k,a_h^k) - V_1^{\pi^k}(s_1^k) \right).
\end{align*}
This leaves us with two quantities to control.

To begin with, let us look at $\widecheck{T}_1$. In view of Lemma~\ref{bennet} and the union bound over $(s,a,h,k)$, 
we see that, with probability at least $1-2SAHK\delta'$,
\begin{align}
\widehat{r}_h^k(s,a)-r_h(s,a)\leq \sqrt{\frac{2v_h(s,a)\log \frac{1}{\delta'}}{N_h^k(s,a)}}+ \frac{H\log \frac{1}{\delta'}}{N_h^k(s,a)}.
\end{align}
As a result, we obtain
\begin{align}
|\widecheck{T}_1| &  \leq \sum_{k=1}^K \sum_{h=1}^H\left(  \sqrt{\frac{2v_h(s_h^k,a_h^k)\log \frac{1}{\delta'}}{N_h^k(s_h^k,a_h^k)}}+ \frac{H\log \frac{1}{\delta'}}{N_h^k(s_h^k,a_h^k)}  \right)\nonumber
	\\ & \leq \sqrt{4SAH (\log_2 K)\log \frac{1}{\delta'}}\cdot \sqrt{\sum_{k=1}^K\sum_{h=1}^H v_h(s_h^k,a_h^k)} + 2SAH^2 (\log_2 K)\log \frac{1}{\delta'}.\label{eq:wc3}
\end{align}
In view of \eqref{eq:bdddv}, with probability exceeding $1-4SAHK\delta'$ we have
\begin{align}
\sum_{k=1}^K\sum_{h=1}^H v_h(s_h^k,a_h^k)\leq 2K\mathrm{var}_1 + 80H^2\log \frac{1}{\delta'}.\label{eq:wc1}
\end{align}
Consequently, we arrive at
\begin{align}
	|\widecheck{T}_1| \leq \sqrt{8SAHK\mathrm{var}_1 (\log_2K)\log \frac{1}{\delta'} }+ 20SAH^2 (\log_2 K)\log \frac{1}{\delta'}.\label{eq:cht1}
\end{align}

Next, we proceed to bound $\widecheck{T}_2$. Towards this, we make the observation that
\begin{align}
\widecheck{T}_2 = \sum_{k=1}^K \sum_{h=1}^H \big\langle e_{s_{h+1}^k}-P_{s_h^k,a_h^k,h} ,\, V_{h+1}^{\pi^k} \big\rangle.
\end{align}
Applying Lemma~\ref{lemma:self-norm} shows that, with probability at least $1-2SAHK\delta'$, 
\begin{align}
|\widecheck{T}_2|  & \leq 2\sqrt{2\sum_{k=1}^K\sum_{h=1}^H \mathbb{V}(P_{s_h^k,a_h^k,h},V_{h+1}^{\pi^k})\log \frac{1}{\delta'} }+ 3H\log \frac{1}{\delta'}\nonumber
\\ & \leq 2\sqrt{4\sum_{k=1}^K \sum_{h=1}^H \big(\mathbb{V}(P_{s_h^k,a_h^k,h},V_{h+1}^{\star}) + \mathbb{V}(P_{s_h^k,a_h^k,h}, V^{\star}_{h+1}-V_{h+1}^{\pi^k})  \big)\log \frac{1}{\delta'}}+3H\log \frac{1}{\delta'}.\label{eq:cls0}
\end{align}
Continue the calculation to derive
\begin{align}
 & \sum_{k=1}^{K}\sum_{h=1}^{H}\mathbb{V}(P_{s_{h}^{k},a_{h}^{k},h},V_{h+1}^{\star}-V_{h+1}^{\pi^{k}})\nonumber\\
 & =\sum_{k=1}^{K}\sum_{h=1}^{H}\left(\big\langle P_{s_{h}^{k},a_{h}^{k},h},(V_{h+1}^{\star}-V_{h+1}^{\pi^{k}})^{2}\big\rangle-\big(\big\langle P_{s_{h}^{k},a_{h}^{k},h},\,V_{h+1}^{\star}-V_{h+1}^{\pi^{k}}\big\rangle\big)^{2}\right)\nonumber\\
 & \leq\sum_{k=1}^{K}\sum_{h=1}^{H}\Big\langle P_{s_{h}^{k},a_{h}^{k},h}-e_{s_{h+1}^{k}},\,(V_{h+1}^{\star}-V_{h+1}^{\pi^{k}})^{2}\Big\rangle\nonumber\\
 & \quad+2H\sum_{k=1}^{K}\sum_{h=1}^{H}\max\left\{ \left(V_{h}^{\star}(s_{h}^{k})-r_{h}(s_{h}^{k},a_{h}^{k})-\big\langle P_{s_{h}^{k},a_{h}^{k},h},V_{h+1}^{\star}\big\rangle\right)-\left(V_{h}^{\pi^{k}}(s_{h}^{k})-r_{h}(s_{h}^{k},a_{h}^{k})-\big\langle P_{s_{h}^{k},a_{h}^{k},h},V_{h+1}^{\pi^{k}}\big\rangle\right),0\right\} \nonumber\\
 & \leq2\sqrt{8H^{2}\sum_{k=1}^{K}\sum_{h=1}^{H}\mathbb{V}(P_{s_{h}^{k},a_{h}^{k},h},V_{h+1}^{\star}-V_{h+1}^{\pi^{k}})\log\frac{1}{\delta'}}\nonumber\\
 & \qquad\qquad+2H\sum_{k=1}^{K}\sum_{h=1}^{H}\left(V_{h}^{\star}(s_{h}^{k})-r_{h}(s_{h}^{k},a_{h}^{k})-\big\langle P_{s_{h}^{k},a_{h}^{k},h},V_{h+1}^{\star}\big\rangle\right)+3H^{2}\log\frac{1}{\delta'}.
	\label{eq:csl1}
\end{align}
Here, \eqref{eq:csl1} holds with probability at least $1-2SAHK\delta'$, a consequence of Lemma~\ref{lemma:self-norm}
 and Lemma~\ref{lemma:sqv}.

To further bound the right-hand side of \eqref{eq:csl1}, we develop the following upper bound: 
\begin{align}
 & \sum_{k=1}^{K}\sum_{h=1}^{H}\left(V_{h}^{\star}(s_{h}^{k})-r_{h}(s_{h}^{k},a_{h}^{k})-\big\langle P_{s_{h}^{k},a_{h}^{k},h},V_{h+1}^{\star}\big\rangle\right)\nonumber\\
 & =\sum_{k=1}^{K}\left(V_{1}^{\star}(s_{1}^{k})-V_{1}^{\pi^{k}}(s_{1}^{k})\right)+\sum_{k=1}^{K}\big(V_{1}^{\pi^{k}}(s_{1}^{k})-\sum_{h=1}^{H}r_{h}(s_{h}^{k},a_{h}^{k})\big)+\sum_{k=1}^{K}\sum_{h=1}^{H}\big\langle e_{s_{h+1}^{k}}-P_{s_{h}^{k},a_{h}^{k},h},\,V_{h+1}^{\star}\big\rangle.
	\label{eq:d1}
\end{align}
Note that the first term on the right-hand side \eqref{eq:d1} is exactly $\mathsf{Regret}(K)=T_1+T_2+T_3+T_4$, the second term on the right-hand side \eqref{eq:d1} corresponds to $-T_4$, 
whereas the third term on the right-hand side \eqref{eq:d1} can be bounded by 
\begin{align}
\sum_{k=1}^K\sum_{h=1}^H \big\langle e_{s_{h+1}^k}-P_{s_h^k,a_h^k,h}, V_{h+1}^{\star} \big\rangle
	\leq 2\sqrt{2\sum_{k=1}^K\sum_{h=1}^H \mathbb{V}(P_{s_h^k,a_h^k,h},V_{h+1}^{\star})\log \frac{1}{\delta'}}+3H\log \frac{1}{\delta'}
\end{align}
with probability at least $1-2SAHK\delta'$.
It then implies the validity of the following bound with probability exceeding $1-8SAHK\delta'$: 
\begin{align}
& \sum_{k=1}^K \sum_{h=1}^H \left(V_{h}^{\star}(s_h^k)-r_h(s_h^k,a_h^k) - \big\langle P_{s_h^k,a_h^k,h}, V_{h+1}^{\star} \big\rangle \right) \nonumber
\\ & \leq T_1+T_2+T_3+2|T_4| + 2\sqrt{2\sum_{k=1}^K\sum_{h=1}^H\mathbb{V}(P_{s_h^k,a_h^k,h},V_{h+1}^{\star})\log \frac{1}{\delta'} } + 55H\log \frac{1}{\delta'} .
\end{align}
Combining these bounds with \eqref{eq:csl1}, we can use a little algebra to further obtain 
\begin{align}
&\sum_{k=1}^K\sum_{h=1}^H \mathbb{V}(P_{s_h^k,a_h^k,h},V_{h+1}^{\star}-V_{h+1}^{\pi^k}) \nonumber
	\\ & \leq   4H\left(T_1+T_2+T_3+2|T_4|+2\sqrt{2\sum_{k=1}^K\sum_{h=1}^H\mathbb{V}(P_{s_h^k,a_h^k,h},V_{h+1}^{\star})\log \frac{1}{\delta'}} \right)    +262H^2\log \frac{1}{\delta'} 
	\label{eq:wc2}
\end{align}
with probability at least $1-8SAHK\delta'$. 
If we define $T_{10} = \sum_{k=1}^K\sum_{h=1}^H\mathbb{V}(P_{s_h^k,a_h^k,h},V_{h+1}^{\star})$,  
then substituting \eqref{eq:wc2} into \eqref{eq:cls0} yields: with probability exceeding $1-10SAHK\delta'$, 
\begin{align}
|\widecheck{T}_2|&\leq 2\sqrt{8K\mathrm{var}_1 \log \frac{1}{\delta'} } +  8\sqrt{H\left(T_1+T_2+T_3+2|T_4|+2\sqrt{2T_{10}\log \frac{1}{\delta'}}\right)\log \frac{1}{\delta'}} +107H\log\frac{1}{\delta'}  \nonumber
\\ &  \leq 11\sqrt{T_{10}\log \frac{1}{\delta'}} +16\sqrt{H(T_1+T_2+T_3+2|T_4|)\log \frac{1}{\delta'}} + 115H\log \frac{1}{\delta'}.
\end{align}

Combining the above bound on $|\widecheck{T}_2|$ with \eqref{eq:cht1},  with probability exceeding $1-10SAHK\delta'$
\begin{align}
|T_4|&\leq |\widecheck{T}_1| + |\widecheck{T}_2| \notag\\
	&\leq
	18\sqrt{SAHT_{10}(\log_2K)\log \frac{1}{\delta'}} +16\sqrt{H(T_1+T_2+T_3+2|T_4|)\log \frac{1}{\delta'}} + 135SAH^2(\log_2K)\log \frac{1}{\delta'},\nonumber
\end{align}
which together with a little algebra yields
\begin{align}
|T_4| & \leq   36\sqrt{SAHT_{10}(\log_2K)\log \frac{1}{\delta'}} + 32\sqrt{H(T_1+T_2+T_3)\log \frac{1}{\delta'}}+           306SAH^2(\log_2K)\log \frac{1}{\delta'}.\label{eq:nbt4}
\end{align}

\subsubsection{Bounding $T_5$ and $T_6$}

We now turn attention to the terms $T_5$ and $T_6$. 
Towards this, 
we start with the following lemma. 
\begin{lemma}\label{lemma:empv}
With probability at least $1-2SAHK\delta'$, one has
\begin{align}
 T_5 \leq  5T_6 +8BSAH^3.\label{eq:var5x}
\end{align}
\end{lemma}
\begin{proof}[Proof of Lemma~\ref{lemma:empv}]
Direct computation gives 
\begin{align}
 & \sum_{k,h}\mathbb{V}(\widehat{P}_{s_{h}^{k},a_{h}^{k},h}^{k},V_{h+1}^{k})\nonumber\\
 & =\sum_{k,h}\left(\big\langle\widehat{P}_{s_{h}^{k},a_{h}^{k},h}^{k},(V_{h+1}^{k})^{2}\big\rangle-\big(\big\langle\widehat{P}_{s_{h}^{k},a_{h}^{k},h},V_{h+1}^{k}\big\rangle\big)^{2}\right)\nonumber\\
 & \leq\sum_{k,h}\left(\big\langle P_{s_{h}^{k},a_{h}^{k},h}^{k},(V_{h+1}^{k})^{2}\big\rangle-\big(\big\langle P_{s_{h}^{k},a_{h}^{k},h},V_{h+1}^{k}\big\rangle\big)^{2}\right)+\sum_{k,h}\big\langle\widehat{P}_{s_{h}^{k},a_{h}^{k},h}-P_{s_{h}^{k},a_{h}^{k},h},(V_{h+1}^{k})^{2}\big\rangle\\
 & \qquad\qquad+2H\sum_{k,h}\big\langle\widehat{P}_{s_{h}^{k},a_{h}^{k},h}-P_{s_{h}^{k},a_{h}^{k},h},V_{h+1}^{k}\big\rangle\nonumber\\
 & \leq\sum_{k,h}\mathbb{V}(P_{s_{h}^{k},a_{h}^{k},h}^{k},V_{h+1}^{k})+\sum_{k,h}\big\langle\widehat{P}_{s_{h}^{k},a_{h}^{k},h}-P_{s_{h}^{k},a_{h}^{k},h},\,(V_{h+1}^{k})^{2}\big\rangle+2H\sum_{k,h}\big\langle\widehat{P}_{s_{h}^{k},a_{h}^{k},h}-P_{s_{h}^{k},a_{h}^{k},h},V_{h+1}^{k}\big\rangle\nonumber\\
 & =T_{5}+T_{7}+2HT_{1}.
\end{align}
Invoking Lemma~\ref{lemma:decouple} to bound $T_7$ and $T_1$, we obtain 
\begin{align}
 \sum_{k,h}\mathbb{V}(\widehat{P}_{s_h^k,a_h^k,h}^k, V_{h+1}^k) &\leq \sum_{k,h}\mathbb{V}(P_{s_h^k,a_h^k,h}^k  , V_{h+1}^k) +  6\sqrt{ \sum_{k,h} \mathbb{V}(P_{s_h^k,a_h^k,h},V_{h+1}^k)BSAH^3 }+3BSAH^3\nonumber
 \\ & \leq 5\sum_{k,h}\mathbb{V}(P_{s_h^k,a_h^k,h}^k  , V_{h+1}^k) +8BSAH^3 \label{eq:var5}
\end{align}
with probability exceeding $1-2SAHK\delta'$. 
\end{proof}

In view of Lemma~\ref{lemma:empv}, it suffices to bound $T_6 = \sum_{k,h}\mathbb{V}(P_{s_h^k,a_h^k,h},V_{h+1}^k)$. 
Given that $\mathsf{Var}(X+Y)\leq 2(\mathsf{Var}(X)+\mathsf{Var}(Y))$ holds for any two random variables $X,Y$, we have 
\begin{align}
	T_6&=\sum_{k,h} \mathbb{V}\big(P_{s_h^k,a_h^k,h},V_{h+1}^k\big)   \leq  2\sum_{k,h} \mathbb{V}\big(P_{s_h^k,a_h^k,h},V_{h+1}^{\star} \big) + 2\sum_{k,h}\mathbb{V}\big(P_{s_h^k,a_h^k,h},V_{h+1}^k -V_{h+1}^{\star} \big)\nonumber
\\ 
	& \leq 3   K\mathrm{var}_1 +\sum_{k=1}^K \left( \sum_{h=1}^H \mathbb{V}\big(P_{s_h^k,a_h^k,h},V_{h+1}^{\star} \big) - 3\mathrm{var}_1 \right) + 2\sum_{k,h}\mathbb{V}\big(P_{s_h^k,a_h^k,h},V_{h+1}^k -V_{h+1}^{\star} \big).\label{eq:var0}
\end{align}
To further upper bound the right-hand side of \eqref{eq:var0}, we make note of the following lemmas. 
\begin{lemma}\label{lemma:k1}
With probability at least $1-4SAHK\delta'$, it holds that
\begin{align}
T_{10} -2K\mathrm{var}_1=\sum_{k=1}^K \left( \sum_{h=1}^H \mathbb{V}\big(P_{s_h^k,a_h^k,h},V_{h+1}^{\star} \big) -2\mathrm{var}_1 \right)  \leq 80H^2\log \frac{1}{\delta'}.
\end{align}
\end{lemma}
%
%
\begin{lemma}\label{lemma:bdv1}
With probability at least $1-2\delta'$, it holds that
\begin{align}
 & \sum_{k,h}\mathbb{V}\big(P_{s_h^k,a_h^k,h}, V_{h+1}^k -V_{h+1}^{\star}\big)   \leq 4\sqrt{BH^2\sum_{k,h}\mathbb{V}(P_{s_h^k,a_h^k,h},V_{h+1}^k)}+ 4H\sum_{k,h}b_h^k(s_h^k,a_h^k)+ 3BSAH^3.\nonumber
\end{align}
\end{lemma}
Combining Lemma~\ref{lemma:k1} and Lemma~\ref{lemma:bdv1} with \eqref{eq:var0}, we see that with probability at least $1-6SAHK\delta'$,
\begin{align}
T_6 
	&= \sum_{k,h}\mathbb{V}(P_{s_h^k,a_h^k,h},V_{h+1}^k)\nonumber
\\ & \leq 4K\mathrm{var}_1 +  8\sqrt{BSAH^3 T_6}+8HT_2 + 7BSAH^3, 
	\nonumber
\end{align}
and as a result, 
\begin{align}
T_6 
	& \leq 8K\mathrm{var}_1 + 16HT_2 + 78BSAH^3.\label{eq:nbt6}
\end{align}
This taken collectively with Lemma~\ref{lemma:empv} yields, with probability at least $1-8SAHK\delta'$, 
\begin{align}
	T_5 =\sum_{k,h}\mathbb{V}(\widehat{P}_{s_h^k,a_h^k,h},V_{h+1}^k)\leq 40K\mathrm{var}_1 + 80HT_2+398BSAH^3.\label{eq:nbt5}
\end{align}
To finish our bounds on $T_5$ and $T_6$, it remains to establish Lemma~\ref{lemma:k1} and Lemma~\ref{lemma:bdv1}.

\begin{proof}[Proof of Lemma~\ref{lemma:k1}]
Let $\overline{R}_{h}^{\star}(s,a) = \mathbb{V}(P_{s,a,h},V_{h+1}^{\star})$, and define
\begin{align}
\overline{V}^k_h (s) = \mathbb{E}\left[ \sum_{h'=h}^H \overline{R}_{h'}(s_{h'},a_{h'}) \,\Big|\, s_h = s\right].\nonumber
\end{align}
Then $\overline{V}_h^k(s)\leq \mathrm{var}_1\leq H^2$. 
It then follows that
\begin{align}
  \sum_{h=1}^H  \mathbb{V}(P_{s_h^k,a_h^k,h},V_{h+1}^{\star})-\mathrm{var}_1 \nonumber & =\sum_{h=1}^H  \overline{R}_h^{\star}(s_h^k,a_h^k)-\mathrm{var}_1  \nonumber
 \\ & \leq \sum_{h=1}^H \overline{R}_h^{\star}(s_h^k,a_h^k) - \overline{V}^k_1(s_1^k)\nonumber
 \\ & = \sum_{h=1}^H \left(e_{s_{h+1}^k} - P_{s_{h}^k,a_h^k,h}  \right)\overline{V}^k_{h+1}.
\end{align}
Note that $\overline{V}^k$ depends only on $\pi^k$, which is determined before the beginning of the $k$-th episode.  
Consequently, applying Lemma~\ref{lemma:self-norm} reveals that, with probability at least $1-2SAHK\delta'$,
\begin{align}
 & \sum_{k=1}^K \left( \sum_{h=1}^H \mathbb{V}(P_{s_h^k,a_h^k,h},V_{h+1}^{\star}) - \overline{V}_1^k(s_1^k)  \right) \nonumber
 \\ & \leq 2\sqrt{2\sum_{k=1}^K \sum_{h=1}^H \mathbb{V}\big(P_{s_h^k,a_h^k,h},\overline{V}_{h+1}^k \big)\log \frac{1}{\delta'} } + 3H^2\log \frac{1}{\delta'} .\label{eq:xlll1}
\end{align}
Regarding the sum of variance terms on the right-hand side of \eqref{eq:xlll1},  
one can further bound
\begin{align}
 & \sum_{k=1}^{K}\sum_{h=1}^{H}\mathbb{V}\big(P_{s_{h}^{k},a_{h}^{k},h},\overline{V}_{h+1}^{k}\big)\nonumber\\
 & =\sum_{k=1}^{K}\sum_{h=1}^{H}\left(\big\langle P_{s_{h}^{k},a_{h}^{k},h},\,(\overline{V}_{h+1}^{k})^{2}\big\rangle-\big(\big\langle P_{s_{h}^{k},a_{h}^{k},h},\overline{V}_{h+1}^{k}\big\rangle\big)^{2}\right)\nonumber\\
 & =\sum_{k=1}^{K}\sum_{h=1}^{H}\big\langle P_{s_{h}^{k},a_{h}^{k},h}-e_{s_{h+1}^{k}},\,(\overline{V}_{h+1}^{k})^{2}\big\rangle\nonumber\\
 & \qquad+\sum_{k=1}^{H}\sum_{h=1}^{H}\left(\big(\overline{V}_{h+1}^{k}(s_{h+1}^{k})\big)^{2}-\big(\overline{V}_{h}^{k}(s_{h}^{k})\big)^{2}\right)+\sum_{k=1}^{K}\sum_{h=1}^{H}\left(\big(\overline{V}_{h}^{k}(s_{h}^{k})\big)^{2}-\big(\big\langle P_{s_{h}^{k},a_{h}^{k},h},\overline{V}_{h+1}^{k}\big\rangle\big)^{2}\right)\nonumber\\
 & \leq2\sqrt{8H^{4}\sum_{k=1}^{K}\sum_{h=1}^{H}\mathbb{V}\big(P_{s_{h}^{k},a_{h}^{k},h},\overline{V}_{h+1}^{k}\big)\log\frac{1}{\delta'}}+2H^{2}\sum_{k=1}^{K}\sum_{h=1}^{H}\overline{R}_{h}(s_{h}^{k},a_{h}^{k})+3H^{4}\log\frac{1}{\delta'}
	\label{eq:xllll2}
\end{align}
with probability at least $1-2SAHK\delta'$.  
Here, the last inequality arises from Lemma~\ref{lemma:self-norm} and Lemma~\ref{lemma:sqv} as well as
the fact that $\overline{V}_h^k(s_h^k) = \overline{R}_h(s_h^k,a_h^k)+P_{s_h^k,a_h^k,h}\overline{V}_{h+1}^k$.
It then follows from elementary algebra that
\begin{align}
\sum_{k=1}^K \sum_{h=1}^H \mathbb{V}(P_{s_h^k,a_h^k,h},\overline{V}_{h+1}^k)  \leq 4H^2\sum_{k=1}^K \sum_{h=1}^H \overline{R}_h(s_h^k,a_h^k) + 42 H^4\log \frac{1}{\delta'}.\label{eq:xlll3}
\end{align}

Substituting \eqref{eq:xlll3} into \eqref{eq:xlll1} gives
\begin{align}
\sum_{k=1}^K \sum_{h=1}^H \mathbb{V}(P_{s_h^k,a_h^k,h},V_{h+1}^{\star}) \leq \sum_{k=1}^H \overline{V}_1^k(s_1^k) +2\sqrt{8 H^2 \sum_{k=1}^K \sum_{h=1}^H \mathbb{V}(P_{s_h^k,a_h^k,h},V_{h+1}^{\star}) \log \frac{1}{\delta'}     } + 21H^2\log \frac{1}{\delta'},
	\nonumber
\end{align}
thus indicating that
\begin{align}
\sum_{k=1}^K \sum_{h=1}^H \mathbb{V}(P_{s_h^k,a_h^k,h},V_{h+1}^{\star}) \leq 2\sum_{k=1}^K \overline{V}_1^k(s_1^k)+ 84H^2\log \frac{1}{\delta'}
	\leq 2K\mathrm{var}_1 + 84H^2\log \frac{1}{\delta'} .\nonumber
\end{align}
The proof of Lemma~\ref{lemma:k1} is thus completed. 
\end{proof}

\begin{proof}[Proof of Lemma~\ref{lemma:bdv1}]
We make the observation that 
\begin{align}
 & \sum_{k,h}\mathbb{V}\big(P_{s_{h}^{k},a_{h}^{k},h},V_{h+1}^{k}-V_{h+1}^{\star}\big)\nonumber\\
 & =\sum_{k,h}\left(\big\langle P_{s_{h}^{k},a_{h}^{k},h},\,(V_{h+1}^{k}-V_{h+1}^{\star})^{2}\big\rangle-\big(\big\langle P_{s_{h}^{k},a_{h}^{k},h},\,V_{h+1}^{k}-V_{h+1}^{\star}\big\rangle\big)^{2}\right)\notag\\
 & =\sum_{k,h}\left(\big\langle P_{s_{h}^{k},a_{h}^{k},h}-e_{s_{h+1}^{k}},\,(V_{h+1}^{k}-V_{h+1}^{\star})^{2}\big\rangle\right)\nonumber\\
 & \qquad\qquad+\sum_{k,h}\left((V_{h+1}^{k}(s_{h+1}^{k})-V_{h+1}^{\star}(s_{h+1}^{k}))^{2}-\big(\big\langle P_{s_{h}^{k},a_{h}^{k},h},\,V_{h+1}^{k}-V_{h+1}^{\star}\big\rangle\big)^{2}\right)\nonumber\\
 & =\sum_{k,h}\left(\big\langle P_{s_{h}^{k},a_{h}^{k},h}-e_{s_{h+1}^{k}},\,(V_{h+1}^{k}-V_{h+1}^{\star})^{2}\big\rangle\right)+\sum_{k,h}\left((V_{h}^{k}(s_{h}^{k})-V_{h}^{\star}(s_{h}^{k}))^{2}-\big(\big\langle P_{s_{h}^{k},a_{h}^{k},h},\,V_{h+1}^{k}-V_{h+1}^{\star}\big\rangle\big)^{2}\right).
	\label{eq:var1}
 \end{align}
%
%
According to Lemma~\ref{lemma:self-norm} and Lemma~\ref{lemma:sqv}, we see that with probability exceeding $1-\delta'$, 
\begin{align}
	&\sum_{k,h}     \big\langle P_{s_h^k,a_h^k,h}-e_{s_{h+1}^k},\, (V_{h+1}^k - V_{h+1}^{\star})^2   \big\rangle \notag\\
	&\quad \leq 2\sqrt{2}\sqrt{4H^2\sum_{k,h}\mathbb{V}\big(P_{s_h^k,a_h^k,h}, V_{h+1}^k - V_{h+1}^{\star} \big)\log\frac{1}{\delta'}} + 3H^2\log \frac{1}{\delta'}.
	\label{eq:var3}
\end{align}
In addition, with probability at least $1-\delta'$ one has
 \begin{align}
 & \sum_{k,h}\left(\big(V_{h}^{k}(s_{h}^{k})-V_{h}^{\star}(s_{h}^{k})\big)^{2}-\big(\big\langle P_{s_{h}^{k},a_{h}^{k},h},\,V_{h+1}^{k}-V_{h+1}^{\star}\big\rangle\big)^{2}\right)\nonumber\\
 & \leq2H\sum_{k,h}\max\big\{ V_{h}^{k}(s_{h}^{k})-\big\langle P_{s_{h}^{k},a_{h}^{k},h},V_{h+1}^{k}\big\rangle-\big(V_{h}^{\star}(s_{h}^{k})-\big\langle P_{h}^{k},V_{h+1}^{\star}\big\rangle\big),0\big\}\nonumber\\
 & \leq2H\sum_{k,h}\max\big\{ V_{h}^{k}(s_{h}^{k})-\big\langle P_{s_{h}^{k},a_{h}^{k},h},V_{h+1}^{k}\big\rangle-r_{h}(s_{h}^{k},a_{h}^{k}),\,0\big\}\nonumber\\
 & \leq2H\sum_{k,h}\max\big\{\big\langle\widehat{P}_{s_{h}^{k},a_{h}^{k},h}-P_{s_{h}^{k},a_{h}^{k},h},V_{h+1}^{k}\big\rangle,0\big\}+2H\sum_{k,h}b_{h}^{k}\nonumber\\
 & \leq2\sqrt{BSAH^{3}\sum_{k,h}\mathbb{V}\big(P_{s_{h}^{k},a_{h}^{k},h},V_{h+1}^{k}\big)}+2H\sum_{k,h}b_{h}^{k}(s_{h}^{k},a_{h}^{k})+BSAH^{3}.
	 \label{eq:var4}
 \end{align}
It then follows that, with probability at least $1-2\delta'$,
\begin{align}
 & \sum_{k,h}\mathbb{V}\big(P_{s_h^k,a_h^k,h}, V_{h+1}^k -V_{h+1}^{\star}\big)  \leq 4\sqrt{BSAH^3\sum_{k,h}\mathbb{V}\big(P_{s_h^k,a_h^k,h},V_{h+1}^k\big)}+ 4H\sum_{k,h}b_h^k(s_h^k,a_h^k)+ 3BSAH^3, 
\end{align}
thereby concluding the proof.
\end{proof}

\subsubsection{Putting all this together}
To finish up, let us rewrite the inequalities $\eqref{eq:obt1}-\eqref{eq:obt8}$ as follows, 
with \eqref{eq:obt2}, \eqref{eq:obt4}, \eqref{eq:obt5} and \eqref{eq:obt6}  replaced by  \eqref{eq:nbt2}, \eqref{eq:nbt4}
\eqref{eq:nbt5} and \eqref{eq:nbt6}, respectively:
\begin{align}
& T_1 \leq \sqrt{128BSAHT_6}+24BSAH^2, \nonumber
\\ & T_7 \leq H\sqrt{512BSAHT_6}+24BSAH^3, \nonumber
\\ & T_9 \leq \sqrt{128BSAHT_6}+24BSAH^2, \nonumber
\\ & T_2\leq 100 \sqrt{B SAHT_5}+140BSAH^2, \nonumber
 \\ & T_3 \leq \sqrt{8BT_6}+3H\log\frac{1}{\delta'}  ,\nonumber
 \\ & T_4 \leq \sqrt{BSAHT_{10}}+32\sqrt{BH(T_1+T_2+T_3)}+BSAH^2,\nonumber
 \\ & T_5 \leq 40K\mathrm{var}_1 + 80HT_2+398BSAH^3  ,\nonumber
 \\ &  T_6 \leq  8K\mathrm{var}_1 + 16HT_2 + 78BSAH^3 ,\nonumber
 \\ & T_8 \leq \sqrt{32BH^2T_6 } + 3BH^2 ,\nonumber
\end{align}
where we recall that $B =4000 (\log_2K)^3\log(3SA)\log \frac{1}{\delta'} $. 
In addition, it follows from Lemma~\ref{lemma:k1} that 
$$ T_{10} \leq 2K\mathrm{var}_1 + 80BH^2.$$ 
Solving the inequalities above reveals that, with probability exceeding $1-200SAH^2K^2\delta'$, 
\begin{align}
\mathsf{Regret}(K)= T_1+T_2+T_3+T_4 \leq O\left( \sqrt{BSAHK\mathrm{var_1} }+ BSAH^2 \right).\label{eq:rbvar1}
\end{align}
One can thus conclude the proof by recalling that $\delta'=\frac{\delta}{200SAH^2K^2}$.

\subsection{Proof of Lemma~\ref{lemma:var2}}

Following similar arguments as in the proof of Lemma~\ref{lemma:var1}, 
we focus on bounding $T_2,T_4,T_5$ and $T_6$ in terms of $\mathrm{var}_2$.

\subsubsection{Bounding $T_2$}

Recall that $\delta'$ is defined as $\delta' = \frac{\delta}{200SAH^2K^2}$, 
and that we have demonstrated in \eqref{eq:boundt2o-temp} that
\begin{align}
T_2 &\leq \frac{460}{9}\sqrt{2SAH T_5 (\log_2K)\log \frac{1}{\delta'} }\nonumber
\\ & \qquad +4\sqrt{SAH(\log_2K)\log \frac{1}{\delta'}}\sqrt{\sum_{k,h}\left(\widehat{\sigma}_h^k(s_h^k,a_h^k)- \big(\widehat{r}_h^k(s_h^k,a_h^k) \big)^2\right)} 
	+ \frac{1088}{9}SAH^2(\log_2K)\log \frac{1}{\delta'}.\label{eq:local31}
 \end{align}
To bound the right-hand side of \eqref{eq:local31}, 
we resort to the following lemma. 
 \begin{lemma}\label{lemma:ks2}
With probability at least $1-4SAHK\delta'$, one has
\begin{align}
\sum_{k,h}\left(\widehat{\sigma}_h^k(s_h^k,a_h^k)- (\widehat{r}_h^k(s_h^k,a_h^k))^2\right) \leq 6K\mathrm{var}_2 + 242H^2(\log_2K)\log \frac{1}{\delta'}.
\end{align}
 \end{lemma}
\begin{proof}
Recall that in Lemma~\ref{lemma:bdrv}, we have shown that with probability at least $1-4SAHK\delta'$,
\begin{align}
    \sum_{k=1}^K\sum_{h=1}^H\left(\widehat{\sigma}_h^k(s_h^k,a_h^k)- \big(\widehat{r}_h^k(s_h^k,a_h^k) \big)^2\right)\leq 3\sum_{k=1}^K  \widetilde{V}_1^k(s_1^k)+2SAH^3(\log_2K)\log \frac{1}{\delta'}.
\end{align}
We then complete the proof by observing that
\begin{align}
	\widetilde{V}_1^k(s_1^k) &\leq  \widetilde{V}_1^k(s_1^k)+\mathbb{E}_{\pi^k}\left[\sum_{h=1}^H \mathbb{V}\big(P_{s_h,a_h,h},V_{h+1}^{\pi^k} \big) \,\Big|\, s_1=s_1^k\right] \notag\\
	&= \mathsf{Var}_{\pi^k}\left[\sum_{h=1}^H r_h(s_h,a_h) \,\Big|\, s_1=s_1^k\right]\leq \mathrm{var}_2.
\end{align}
\end{proof}
Combining Lemma~\ref{lemma:ks2} with \eqref{eq:local31} gives: with probability at least $1-4SAHK\delta'$,
\begin{align}
	T_2 &\leq \frac{460}{9}\sqrt{2SAH T_5 (\log_2K)\log \frac{1}{\delta'} } +12\sqrt{SAH(\log_2K)\log \frac{1}{\delta'}}\sqrt{2K\mathrm{var}_2}  \notag\\
	&\qquad\qquad + 157SAH^2(\log_2K)\log \frac{1}{\delta'}.\label{eq:xbt2}
\end{align}

\subsubsection{Bounding $T_4$}

Recall that $T_4 = \widecheck{T}_1 + \widecheck{T}_2$, 
where 
\begin{align*}
	\widecheck{T}_1 &= \sum_{k=1}^K\sum_{h=1}^H \left(\widehat{r}_h^k(s_h^k,a_h^k) -r_h(s_h^k,a_h^k) \right),\\
	\widecheck{T}_2 &= \sum_{k=1}^K\left(\sum_{h=1}^H r_h(s_h^k,a_h^k) -V_{1}^{\pi^k}(s_1^k) \right).
\end{align*}
Repeating similar arguments employed in the proof of Lemma~\ref{lemma:bdrv} and using \eqref{eq:wc3}, we see that with probability exceeding $1-6SAHK\delta'$, 
\begin{align}
	|\widecheck{T}_1| & \leq \sqrt{4SAH (\log_2 K)\log \frac{1}{\delta'}}\cdot \sqrt{\sum_{k=1}^K\sum_{h=1}^H v_h(s_h^k,a_h^k)} + 2SAH^2(\log_2K)\log \frac{1}{\delta'} \nonumber
\\ & \leq \sqrt{8SAHK\mathrm{var}_2(\log_2K)\log \frac{1}{\delta'}}+ 20SAH^2(\log_2K)\log \frac{1}{\delta'}.\nonumber
\end{align}
In addition, from Lemma~\ref{lemma:self-norm} and the definition of $\mathrm{var}_2$, we see that 
\begin{align}
|\widecheck{T}_2|\leq 2\sqrt{2K\mathrm{var}_2 \log \frac{1}{\delta'}}+3H\log \frac{1}{\delta'} 
\end{align}
with probability at least $1-2SAHK\delta'$. 
Therefore, with probability at least $1-8SAHK\delta'$, it holds that
\begin{align}
T_4 \leq 4\sqrt{2SAHK\mathrm{var}_2(\log_2K)\log \frac{1}{\delta'}}+ 23SAH^2(\log_2K)\log \frac{1}{\delta'}.\label{eq:xbt4}
\end{align}

\subsubsection{Bounding $T_5$ and $T_6$}

Recall that Lemma~\ref{lemma:empv} asserts that with probability exceeding $1-2\delta'$, 
$$
	T_5\leq 5T_6+8BSAH^3.
$$
Hence, it suffices to bound $T_6$.

From the elementary inequality $\mathsf{Var}(X+Y)\leq 2\mathsf{Var}(X)+ 2\mathsf{Var}(Y)$, we obtain
\begin{align}
	T_6 &=\sum_{k,h} \mathbb{V}\big(P_{s_h^k,a_h^k,h},V_{h+1}^k\big)   \leq  2\sum_{k,h} \mathbb{V}\big(P_{s_h^k,a_h^k,h},V_{h+1}^{\pi^k} \big) 
	+ 2\sum_{k,h}\mathbb{V}\big(P_{s_h^k,a_h^k,h},V_{h+1}^k -V_{h+1}^{\pi^k}\big )\nonumber
\\ 
	& \leq 3   K\mathrm{var}_2 +\sum_{k=1}^K \left( \sum_{h=1}^H \mathbb{V}\big(P_{s_h^k,a_h^k,h},V_{h+1}^{\pi^k} \big) - 3\mathrm{var}_2 \right) + 2\sum_{k,h}\mathbb{V}\big(P_{s_h^k,a_h^k,h},V_{h+1}^k -V_{h+1}^{\pi^k} \big).\label{eq:var01}
\end{align}
To bound the right-hand side of \eqref{eq:var01}, we resort to the following two lemmas.
\begin{lemma}\label{lemma:kx1}
With probability  at least $1-4SAHK\delta'$, it holds that
\begin{align}
\sum_{k=1}^K \left( \sum_{h=1}^H \mathbb{V}(P_{s_h^k,a_h^k,h},V_{h+1}^{\pi^k} ) -2\mathrm{var}_2 \right)  \leq 80H^2\log \frac{1}{\delta'}.
\end{align}
\end{lemma}
\begin{lemma}\label{lemma:bdv11}
With probability exceeding $1-4SAKH\delta'$, it holds that
\begin{align}
 & \sum_{k,h}\mathbb{V}\big(P_{s_h^k,a_h^k,h}, V_{h+1}^k -V_{h+1}^{\pi^k}\big)   \leq 4\sqrt{BH^2\sum_{k,h}\mathbb{V}\big(P_{s_h^k,a_h^k,h},V_{h+1}^k\big)}+ 4H\sum_{k,h}b_h^k(s_h^k,a_h^k)+ 3BSAH^3.\nonumber
\end{align}
\end{lemma}
With Lemma~\ref{lemma:kx1} and Lemma~\ref{lemma:bdv11} in place, we can demonstrate that with probability at least $1-6SAHK\delta'$,
\begin{align}
T_6&
	\leq  2\sum_{k,h}\mathbb{V}\big(P_{s_h^k,a_h^k,h},V_{h+1}^{\pi^k}\big)+2\sum_{k,h}\mathbb{V}\big(P_{s_h^k,a_h^k,h},V_{h+1}^k - V_{h+1}^{\pi^k}\big)\nonumber
\\ & \leq 4K\mathrm{var}_2 +  8\sqrt{BSAH^3 T_6}+8HT_2 + 7BSAH^3,\nonumber
\end{align}
\begin{align}
\Longrightarrow \qquad 
	T_6 \leq 8K\mathrm{var}_2 + 16HT_2 + 78BSAH^3.\label{eq:xbt6}
\end{align}
Taking this result together with  Lemma~\ref{lemma:empv} gives, with probability exceeding $1-8SAHK\delta'$, 
\begin{align}
T_5 =\sum_{k,h}\mathbb{V}\big(\widehat{P}_{s_h^k,a_h^k,h},V_{h+1}^k\big)\leq 40K\mathrm{var}_2 + 80HT_2+398BSAH^3.\label{eq:xbt5}
\end{align}
To finish establishing the above bounds on $T_5$ and $T_6$, it suffices to prove Lemma~\ref{lemma:kx1} and Lemma~\ref{lemma:bdv11}, 
which we accomplish in the sequel.

%
%
%

\begin{proof}[Proof of Lemma~\ref{lemma:kx1}]
For notational convenience, define 
\begin{align}
	\widecheck{R}^k_{h}(s,a) = \mathbb{V}(P_{s,a,h},V_{h+1}^{\pi^k}) \qquad \text{and} \qquad
\widecheck{V}^k_h (s) = \mathbb{E}\left[ \sum_{h'=h}^H \widecheck{R}^k_{h'}(s_{h'},a_{h'}) \,\Big|\, s_h = s\right].\nonumber
\end{align}
It is easily seen that $\widecheck{V}_h^k(s)\leq\mathrm{var}_2 \leq  H^2$.

We also make the observation that 
\begin{align}
  \sum_{h=1}^H  \mathbb{V}(P_{s_h^k,a_h^k,h},V_{h+1}^{\pi^k})-\mathrm{var}_2 \nonumber & =\sum_{h=1}^H  \widecheck{R}_h^k(s_h^k,a_h^k)-\mathrm{var}_2  \nonumber
 \\ & \leq \sum_{h=1}^H \widecheck{R}_h^k(s_h^k,a_h^k) - \widecheck{V}^k_1(s_1^k)\nonumber
 \\ & = \sum_{h=1}^H \big\langle e_{s_{h+1}^k} - P_{s_{h}^k,a_h^k,h},\,\widecheck{V}^k_{h+1}   \big\rangle.
\end{align}
Note that $\widecheck{V}^k$ only depends on $\pi^k$, which is determined before the $k$-th episode starts.  
 Lemma~\ref{lemma:self-norm} then tells us that, with probability at least $1-2SAHK\delta'$,
\begin{align}
 & \sum_{k=1}^K \left( \sum_{h=1}^H \mathbb{V}\big(P_{s_h^k,a_h^k,h},V_{h+1}^{\pi^k}\big) - \widecheck{V}_1^k(s_1^k)  \right) \nonumber
 \\ & \leq 2\sqrt{2\sum_{k=1}^K \sum_{h=1}^H \mathbb{V}\big(P_{s_h^k,a_h^k,h},\widecheck{V}_{h+1}^k\big)\log \frac{1}{\delta'} } + 3H^2\log \frac{1}{\delta'}.\label{eq:xlll11}
\end{align}
Further, it is observed that with probability at least $1-2SAHK\delta'$, 
\begin{align}
 & \sum_{k=1}^{K}\sum_{h=1}^{H}\mathbb{V}\big(P_{s_{h}^{k},a_{h}^{k},h},\widecheck{V}_{h+1}^{k}\big)\nonumber\\
 & =\sum_{k=1}^{K}\sum_{h=1}^{H}\left(\big\langle P_{s_{h}^{k},a_{h}^{k},h},(\widecheck{V}_{h+1}^{k})^{2}\big\rangle-\big(\big\langle P_{s_{h}^{k},a_{h}^{k},h},\widecheck{V}_{h+1}^{k}\big\rangle\big)^{2}\right)\nonumber\\
 & =\sum_{k=1}^{K}\sum_{h=1}^{H}\big\langle P_{s_{h}^{k},a_{h}^{k},h}-e_{s_{h+1}^{k}},\,(\widecheck{V}_{h+1}^{k})^{2}\big\rangle\nonumber\\
 & \qquad+\sum_{k=1}^{H}\sum_{h=1}^{H}\left( \big(\widecheck{V}_{h+1}^{k}(s_{h+1}^{k}) \big)^{2}- \big(\widecheck{V}_{h}^{k}(s_{h}^{k})\big)^{2}\right)
	+\sum_{k=1}^{K}\sum_{h=1}^{H}\left(\big(\widecheck{V}_{h}^{k}(s_{h}^{k})\big)^{2}-\left(\big\langle P_{s_{h}^{k},a_{h}^{k},h},\widecheck{V}_{h+1}^{k}\big\rangle\right)^{2}\right)\nonumber\\
 & \leq2\sqrt{8H^{4}\sum_{k=1}^{K}\sum_{h=1}^{H}\mathbb{V}\big(P_{s_{h}^{k},a_{h}^{k},h},\widecheck{V}_{h+1}^{k}\big)\log\frac{1}{\delta'}}+2H^{2}\sum_{k=1}^{K}\sum_{h=1}^{H}\widecheck{R}_{h}(s_{h}^{k},a_{h}^{k})+3H^{4}\log\frac{1}{\delta'}.
	\label{eq:xllll21}
\end{align}
Here, the last inequality results from Lemma~\ref{lemma:self-norm}, Lemma~\ref{lemma:sqv}  and the fact that 
	$\widecheck{V}_h^k(s_h^k) = \widecheck{R}_h(s_h^k,a_h^k)+ \langle P_{s_h^k,a_h^k,h}, \widecheck{V}_{h+1}^k \rangle$. 
It then follows that
\begin{align}
\sum_{k=1}^K \sum_{h=1}^H \mathbb{V}\big(P_{s_h^k,a_h^k,h},\widecheck{V}_{h+1}^k \big)  \leq 4H^2\sum_{k=1}^K \sum_{h=1}^H \widecheck{R}_h(s_h^k,a_h^k) + 42 H^4\log \frac{1}{\delta'}.\label{eq:xlll31}
\end{align}
Taking \eqref{eq:xlll11} and \eqref{eq:xlll31} together leads to
\begin{align}
\sum_{k=1}^K \sum_{h=1}^H \mathbb{V}\big(P_{s_h^k,a_h^k,h},V_{h+1}^{\pi^k}\big) 
	\leq \sum_{k=1}^H \widecheck{V}_1^k(s_1^k) +2\sqrt{8 H^2 \sum_{k=1}^K \sum_{h=1}^H \mathbb{V}\big(P_{s_h^k,a_h^k,h},V_{h+1}^{\pi^k}\big) \log \frac{1}{\delta'}    } + 21H^2\log \frac{1}{\delta'},\nonumber
\end{align}
which further implies that
\begin{align}
	\sum_{k=1}^K \sum_{h=1}^H \mathbb{V}\big(P_{s_h^k,a_h^k,h},V_{h+1}^{\pi^k}\big) 
	\leq 2\sum_{k=1}^K \widecheck{V}_1^k(s_1^k)+ 84H^2\log \frac{1}{\delta'} 
	\leq 2K\mathrm{var}_2 + 84H^2\log \frac{1}{\delta'}.\nonumber
\end{align}
This concludes the proof. 
\end{proof}

%
%
\begin{proof}[Proof of Lemma~\ref{lemma:bdv1}]
A little algebra gives
\begin{align}
 & \sum_{k,h}\mathbb{V}\big(P_{s_{h}^{k},a_{h}^{k},h},V_{h+1}^{k}-V_{h+1}^{\pi^{k}}\big)\nonumber\\
 & =\sum_{k,h}\left(\big\langle P_{s_{h}^{k},a_{h}^{k},h},\,(V_{h+1}^{k}-V_{h+1}^{\pi^{k}})^{2}\big\rangle-\big(\big\langle P_{s_{h}^{k},a_{h}^{k},h},V_{h+1}^{k}-V_{h+1}^{\pi^{k}}\big\rangle\big)^{2}\right)\notag\\
 & =\sum_{k,h}\left(\big\langle P_{s_{h}^{k},a_{h}^{k},h}-e_{s_{h+1}^{k}},(V_{h+1}^{k}-V_{h+1}^{\pi^{k}})^{2}\big\rangle\right)\nonumber\\
 & \qquad\qquad+\sum_{k,h}\left(\big(V_{h+1}^{k}(s_{h+1}^{k})-V_{h+1}^{\pi^{k}}(s_{h+1}^{k})\big)^{2}-\big(\big\langle P_{s_{h}^{k},a_{h}^{k},h},V_{h+1}^{k}-V_{h+1}^{\pi^{k}}\big\rangle\big)^{2}\right)\nonumber\\
 & =\sum_{k,h}\left(\big\langle P_{s_{h}^{k},a_{h}^{k},h}-e_{s_{h+1}^{k}},(V_{h+1}^{k}-V_{h+1}^{\pi^{k}})^{2}\big\rangle\right)+\sum_{k,h}\left(\big(V_{h}^{k}(s_{h}^{k})-V_{h}^{\pi^{k}}(s_{h}^{k})\big)^{2}-\big(\big\langle P_{s_{h}^{k},a_{h}^{k},h},V_{h+1}^{k}-V_{h+1}^{\pi^{k}}\big\rangle\big)^{2}\right).
	\label{eq:var11}
 \end{align}
From Lemma~\ref{lemma:self-norm} and Lemma~\ref{lemma:sqv}, we can show that with probability $1-2SAKH\delta'$, 
\begin{align}
 & \sum_{k,h}\big\langle P_{s_{h}^{k},a_{h}^{k},h}-e_{s_{h+1}^{k}},(V_{h+1}^{k}-V_{h+1}^{\pi^{k}})^{2}\big\rangle\\
 & \qquad\leq2\sqrt{2}\sqrt{4H^{2}\sum_{k,h}\mathbb{V}\big(P_{s_{h}^{k},a_{h}^{k},h},V_{h+1}^{k}-V_{h+1}^{\pi^{k}}\big)\log\frac{1}{\delta'}}+3H^{2}\log\frac{1}{\delta'}. 
	\label{eq:var31}
\end{align}
Additionally, with probability at least $1-2SAKH\delta'$, 
 \begin{align}
  & \sum_{k,h}\left\{ \big(V_{h}^{k}(s_{h}^{k})-V_{h}^{\pi^{k}}(s_{h}^{k})\big)^{2}-\big(\big\langle P_{s_{h}^{k},a_{h}^{k},h},V_{h+1}^{k}-V_{h+1}^{\pi^{k}}\big\rangle)^{2}\right\} \nonumber\\
 & \leq2H\sum_{k,h}\max\Big\{ V_{h}^{k}(s_{h}^{k})-\big\langle P_{s_{h}^{k},a_{h}^{k},h},V_{h+1}^{k}\big\rangle-\big(V_{h}^{\pi^{k}}(s_{h}^{k})-\big\langle P_{h}^{k},V_{h+1}^{\pi^{k}}\big\rangle\big),0\Big\}\nonumber\\
 & =2H\sum_{k,h}\max\Big\{ V_{h}^{k}(s_{h}^{k})-\big\langle P_{s_{h}^{k},a_{h}^{k},h},V_{h+1}^{k}\big\rangle-r_{h}(s_{h}^{k},a_{h}^{k}),0\Big\}\nonumber\\
 & \leq2H\sum_{k,h}\max\Big\{\big\langle\widehat{P}_{s_{h}^{k},a_{h}^{k},h}-P_{s_{h}^{k},a_{h}^{k},h},V_{h+1}^{k}\big\rangle,0\Big\}+2H\sum_{k,h}b_{h}^{k}(s_{h}^{k},a_{h}^{k})\nonumber\\
 & \leq2\sqrt{BSAH^{3}\sum_{k,h}\mathbb{V}\big(P_{s_{h}^{k},a_{h}^{k},h},V_{h+1}^{k}\big)}+2H\sum_{k,h}b_{h}^{k}(s_{h}^{k},a_{h}^{k})+BSAH^{3}.
	 \label{eq:var41}
 \end{align}
It then follows that
\begin{align}
 & \sum_{k,h}\mathbb{V}(P_{s_h^k,a_h^k,h}, V_{h+1}^k -V_{h+1}^{\pi^k})  \leq 4\sqrt{BSAH^3\sum_{k,h}\mathbb{V}(P_{s_h^k,a_h^k,h},V_{h+1}^k)}+ 4H\sum_{k,h}b_h^k(s_h^k,a_h^k)+ 3BSAH^3
\end{align}
with probability at least $1-4SAKH\delta'$. 
The proof is thus complete.
\end{proof}

\subsubsection{Putting all pieces together}

Recall that $B =4000 (\log_2 K)^3 \log(3SA)\log \frac{1}{\delta'}$. 
The last step is to rewrite the inequalities $\eqref{eq:obt1}-\eqref{eq:obt8}$ as follows with \eqref{eq:obt2}, \eqref{eq:obt4}, \eqref{eq:obt5} and \eqref{eq:obt6}  replaced by  \eqref{eq:xbt2},\eqref{eq:xbt4}
\eqref{eq:xbt5} and \eqref{eq:xbt6} respectively:
\begin{align}
& T_1 \leq \sqrt{128BSAHT_6}+24BSAH^2,\nonumber
\\ & T_7 \leq H\sqrt{512BSAHT_6}+24BSAH^3,\nonumber
\\ & T_9 \leq \sqrt{128BSAHT_6}+24BSAH^2,\nonumber
\\ & T_2\leq 100 \sqrt{BSAHT_5}+140BSAH^2,\nonumber
 \\ & T_3 \leq \sqrt{8BT_6}+3H\log \frac{1}{\delta'} ,\nonumber
 \\ & T_4 \leq \sqrt{BSAHK\mathrm{var}_2}+BSAH^2;\nonumber
 \\ & T_5 \leq 40K\mathrm{var}_2 + 80HT_2+398BSAH^3  ,\nonumber
 \\ &  T_6 \leq  8K\mathrm{var}_2 + 16HT_2 + 78BSAH^3 ,\nonumber
 \\ & T_8 \leq \sqrt{32BH^2T_6 } + 3BH^2 ,\nonumber
\end{align}
which are valid with probability at least $1-200SAH^2K^2\delta'$. 
Solving the inequalities listed above, we can readily conclude that
\begin{align}
\mathsf{Regret}(K)= T_1+T_2+T_3+T_4 \leq O\left( \sqrt{BSAHK\mathrm{var_2} }+ BSAH^2 \right).\label{eq:rbvar2}
\end{align}
This finishes the proof by recalling that $\delta' = \frac{\delta}{200SAH^2K^2}$.

\section{Minimax lower bounds}\label{app:lb}

In this section, we establish the lower bounds advertised in this paper. 

\subsection{Proof of Theorem~\ref{thm:lb1}}\label{app:lbf}

Consider any given $(S,A,H)$.   We start by establishing the following lemma. 
\begin{lemma}\label{lemma:lb2}
Consider any $K'\geq 1$. 
For any algorithm, there exists an MDP instance with $S$ states, $A$ actions, and horizon $H$, such that the regret in $K'$ episodes is at least 
\begin{align}
\mathsf{Regret}(K')=\Omega\big( f(K')\big) = \Omega\left(\min\big\{\sqrt{SAH^3K'},K'H\big\}\right).\nonumber
\end{align}
\end{lemma}
\begin{proof}[Proof of Lemma~\ref{lemma:lb2}]
	Our construction of the hard instance is based on the hard instance JAO-MDP constructed in  \citet{jaksch2010near,jin2018q}. In  \citet[Appendix.D]{jin2018q}, the authors already showed that when $K'\geq C_0SAH$ for some constant $C_0>0$, the minimax regret lower bound is $\Omega(\sqrt{SAH^3K'})$. 
	Hence, it suffices for us to focus on the regime where $K'\leq C_0SAH$. Without loss of generality, we assume $S=A=2$, and the argument to generalize it to arbitrary $(S,A)$ is standard and henc omitted for brevity.

	Recall the construction of JAO-MDP in \citet{jaksch2010near}. Let the two states be $x$ and $y$, and the two actions be $a$ and $b$. 
	The reward  is always equal to $x$ in state $1$ and $1/2$ in state $y$. The probability transition kernel is given by 
	$$
		P_{x,a} =P_{x,b}= [1-\delta,\delta], 
		~~~P_{y,a} = [1-\delta,\delta], ~~~
		P_{y,b}= [1-\delta -\epsilon,\delta+\epsilon],
	$$
	where we choose $\delta = C_1 / H$ and $\epsilon =1/H$. Then the mixing time of the MDP is roughly $O(H)$. By choosing $C_1$ large enough, 
	we can ensure that the MDP is $C_3$-mixing after the first half of the horizons for some proper constant $C_3\in (0,1/2)$.

It is then easy to show that action $b$ is the optimal action for state $y$. 
	Moreover, whenever action $a$ is chosen in state $y$, the learner needs to pay regret $\Omega(\epsilon H)=\Omega(1)$. 
	In addition, to differentiate action $a$ from action $b$ in state $y$ with probability at least $1-\frac{1}{10}$, the learner needs at least $\Omega(\frac{\epsilon}{\delta^2}) = \Omega(H)$ rounds --- let us call it $C_4H$ rounds for some proper constant $C_4>0$. As a result, in the case where $K'\leq C_4H$, the minimax regret is at least $\Omega(K'H^2\epsilon)=\Omega(K'H)$. When $C_4H \leq K' \leq C_0SAH = 4C_0H$, the minimax regret is at least $\Omega(C_4H^2)=\Omega(K'H)$. This concludes the proof.
\end{proof}


With Lemma~\ref{lemma:lb2}, we are ready to prove Theorem~\ref{thm:lb1}. 
Let $\mathcal{M}$ be the hard instance  for $K' = \max\left\{\frac{1}{10}Kp ,1\right\}$ constructed in the proof of Lemma~\ref{lemma:lb2}. 
We construct an MDP $\mathcal{M}'$ as below. 
\begin{itemize}
	\item	In the first step, for any state $s$, with probability $p$, the leaner transitions to a copy of $\mathcal{M}$, and with probability $1-p$, the learner transitions to a dumb state with $0$ reward. 
\end{itemize}
It can be easily verified that $v^{\star}\leq pH$.
Let $X=X_1+X_2+\dots+X_k$, where $\{X_i\}_{i=1}^K$ are i.i.d.~Bernoulli random variables with mean $p$. Let $g(X,K')$ denote the minimax regret on the hard instance $\mathcal{M}$ in $X$ episodes. Given that $g(X,K')$ is non-decreasing in $X$,  
one sees that $$\mathsf{Regret}(K) \geq  \mathbb{E}\big[g(X,K') \big].$$ 
\begin{itemize}
	\item
In the case where $Kp\geq 10$,  Lemma~\ref{lemma:con} tells us that with probability at least $1/2$, $X\geq \frac{1}{10}Kp = K'$, and then it holds that 
		$$
			\mathbb{E}\big[g(X,K')\big] \geq \frac{1}{2} g(K',K')=\frac{1}{2}f(K')= \frac{1}{2} \Omega\left(\min\left\{\sqrt{SAH^3K'},K'H\right\}\right) = \Omega(\sqrt{SAH^3Kp},KHp).
		$$
	\item 
		In the case where $Kp<10$, with probability exceeding $1-(1-p)^K \geq (1-e^{-Kp})\geq  \frac{Kp}{30}$, one has $X\geq 1$. Then one has 
		$$
			\mathbb{E}\big[g(X,K')\big]\geq \frac{Kp}{30}\cdot g(1,K')=\frac{Kp}{30}\cdot g(1,1) =\Omega(KHp). 
		$$
\end{itemize}
The preceding bounds taken together complete the proof.

\subsection{Proof of Theorem~\ref{corollary:costlb}}\label{app:lbc}

Without loss of generality, assume that $S=A=2$ (as in the proof of Theorem~\ref{thm:lb1}), and recall the assumption that $p\leq 1/4$. 
In what follows, we construct a hard instance for which the learner needs to identify the correct action for each step.

Let $\mathcal{S}=\{s_1,s_2\}$, and take the initial state to be $s_1$. The transition kernel and cost are chosen as follows.  
\begin{itemize}
\item Select $a_h^{\star}\in \{a_1,a_2\}$ for every $h\in [H]$.
	\item For each  action $a$ and each step $h$,  set $P_{s_2,a,h} = e_{s_2}$ and $c_h(s_2,a)=0$. 
	\item For each step $h$ and each  action $a\neq a^{\star}_h$,  set $P_{s_1,a,h} = e_{s_2}$ and $c_h(s_1,a)=1$. 
	\item Set $P_{s_1,a^{\star}_h,h} = e_{s_1}$ and $c_h(s_1,a_h^{\star}) = p$. 
\end{itemize}
It can be easily checked that $c^{\star} =Hp$ (the cost obtained by choosing action $a_h^{\star}$ for each step $h$). 

Note that in the above construction, 
the $a_h^{\star}$'s are selected independently across different steps. 
Thus, to identify the optimal action $a_h^{\star}$ for at least half of the $H$ steps, we need at least $\Omega(H)$ episodes. This implies that: there exists a constant $C_5>0$ such that in the first $K\leq C_5H$ episodes, the cost of the learner is at least $\Omega(H(1-p))$.  As a result, the minimax regret is at least 
$$
	\Omega\big(K(H-c^{\star}) \big)=\Omega\big(KH(1-p) \big)
$$ 
when $K\leq C_5H$. 
Similarly, in the case where $C_5H \leq K \leq  \frac{100H}{p}$, the minimax regret is at least $$ \Omega\big(H(H-c^{\star})\big)= \Omega\big(H^2(1-p)\big).$$ 

We then turn to the case where $K\geq \frac{100H}{p}$. Let $\mathcal{M}$ be the hard instance having the same transition as the instance constructed in the proof of Lemma~\ref{lemma:lb2}, and set the cost as $1/2$ (resp.~$1$) for state $x$ (resp.~state $y$), with respect to $K' = Kp/10 \geq 10H$ (a quantity defined therein).  
Let $\mathcal{M}'$ be the MDP such that: in the first step,  with probability $p$, the learner transitions to a copy of $\mathcal{M}$, and with probability $1-p$, the learner transitions to a dumb state with $0$ cost. Then we clearly have $c^{\star} =\Theta(Hp)$. 
It follows from Lemma~\ref{lemma:con} that, with probability exceeding $1/2$, one has $ X\geq \frac{1}{3}Kp - \log 2 \geq \frac{1}{6}Kp$, where $X$ is again defined in the proof of Lemma~\ref{lemma:lb2}. Then one has 
$$\mathsf{Regret}(K)  \geq 
\frac{1}{2} \Omega\left(\min\big\{\sqrt{H^3K'}, K'H \big\}\right)= \Omega\big(\sqrt{H^3Kp}\big).$$ 
The proof is thus completed by combining the above minimax regret lower bounds for the three regimes $K\in [1,C_5H]$, $K\in(C_5H,\frac{100H}{p}]$ and $K\in(\frac{100H}{p},\infty]$.


\subsection{Proof of Theorem~\ref{thm:lb3}}

When $K\geq SAH/p$, the lower bound in Theorem~\ref{thm:lb1} readily applies because the regret is at least $\Omega(\sqrt{SAH^3Kp})$ and the variance $\mathrm{var}$ is at most $pH^2$. 
When $SAH\leq K \leq SAH/p$, the regret is at least $\Omega(SAH^2)=\Omega(\min\{\sqrt{SAH^3Kp}+SAH^2,KH \})$. 
As a result, it suffices to focus on the case where  $1\leq K\leq SAH$, 
Towards this end, we only need the following lemma, which suffices for us to complete the proof.



\begin{lemma}\label{lemma:lb34}
Consider any $1\leq K\leq SAH$. There exists  an MDP instance with $S$ states, $A$ actions, horizon $H$, and $\mathrm{var}_1 = \mathrm{var}_2 = 0$, such that the regret is at least $\Omega(KH)$.
\end{lemma}

\begin{proof}
Let us construct an MDP with deterministic transition; more precisely, for each $(s,a,h)$, there is some $s'$ such that $P_{s,a,h,s'}=1$ and $P_{s,a,h,s''}=0$ for any $s''\neq s'$. 
	The reward function is also chosen to be deterministic. In this case, it is easy to verify that $\mathrm{var}_1 = \mathrm{var}_2 = 0$.

	We first assume $S=2$. For any action $a$ and horizon $h$, we set $P_{s_2,a,h} = e_{s_2}$ and $r_h(s_2,a)=0$. For any action $a\neq a^{\star}$ and $h$, we also set $P_{s_1,a,h} = e_{s_2}$ and $r_h(s_2,a)=0$. At last, we set $P_{s_1,a^{\star},h} = e_{s_1}$ and $r_h(s_1,a^{\star}) = 1$. In other words, there are a dumb state and a normal state in each step. The learner would naturally hope to find the correct action to avoid the dumb state. Obviously, $V_1^{\star}(s_1)=H$. To find an $\frac{H}{2}$-optimal policy, the learner needs to identify $a^{\star}$ for the first $\frac{H}{2}$ steps, requiring at least $\Omega(HA)$ rounds in expectation. As a result, the minimax regret is at least $\Omega(KH)$ when $K\leq cHA$ for some proper constant $c>0$.

Let us refer to the hard instance above as a \emph{hard chain}.
	For general $S$, we can construct $d \coloneqq \frac{S}{2}$ hard chains. Let the two states in the $i$-th hard chain be $(s_1(i),s_2(i) )$. We set the initial distribution to be the uniform distribution over $\{s_1(i)\}_{i=1}^d$. Then $V_1^{\star}(s_1(i))=H$ holds for any $1\leq i\leq d$.  Let $\mathsf{Regret}_i(K)$ be the expected regret resulting from the $i$-th hard chain. 
When $K\geq 100S$,  Lemma~\ref{lemma:con} tells us that with probability at least $\frac{1}{2}$, $s_1(i)$ is visited for at least $\frac{K}{10S}\geq 10$ times. As a result, we have $$\mathsf{Regret}_i(K)\geq \frac{1}{2}\cdot \Omega\left(\frac{KH}{S}\right).$$  Summing over $i$, we see that the total regret is at least $\sum_{i=1}^d \mathsf{Regret}_i(K) = \Omega(KH)$. When $K<100S$, with probability at least $1-(1-\frac{1}{S})^K\geq 0.0001\frac{K}{S}$, we know that $s_1(i)$ is visited for at least one time. Therefore, it holds that $\mathsf{Regret}_i(K)\geq \Omega(\frac{KH}{S})$. 
	Summing over $i$, we obtain 
	$$\mathsf{Regret}(K) = \sum_{i=1}^K \mathsf{Regret}_i(K) =\Omega(KH)$$
	as claimed. 
\end{proof}

\bibliography{ref}

\begin{thebibliography}{}

\bibitem[Agarwal et~al., 2020]{agarwal2020model}
Agarwal, A., Kakade, S., and Yang, L.~F. (2020).
\newblock Model-based reinforcement learning with a generative model is minimax
  optimal.
\newblock In {\em Conference on Learning Theory}, pages 67--83.

\bibitem[Agarwal et~al., 2017]{agarwal2017open}
Agarwal, A., Krishnamurthy, A., Langford, J., Luo, H., et~al. (2017).
\newblock Open problem: First-order regret bounds for contextual bandits.
\newblock In {\em Conference on Learning Theory}, pages 4--7.

\bibitem[Agrawal and Jia, 2017]{agralwal2017optimistic}
Agrawal, S. and Jia, R. (2017).
\newblock Optimistic posterior sampling for reinforcement learning: worst-case
  regret bounds.
\newblock In Guyon, I., Luxburg, U.~V., Bengio, S., Wallach, H., Fergus, R.,
  Vishwanathan, S., and Garnett, R., editors, {\em Advances in Neural
  Information Processing Systems 30}, pages 1184--1194. Curran Associates, Inc.

\bibitem[Allen-Zhu et~al., 2018]{allen2018make}
Allen-Zhu, Z., Bubeck, S., and Li, Y. (2018).
\newblock Make the minority great again: First-order regret bound for
  contextual bandits.
\newblock In {\em International Conference on Machine Learning}, pages
  186--194.

\bibitem[Azar et~al., 2013]{azar2013minimax}
Azar, M.~G., Munos, R., and Kappen, H.~J. (2013).
\newblock Minimax {PAC} bounds on the sample complexity of reinforcement
  learning with a generative model.
\newblock {\em Machine learning}, 91(3):325--349.

\bibitem[Azar et~al., 2017]{azar2017minimax}
Azar, M.~G., Osband, I., and Munos, R. (2017).
\newblock Minimax regret bounds for reinforcement learning.
\newblock In {\em Proceedings of the 34th International Conference on Machine
  Learning}, pages 263--272.

\bibitem[Bai et~al., 2019]{bai2019provably}
Bai, Y., Xie, T., Jiang, N., and Wang, Y.-X. (2019).
\newblock Provably efficient {Q}-learning with low switching cost.
\newblock In {\em Advances in Neural Information Processing Systems}, pages
  8004--8013.

\bibitem[Bartlett and Tewari, 2009]{bartlett2009regal}
Bartlett, P.~L. and Tewari, A. (2009).
\newblock Regal: a regularization based algorithm for reinforcement learning in
  weakly communicating mdps.
\newblock In {\em Proceedings of the 25th Conference on Uncertainty in
  Artificial Intelligence (UAI 2009))}.

\bibitem[Beck and Srikant, 2012]{beck2012error}
Beck, C.~L. and Srikant, R. (2012).
\newblock Error bounds for constant step-size {Q}-learning.
\newblock {\em Systems \& control letters}, 61(12):1203--1208.

\bibitem[Bertsekas, 2019]{bertsekas2019reinforcement}
Bertsekas, D. (2019).
\newblock {\em Reinforcement learning and optimal control}.
\newblock Athena Scientific.

\bibitem[Brafman and Tennenholtz, 2003]{brafman2002r}
Brafman, R.~I. and Tennenholtz, M. (2003).
\newblock R-max - a general polynomial time algorithm for near-optimal
  reinforcement learning.
\newblock {\em J. Mach. Learn. Res.}, 3(Oct):213--231.

\bibitem[Cai et~al., 2019]{cai2019provably}
Cai, Q., Yang, Z., Jin, C., and Wang, Z. (2019).
\newblock Provably efficient exploration in policy optimization.
\newblock {\em arXiv preprint arXiv:1912.05830}.

\bibitem[Chen et~al., 2021]{chen2021implicit}
Chen, L., Jafarnia-Jahromi, M., Jain, R., and Luo, H. (2021).
\newblock Implicit finite-horizon approximation and efficient optimal
  algorithms for stochastic shortest path.
\newblock {\em Advances in Neural Information Processing Systems}, 34.

\bibitem[Chen et~al., 2020]{chen2020finite}
Chen, Z., Maguluri, S.~T., Shakkottai, S., and Shanmugam, K. (2020).
\newblock Finite-sample analysis of contractive stochastic approximation using
  smooth convex envelopes.
\newblock {\em Advances in Neural Information Processing Systems},
  33:8223--8234.

\bibitem[Cui and Yang, 2021]{cui2021minimax}
Cui, Q. and Yang, L.~F. (2021).
\newblock Minimax sample complexity for turn-based stochastic game.
\newblock In {\em Uncertainty in Artificial Intelligence}, pages 1496--1504.

\bibitem[Dann and Brunskill, 2015]{dann2015sample}
Dann, C. and Brunskill, E. (2015).
\newblock Sample complexity of episodic fixed-horizon reinforcement learning.
\newblock In {\em Advances in Neural Information Processing Systems}, pages
  2818--2826.

\bibitem[Dann et~al., 2017]{dann2017unifying}
Dann, C., Lattimore, T., and Brunskill, E. (2017).
\newblock Unifying {PAC} and regret: Uniform {PAC} bounds for episodic
  reinforcement learning.
\newblock {\em Advances in Neural Information Processing Systems}, 30.

\bibitem[Dann et~al., 2019]{dann2019policy}
Dann, C., Li, L., Wei, W., and Brunskill, E. (2019).
\newblock Policy certificates: Towards accountable reinforcement learning.
\newblock In {\em Proceedings of the 36th International Conference on Machine
  Learning}, pages 1507--1516.

\bibitem[Dann et~al., 2021]{dann2021beyond}
Dann, C., Marinov, T.~V., Mohri, M., and Zimmert, J. (2021).
\newblock Beyond value-function gaps: Improved instance-dependent regret bounds
  for episodic reinforcement learning.
\newblock {\em Advances in Neural Information Processing Systems}, 34:1--12.

\bibitem[Domingues et~al., 2021]{domingues2021episodic}
Domingues, O.~D., M{\'e}nard, P., Kaufmann, E., and Valko, M. (2021).
\newblock Episodic reinforcement learning in finite mdps: Minimax lower bounds
  revisited.
\newblock In {\em Algorithmic Learning Theory}, pages 578--598.

\bibitem[Dong et~al., 2019]{dong2019q}
Dong, K., Wang, Y., Chen, X., and Wang, L. (2019).
\newblock Q-learning with {UCB} exploration is sample efficient for
  infinite-horizon {MDP}.
\newblock {\em arXiv preprint arXiv:1901.09311}.

\bibitem[Efroni et~al., 2019]{efroni2019tight}
Efroni, Y., Merlis, N., Ghavamzadeh, M., and Mannor, S. (2019).
\newblock Tight regret bounds for model-based reinforcement learning with
  greedy policies.
\newblock {\em Advances in Neural Information Processing Systems}, 32.

\bibitem[Even{-}Dar and Mansour, 2003]{even2003learning}
Even{-}Dar, E. and Mansour, Y. (2003).
\newblock Learning rates for {Q}-learning.
\newblock {\em Journal of Machine Learning Research}, 5(Dec):1--25.

\bibitem[Freedman, 1975]{freedman1975tail}
Freedman, D.~A. (1975).
\newblock On tail probabilities for martingales.
\newblock {\em the Annals of Probability}, 3(1):100--118.

\bibitem[Fruit et~al., 2018]{fruit2018efficient}
Fruit, R., Pirotta, M., Lazaric, A., and Ortner, R. (2018).
\newblock Efficient bias-span-constrained exploration-exploitation in
  reinforcement learning.
\newblock In {\em ICML 2018-The 35th International Conference on Machine
  Learning}, volume~80, pages 1578--1586.

\bibitem[Jaksch et~al., 2010]{jaksch2010near}
Jaksch, T., Ortner, R., and Auer, P. (2010).
\newblock Near-optimal regret bounds for reinforcement learning.
\newblock {\em Journal of Machine Learning Research}, 11(Apr):1563--1600.

\bibitem[Ji and Li, 2023]{ji2023regret}
Ji, X. and Li, G. (2023).
\newblock Regret-optimal model-free reinforcement learning for discounted
  {MDP}s with short burn-in time.
\newblock {\em Advances in neural information processing systems}.

\bibitem[Jiang and Agarwal, 2018]{jiang2018open}
Jiang, N. and Agarwal, A. (2018).
\newblock Open problem: The dependence of sample complexity lower bounds on
  planning horizon.
\newblock In {\em Conference On Learning Theory}, pages 3395--3398.

\bibitem[Jin et~al., 2018]{jin2018q}
Jin, C., Allen-Zhu, Z., Bubeck, S., and Jordan, M.~I. (2018).
\newblock Is {Q}-learning provably efficient?
\newblock In {\em Advances in Neural Information Processing Systems}, pages
  4863--4873.

\bibitem[Jin et~al., 2020]{jin2020reward}
Jin, C., Krishnamurthy, A., Simchowitz, M., and Yu, T. (2020).
\newblock Reward-free exploration for reinforcement learning.
\newblock {\em International Conference on Machine Learning}.

\bibitem[Jin et~al., 2021]{jin2021pessimism}
Jin, Y., Yang, Z., and Wang, Z. (2021).
\newblock Is pessimism provably efficient for offline {RL}?
\newblock In {\em International Conference on Machine Learning}, pages
  5084--5096.

\bibitem[Kakade, 2003]{kakade2003sample}
Kakade, S.~M. (2003).
\newblock {\em On the sample complexity of reinforcement learning}.
\newblock PhD thesis, University of London London, England.

\bibitem[Kearns and Singh, 1998a]{kearns1998finite}
Kearns, M. and Singh, S. (1998a).
\newblock Finite-sample convergence rates for {Q}-learning and indirect
  algorithms.
\newblock {\em Advances in neural information processing systems}, 11.

\bibitem[Kearns and Singh, 1998b]{kearns1998near}
Kearns, M.~J. and Singh, S.~P. (1998b).
\newblock Near-optimal reinforcement learning in polynominal time.
\newblock In {\em Proceedings of the Fifteenth International Conference on
  Machine Learning}, pages 260--268.

\bibitem[Kolter and Ng, 2009]{kolter2009near}
Kolter, J.~Z. and Ng, A.~Y. (2009).
\newblock Near-bayesian exploration in polynomial time.
\newblock In {\em Proceedings of the 26th annual international conference on
  machine learning}, pages 513--520.

\bibitem[Lattimore and Hutter, 2012]{lattimore2012pac}
Lattimore, T. and Hutter, M. (2012).
\newblock {PAC} bounds for discounted {MDPs}.
\newblock In {\em International Conference on Algorithmic Learning Theory},
  pages 320--334. Springer.

\bibitem[Lee et~al., 2020]{lee2020bias}
Lee, C.-W., Luo, H., Wei, C.-Y., and Zhang, M. (2020).
\newblock Bias no more: high-probability data-dependent regret bounds for
  adversarial bandits and mdps.
\newblock {\em Advances in neural information processing systems},
  33:15522--15533.

\bibitem[Levine et~al., 2020]{levine2020offline}
Levine, S., Kumar, A., Tucker, G., and Fu, J. (2020).
\newblock Offline reinforcement learning: Tutorial, review, and perspectives on
  open problems.
\newblock {\em arXiv preprint arXiv:2005.01643}.

\bibitem[Li et~al., 2024a]{li2023q}
Li, G., Cai, C., Chen, Y., Wei, Y., and Chi, Y. (2024a).
\newblock Is {Q}-learning minimax optimal? a tight sample complexity analysis.
\newblock {\em Operations Research}, 72(1):222--236.

\bibitem[Li et~al., 2022]{li2022minimax}
Li, G., Chi, Y., Wei, Y., and Chen, Y. (2022).
\newblock Minimax-optimal multi-agent {RL} in {M}arkov games with a generative
  model.
\newblock {\em Advances in Neural Information Processing Systems},
  35:15353--15367.

\bibitem[Li et~al., 2024b]{li2022settling}
Li, G., Shi, L., Chen, Y., Chi, Y., and Wei, Y. (2024b).
\newblock Settling the sample complexity of model-based offline reinforcement
  learning.
\newblock {\em The Annals of Statistics}, 52(1):233--260.

\bibitem[Li et~al., 2021a]{li2021breaking}
Li, G., Shi, L., Chen, Y., Gu, Y., and Chi, Y. (2021a).
\newblock Breaking the sample complexity barrier to regret-optimal model-free
  reinforcement learning.
\newblock {\em Advances in Neural Information Processing Systems}, 34.

\bibitem[Li et~al., 2024c]{li2020breaking}
Li, G., Wei, Y., Chi, Y., and Chen, Y. (2024c).
\newblock Breaking the sample size barrier in model-based reinforcement
  learning with a generative model.
\newblock {\em Operations Research}, 72(1):203--221.

\bibitem[Li et~al., 2021b]{li2021sample}
Li, G., Wei, Y., Chi, Y., Gu, Y., and Chen, Y. (2021b).
\newblock Sample complexity of asynchronous {Q}-learning: Sharper analysis and
  variance reduction.
\newblock {\em IEEE Transactions on Information Theory}, 68(1):448--473.

\bibitem[Li et~al., 2024d]{li2023minimax}
Li, G., Yan, Y., Chen, Y., and Fan, J. (2024d).
\newblock Minimax-optimal reward-agnostic exploration in reinforcement
  learning.
\newblock {\em Conference on Learning Theory (COLT)}.

\bibitem[Li et~al., 2024e]{li2024reward}
Li, G., Zhan, W., Lee, J.~D., Chi, Y., and Chen, Y. (2024e).
\newblock Reward-agnostic fine-tuning: Provable statistical benefits of hybrid
  reinforcement learning.
\newblock {\em Advances in Neural Information Processing Systems}, 36.

\bibitem[Li et~al., 2021c]{li2021settling}
Li, Y., Wang, R., and Yang, L.~F. (2021c).
\newblock Settling the horizon-dependence of sample complexity in reinforcement
  learning.
\newblock In {\em IEEE Symposium on Foundations of Computer Science}.

\bibitem[Maurer and Pontil, 2009]{maurer2009empirical}
Maurer, A. and Pontil, M. (2009).
\newblock Empirical {B}ernstein bounds and sample variance penalization.
\newblock In {\em Conference on Learning Theory}.

\bibitem[M{\'e}nard et~al., 2021]{menard2021ucb}
M{\'e}nard, P., Domingues, O.~D., Shang, X., and Valko, M. (2021).
\newblock {UCB} momentum {Q}-learning: Correcting the bias without forgetting.
\newblock In {\em International Conference on Machine Learning}, pages
  7609--7618.

\bibitem[Neu and Pike-Burke, 2020]{neu2020unifying}
Neu, G. and Pike-Burke, C. (2020).
\newblock A unifying view of optimism in episodic reinforcement learning.
\newblock {\em arXiv preprint arXiv:2007.01891}.

\bibitem[Osband et~al., 2013]{osband2013more}
Osband, I., Russo, D., and Van~Roy, B. (2013).
\newblock (more) efficient reinforcement learning via posterior sampling.
\newblock In {\em Advances in Neural Information Processing Systems}, pages
  3003--3011.

\bibitem[Pacchiano et~al., 2020]{pacchiano2020optimism}
Pacchiano, A., Ball, P., Parker-Holder, J., Choromanski, K., and Roberts, S.
  (2020).
\newblock On optimism in model-based reinforcement learning.
\newblock {\em arXiv preprint arXiv:2006.11911}.

\bibitem[Pananjady and Wainwright, 2020]{pananjady2020instance}
Pananjady, A. and Wainwright, M.~J. (2020).
\newblock Instance-dependent $\ell_{\infty}$-bounds for policy evaluation in
  tabular reinforcement learning.
\newblock {\em IEEE Transactions on Information Theory}, 67(1):566--585.

\bibitem[Qu and Wierman, 2020]{qu2020finite}
Qu, G. and Wierman, A. (2020).
\newblock Finite-time analysis of asynchronous stochastic approximation and
  {Q}-learning.
\newblock In {\em Proceedings of Thirty Third Conference on Learning Theory
  (COLT)}.

\bibitem[Rashidinejad et~al., 2021]{rashidinejad2021bridging}
Rashidinejad, P., Zhu, B., Ma, C., Jiao, J., and Russell, S. (2021).
\newblock Bridging offline reinforcement learning and imitation learning: A
  tale of pessimism.
\newblock {\em Advances in Neural Information Processing Systems},
  34:11702--11716.

\bibitem[Ren et~al., 2021]{ren2021nearly}
Ren, T., Li, J., Dai, B., Du, S.~S., and Sanghavi, S. (2021).
\newblock Nearly horizon-free offline reinforcement learning.
\newblock {\em Advances in neural information processing systems},
  34:15621--15634.

\bibitem[Russo, 2019]{russo2019worst}
Russo, D. (2019).
\newblock Worst-case regret bounds for exploration via randomized value
  functions.
\newblock In {\em Advances in Neural Information Processing Systems}, pages
  14433--14443.

\bibitem[Shi et~al., 2022]{shi2022pessimistic}
Shi, L., Li, G., Wei, Y., Chen, Y., and Chi, Y. (2022).
\newblock Pessimistic {Q}-learning for offline reinforcement learning: Towards
  optimal sample complexity.
\newblock In {\em International Conference on Machine Learning}, pages
  19967--20025.

\bibitem[Shi et~al., 2023]{shi2023curious}
Shi, L., Li, G., Wei, Y., Chen, Y., Geist, M., and Chi, Y. (2023).
\newblock The curious price of distributional robustness in reinforcement
  learning with a generative model.
\newblock {\em Advances in Neural Information Processing Systems}.

\bibitem[Sidford et~al., 2018a]{sidford2018near}
Sidford, A., Wang, M., Wu, X., Yang, L., and Ye, Y. (2018a).
\newblock Near-optimal time and sample complexities for solving {Markov}
  decision processes with a generative model.
\newblock In {\em Advances in Neural Information Processing Systems}, pages
  5186--5196.

\bibitem[Sidford et~al., 2018b]{sidford2018variance}
Sidford, A., Wang, M., Wu, X., and Ye, Y. (2018b).
\newblock Variance reduced value iteration and faster algorithms for solving
  {Markov} decision processes.
\newblock In {\em Proceedings of the Twenty-Ninth Annual ACM-SIAM Symposium on
  Discrete Algorithms}, pages 770--787. Society for Industrial and Applied
  Mathematics.

\bibitem[Simchowitz and Jamieson, 2019]{simchowitz2019non}
Simchowitz, M. and Jamieson, K.~G. (2019).
\newblock Non-asymptotic gap-dependent regret bounds for tabular {MDPs}.
\newblock In {\em Advances in Neural Information Processing Systems}, pages
  1153--1162.

\bibitem[Strehl et~al., 2006]{strehl2006pac}
Strehl, A.~L., Li, L., Wiewiora, E., Langford, J., and Littman, M.~L. (2006).
\newblock {PAC model-free reinforcement learning}.
\newblock In {\em Proceedings of the 23rd international conference on Machine
  learning}, pages 881--888. ACM.

\bibitem[Strehl and Littman, 2008]{strehl2008analysis}
Strehl, A.~L. and Littman, M.~L. (2008).
\newblock An analysis of model-based interval estimation for markov decision
  processes.
\newblock {\em Journal of Computer and System Sciences}, 74(8):1309--1331.

\bibitem[Szita and Szepesv{\'a}ri, 2010]{szita2010model}
Szita, I. and Szepesv{\'a}ri, C. (2010).
\newblock Model-based reinforcement learning with nearly tight exploration
  complexity bounds.
\newblock In {\em ICML}.

\bibitem[Talebi and Maillard, 2018]{talebi2018variance}
Talebi, M.~S. and Maillard, O.-A. (2018).
\newblock Variance-aware regret bounds for undiscounted reinforcement learning
  in mdps.
\newblock {\em arXiv preprint arXiv:1803.01626}.

\bibitem[Tarbouriech et~al., 2021]{tarbouriech2021stochastic}
Tarbouriech, J., Zhou, R., Du, S.~S., Pirotta, M., Valko, M., and Lazaric, A.
  (2021).
\newblock Stochastic shortest path: Minimax, parameter-free and towards
  horizon-free regret.
\newblock {\em Advances in Neural Information Processing Systems}, 34.

\bibitem[Tirinzoni et~al., 2021]{tirinzoni2021fully}
Tirinzoni, A., Pirotta, M., and Lazaric, A. (2021).
\newblock A fully problem-dependent regret lower bound for finite-horizon
  {MDPs}.
\newblock {\em arXiv preprint arXiv:2106.13013}.

\bibitem[Wagenmaker et~al., 2022]{wagenmaker2022first}
Wagenmaker, A.~J., Chen, Y., Simchowitz, M., Du, S., and Jamieson, K. (2022).
\newblock First-order regret in reinforcement learning with linear function
  approximation: A robust estimation approach.
\newblock In {\em International Conference on Machine Learning}, pages
  22384--22429.

\bibitem[Wainwright, 2019a]{wainwright2019stochastic}
Wainwright, M.~J. (2019a).
\newblock Stochastic approximation with cone-contractive operators: Sharp
  $\ell_{\infty} $-bounds for {Q}-learning.
\newblock {\em arXiv preprint arXiv:1905.06265}.

\bibitem[Wainwright, 2019b]{wainwright2019variance}
Wainwright, M.~J. (2019b).
\newblock Variance-reduced {Q}-learning is minimax optimal.
\newblock {\em arXiv preprint arXiv:1906.04697}.

\bibitem[Wang et~al., 2023]{wang2023benefits}
Wang, K., Zhou, K., Wu, R., Kallus, N., and Sun, W. (2023).
\newblock The benefits of being distributional: Small-loss bounds for
  reinforcement learning.
\newblock {\em arXiv preprint arXiv:2305.15703}.

\bibitem[Wang et~al., 2020]{wang2020long}
Wang, R., Du, S.~S., Yang, L.~F., and Kakade, S.~M. (2020).
\newblock Is long horizon reinforcement learning more difficult than short
  horizon reinforcement learning?
\newblock In {\em Advances in Neural Information Processing Systems}.

\bibitem[Wang et~al., 2022]{wang2022gap}
Wang, X., Cui, Q., and Du, S.~S. (2022).
\newblock On gap-dependent bounds for offline reinforcement learning.
\newblock {\em Advances in Neural Information Processing Systems},
  35:14865--14877.

\bibitem[Xie et~al., 2021]{xie2021policy}
Xie, T., Jiang, N., Wang, H., Xiong, C., and Bai, Y. (2021).
\newblock Policy finetuning: Bridging sample-efficient offline and online
  reinforcement learning.
\newblock {\em Advances in neural information processing systems},
  34:27395--27407.

\bibitem[Xiong et~al., 2022]{xiong2021randomized}
Xiong, Z., Shen, R., Cui, Q., Fazel, M., and Du, S.~S. (2022).
\newblock Near-optimal randomized exploration for tabular markov decision
  processes.
\newblock {\em Advances in Neural Information Processing Systems},
  35:6358--6371.

\bibitem[Xu et~al., 2021]{xu2021fine}
Xu, H., Ma, T., and Du, S. (2021).
\newblock Fine-grained gap-dependent bounds for tabular {MDPs} via adaptive
  multi-step bootstrap.
\newblock In {\em Conference on Learning Theory}, pages 4438--4472.

\bibitem[Yan et~al., 2023]{yan2022efficacy}
Yan, Y., Li, G., Chen, Y., and Fan, J. (2023).
\newblock The efficacy of pessimism in asynchronous {Q}-learning.
\newblock {\em IEEE Transactions on Information Theory}, 69(11):7185--7219.

\bibitem[Yang et~al., 2021]{yang2021q}
Yang, K., Yang, L., and Du, S. (2021).
\newblock {$Q$}-learning with logarithmic regret.
\newblock In {\em International Conference on Artificial Intelligence and
  Statistics}, pages 1576--1584.

\bibitem[Yin et~al., 2022]{yin2022near}
Yin, M., Duan, Y., Wang, M., and Wang, Y.-X. (2022).
\newblock Near-optimal offline reinforcement learning with linear
  representation: Leveraging variance information with pessimism.
\newblock {\em arXiv preprint arXiv:2203.05804}.

\bibitem[Zanette and Brunskill, 2019]{zanette2019tighter}
Zanette, A. and Brunskill, E. (2019).
\newblock Tighter problem-dependent regret bounds in reinforcement learning
  without domain knowledge using value function bounds.
\newblock In {\em International Conference on Machine Learning}, pages
  7304--7312.

\bibitem[Zhang et~al., 2021a]{zhang2020reinforcement}
Zhang, Z., Ji, X., and Du, S. (2021a).
\newblock Is reinforcement learning more difficult than bandits? a near-optimal
  algorithm escaping the curse of horizon.
\newblock In {\em Conference on Learning Theory}, pages 4528--4531.

\bibitem[Zhang et~al., 2022]{zhang2022horizon}
Zhang, Z., Ji, X., and Du, S. (2022).
\newblock Horizon-free reinforcement learning in polynomial time: the power of
  stationary policies.
\newblock In {\em Conference on Learning Theory}, pages 3858--3904.

\bibitem[Zhang et~al., 2020]{zhang2020almost}
Zhang, Z., Zhou, Y., and Ji, X. (2020).
\newblock Almost optimal model-free reinforcement learning via
  reference-advantage decomposition.
\newblock In {\em Advances in Neural Information Processing Systems}.

\bibitem[Zhang et~al., 2021b]{zhang2020model}
Zhang, Z., Zhou, Y., and Ji, X. (2021b).
\newblock Model-free reinforcement learning: from clipped pseudo-regret to
  sample complexity.
\newblock In {\em International Conference on Machine Learning}, pages
  12653--12662.

\bibitem[Zhao et~al., 2023]{zhao2023variance}
Zhao, H., He, J., Zhou, D., Zhang, T., and Gu, Q. (2023).
\newblock Variance-dependent regret bounds for linear bandits and reinforcement
  learning: Adaptivity and computational efficiency.
\newblock {\em arXiv preprint arXiv:2302.10371}.

\bibitem[Zhou et~al., 2023]{zhou2023sharp}
Zhou, R., Zihan, Z., and Du, S.~S. (2023).
\newblock Sharp variance-dependent bounds in reinforcement learning: Best of
  both worlds in stochastic and deterministic environments.
\newblock In {\em International Conference on Machine Learning}, pages
  42878--42914.

\end{thebibliography}
\bibliographystyle{apalike}

\end{document}